\documentclass[aos]{imsart}

\RequirePackage{amsthm,amsmath,amsfonts,amssymb}
\RequirePackage[authoryear]{natbib} 
\RequirePackage[colorlinks,citecolor=blue,urlcolor=blue]{hyperref}
\RequirePackage{graphicx}
\usepackage{epsfig}
\usepackage{comment}
\usepackage{booktabs, float}

\usepackage{algorithm}
\usepackage{algpseudocode}

\startlocaldefs

\newtheorem{theorem}{Theorem}[section]
\newtheorem{lemma}[theorem]{Lemma}
\theoremstyle{remark}

\numberwithin{equation}{section}
\theoremstyle{plain}

\newtheorem{example}{Example}
\newtheorem{assumption}{Assumption}
\newtheorem{definition}{Definition}

\newtheorem{corollary}{Corollary}[section]

\usepackage{bbm}

\DeclareMathOperator*{\argmin}{arg\,min}

\def\sde{\text{sde}}
\def\ode{\text{ode}}

\endlocaldefs

\begin{document}

\begin{frontmatter}
\title{Conditional Stochastic Interpolation for\\ Generative Learning}
\runtitle{Conditional Stochastic Interpolation}

\begin{aug}
\author[A]{\fnms{Ding} \snm{Huang}\ead[label=e1,mark]{ding.huang@connect.polyu.hk}},
\author[A,B]
{\fnms{Jian} \snm{Huang}\ead[label=e2,mark]{j.huang@polyu.edu.hk}\ead[label=u1,url]{https://www.polyu.edu.hk/en/ama/people/academic-staff/prof-huang-jian/}},
\author[A]{\fnms{Ting} \snm{Li}\ead[label=e3,mark]{tingeric.li@polyu.edu.hk}} and
\author[A]{\fnms{Guohao} \snm{Shen}\ead[label=e4,mark]{guohao.shen@polyu.edu.hk}}

\address[A]{Department of Applied Mathematics,
The Hong Kong Polytechnic University, Hong Kong SAR, China\\ \printead{e1}, \printead{e3}, \printead{e4}}

\address[B]{Department of Data Science and AI, The Hong Kong Polytechnic
University, Hong Kong SAR, China\\ \printead{e2}}

\end{aug}

\begin{abstract}
    We propose a conditional stochastic interpolation (CSI) method for learning conditional distributions. CSI  is based on estimating probability flow equations or stochastic differential equations that transport a reference distribution to the target conditional distribution. This is achieved by first learning the conditional drift and score functions based on CSI, which are then used to construct a deterministic process governed by an ordinary differential equation or a diffusion process for conditional sampling. In our proposed approach, we incorporate an adaptive diffusion term to address the instability issues arising in the diffusion process.
    We derive explicit expressions of the conditional drift and score functions in terms of conditional expectations, which naturally lead to an nonparametric regression approach to estimating these functions. Furthermore, we establish nonasymptotic error bounds for learning the target conditional distribution. We illustrate the application of CSI on image generation using a benchmark image dataset.
\end{abstract}



\end{frontmatter}

\section{Introduction}

In recent years, there have been important advances in statistical modeling and analysis of high-dimensional data using a \textit{generative learning} approach with deep neural networks. For example, for learning  (unconditional) distributions of high-dimensional data arising in image analysis and natural language processing, the {generative adversarial networks} (GANs) \citep{goodfellow14, arjovsky2017wasserstein} have proven to be effective and achieved impressive success \citep{reed16,zhu17}.
Instead of estimating the functional forms of density functions, GANs start from a known reference distribution and learn a function that maps the reference distribution to the data distribution.

The basic idea of GANs has also been extended to learn conditional distributions. In particular, conditional generative adversarial networks \citep{mirza2014conditional, zhou2022deep} have shown to be able to generate high-quality samples under predetermined conditions. However, these adversarially trained generative models are not exempt from drawbacks, primarily the problems of training instability \citep{arjovsky2017wasserstein, karras2019style} and mode collapse \citep{zhao2018bias}.

Recently, process-based methods have emerged as an area of significant interest for generative modeling due to its impressive results. In contrast to estimating a generator function in GANs, it learns a transport procedure that converts a simple reference distribution into a high-dimensional target distribution. There are mainly two categories of methods: Stochastic Differential Equation (SDE)-based generative models and Ordinary Differential Equation (ODE)-based generative models. A notable instance of SDE-based generative modelings is diffusion models \citep{sohl2015deep, ho2020denoising, song2020score, meng2021sdedit}, which have achieved impressive empirical success \citep{dhariwal2021diffusion}. This approach embeds the target distribution into a Gaussian distribution via an Ornstein-Uhlenbeck process and solves a reverse-time SDE. This reverse-time process is extended to the deterministic process that maintains the consistency in time-dependent distribution \citep{song2020denoising, song2020score, lu2022dpm, lipman2022flow}. ODE-based models also have excellent performance \citep{liu2022flow, albergo2022building, liu2023flowgrad, xu2022poisson}. Most ODE-based methods use an interpolative trajectory modeling approach \citep{liu2022flow, albergo2022building, liu2023flowgrad,lipman2022flow}. In particular, a linear interpolation to connect a reference distribution and the target distribution is introduced by \citet{liu2022flow}. Both the ODE and SDE  methods are known as score-based generative models.
While these methods have achieved excellent performance, the process is truncated at a finite time point, which introduces bias into the reverse process and the final estimation, due to the requirement for an infinite time horizon in this method.

In a recent stimulating paper \citep{albergo2023stochastic}, the authors propose to use stochastic interpolations to generate samples for estimating drift and score functions, and then employ ODE and SDE-based generative models to transform a reference distribution to a target distribution within a finite time interval. In addition, the score and drift functions under this framework can be estimated more easily based on a quadratic loss function, instead of score matching as in the existing score based methods. However,  the problem of score function explosion in SDE generators has not been considered. Additionally, the statistical error properties of the ODE- and SDE-based generators trained with stochastic interpolation have not been studied.

We extend the approach of \citet{albergo2023stochastic} from the unconditional setting and propose a Conditional Stochastic Interpolation (CSI) method to conditional generative learning. We show that the conditional dirft and score functions of CSI can be expressed as certain conditional expectations. Based on these expressions, we can convert conditional drift and score estimation problems into nonparametric regression problems. We leverage the expressive power of deep neural networks for estimating high-dimensional conditional score and drift functions. We establish nonasymptotic error bounds of the proposed CSI method for learning the target conditional distribution.

The main contributions and the novel aspects of our work are outlined as follows.
\begin{itemize}
\item We formulate a conditional stochastic interpolation approach for learning conditional
distributions. This approach facilitates a bias-free generative model for both ODE- and SDE-based generators over a finite time interval. This represents a significant advancement in generative methods for learning conditional conditions.

\item We provide sufficient conditions on the interpolation process that guarantee the stability of the
conditional drift and score functions at the boundary points of the time interval.
To mitigate instability issues, we also propose an adaptive diffusion term in our SDE-based generators, which, when combined with the stability conditions on the interpolation process,  ensures our method's stability on the entire time interval $[0, 1].$

\item We establish error bounds for the estimated conditional distributions of the ODE- and
SDE-based generative models in terms of the 2-Wasserstein distance and KL divergence, respectively. We take into account the errors incurred from using deep neural networks to approximate the conditional drift and score functions. To the best of our knowledge, these results are new within the context of ODE- and SDE-based conditional flow models.

\item We derive a linear relationship between the conditional score function and the drift function when the interpolation takes an additive form. This relationship helps reduce the computational burden and leads to an improved convergence property for the estimated conditional score function.

\item We conduct numerical experiments to illustrate that our proposed CSI approach is capable
of generating high-quality samples, underscoring the effectiveness and reliability of our method.
\end{itemize}

These contributions highlight the significance of our work, offering new methods with theoretical guarantees for learning conditional distributions.

The remainder of this article is organized as follows. Section \ref{sec_framework} introduces the definition of CSI along with its basic properties. Section \ref{sub_ode} describes the generative models under the CSI approach.  Section \ref{sec_theory} presents the theoretical analysis of the CSI estimators for the underlying target distribution. In Section \ref{sec_relatedwork}, we examine relevant literature and outline the similarities and distinctions between our framework and existing approaches. Section \ref{sec_experiment} provides an evaluation of our proposal methods through simulation studies on image generation and reconstruction tasks with benchmark data.
 Finally, Section \ref{sec_conclusion} concludes and discusses some future works. Proofs and  technical details are provided in the supplementary material.

\section{Conditional stochastic interpolation}
\label{sec_framework}

Our objective is to learn a conditional distribution, a task that can be particularly challenging in high-dimensional settings when approached through direct nonparametric estimation. The proposed CSI method addresses this challenge by using a stochastic interpolation process that serves as a conduit between the target conditional distribution and a simple reference distribution. It then constructs a conditional process based on ODEs/SDEs using the estimated conditional drift and score functions to generate samples that approximately follow the target conditional distribution. In this section, we will present the CSI method and study its basic properties.

Let $(X,Y_1) \in \mathcal{X} \times \mathcal{Y}$ be a pair of random vectors, where
$\mathcal{X} \times \mathcal{Y} \subseteq \mathbb{R}^k \times \mathbb{R}^d$ for positive integers $k$ and $d.$
We are interested in the conditional distribution of $Y_1$ given $X$, denoted by $\mathbb{P}_{Y_1 \mid X}.$  Given a known distribution $\mathbb{P}_{Y_0}$ of a random vector $Y_0\in \mathbb{R}^d,$ we first construct a pathway between $\mathbb{P}_{Y_0}$ and the marginal distribution $\mathbb{P}_{Y_1}$ using the following stochastic interpolation.

\subsection{Definitions}
\label{sub_CSI}
 Assume that $Y_0$ and $Y_1$
are independent  and both are absolutely continuous with respect to the Lebesgue measure. Let
  \begin{align}\label{cond_stoch_interp_equa}
        Y_t  = \mathcal{I}(Y_0, Y_1, t)  + \gamma(t)\boldsymbol{\eta}, \  t \in [0,1],
    \end{align}
where $ \boldsymbol{\eta}\sim N(\mathbf{0}, \mathbf{I}_d)$ and is independent of $Y_0 $ and $Y_1$, the function $\mathcal{I}: \mathcal{Y} \times \mathcal{Y} \times [0,1] \rightarrow \mathbb{R}^d$ is called an interpolation function, and $\gamma: [0,1] \to \mathbb{R}$ is a real-valued function.
For every $ t \in (0, 1),$  $Y_t$ is an interpolation between $Y_0$ and $Y_1$ plus a Gaussian noise with time-dependent variance $\gamma^2(t).$
We make the following assumptions about $\gamma$ and $\mathcal{I}.$

\begin{assumption}
\label{assumption0}

(a) The interpolation function $\mathcal{I} (\mathbf{y}_0,\mathbf{y}_1,t)$ is continuous differentiable with respect to $y_0$ and $y_1$, and satisfies the boundary condition $\mathcal{I}(\mathbf{y}_0,\mathbf{y}_1,0)=\mathbf{y}_0 $ and $\mathcal{I}(\mathbf{y}_0,\mathbf{y}_1, 1)=\mathbf{y}_1$ for every $\mathbf{y}_0, \mathbf{y}_1 \in \mathcal{Y}.$

 (b) The function $\gamma$ satisfies one of the following two conditions:
  (i) For $t\in [0,1]$, $\gamma(t) \equiv 0$;
       or  (ii)  For $t\in (0,1)$, $\gamma(t) > 0 $ and $\gamma(0) = \gamma(1)=0 .$
\end{assumption}
These are the basic requirements on $\gamma$ and $\mathcal{I},$ which are needed in constructing
proper interpolation processes.

\begin{definition}[Conditional stochastic interpolation]
\label{cond_stoch_interp}
Assume that $Y_0$, $\boldsymbol{\eta}$ and $(X, Y_1)$ are mutually independent. A conditional stochastic interpolation between the reference distribution $\mathbb{P}_{Y_0}$ and the conditional distribution  $\mathbb{P}_{Y_1 \mid X}$ is defined as
 the set of the conditional distributions $\{\mathbb{P}_{Y_t|X}: t \in [0, 1]\},$ where $\{Y_t, t \in [0, 1]\}$ is defined in  (\ref{cond_stoch_interp_equa}).
\end{definition}

The set $\{\mathbb{P}_{Y_t|X}: t \in [0, 1]\}$ can be viewed as a path of conditional distributions
connecting $\mathbb{P}_{Y_0|X} \equiv \mathbb{P}_{Y_0}$ and $\mathbb{P}_{Y_1|X}.$
We can also conceptualize this using a stochastic process with $\mathbb{P}_{Y_t|X}$ as the corresponding conditional distributions.
Indeed, the proposed CSI method aims to construct such a process. It involves two steps:
(a) Calculating the interpolation $\{Y_t, t \in [0, 1]\}$
and estimating a conditional generative model through the paired data $\{(Y_t, X), t \in [0, 1]\},$
 (b) Generate samples from the estimated model in the previous step.
We use an ODE Flow model and an SDE Diffusion model for conditional sample generation by first estimating the conditional drift and score functions.

The interpolation process $\{Y_t, t \in [0, 1]\}$  involves $\mathcal{I} ,\gamma $ and the reference distribution $\mathbb{P}_{Y_0}$. Different choices of $\mathcal{I} ,\gamma $, and $\mathbb{P}_{Y_0}$ result in a flexible selection space. When $\gamma(t) \equiv 0$, CSI degenerates into a deterministic process. For example, the rectified flow \citep{liu2022flow} is a linear interpolation model with $\mathcal{I}(\mathbf{y}_0, \mathbf{y}_1, t) =(1-t)\mathbf{y}_0+ t \mathbf{y}_1 $ and $\gamma(t)=0$.
Figure \ref{eig} provides some examples of the interpolation function $\cal{I}$ and the perturbation function $\gamma$, as well as displays the generating processes from white noise to an image using these functions.

\begin{figure}[H]
    \centering
    \includegraphics[width = 0.85\linewidth]{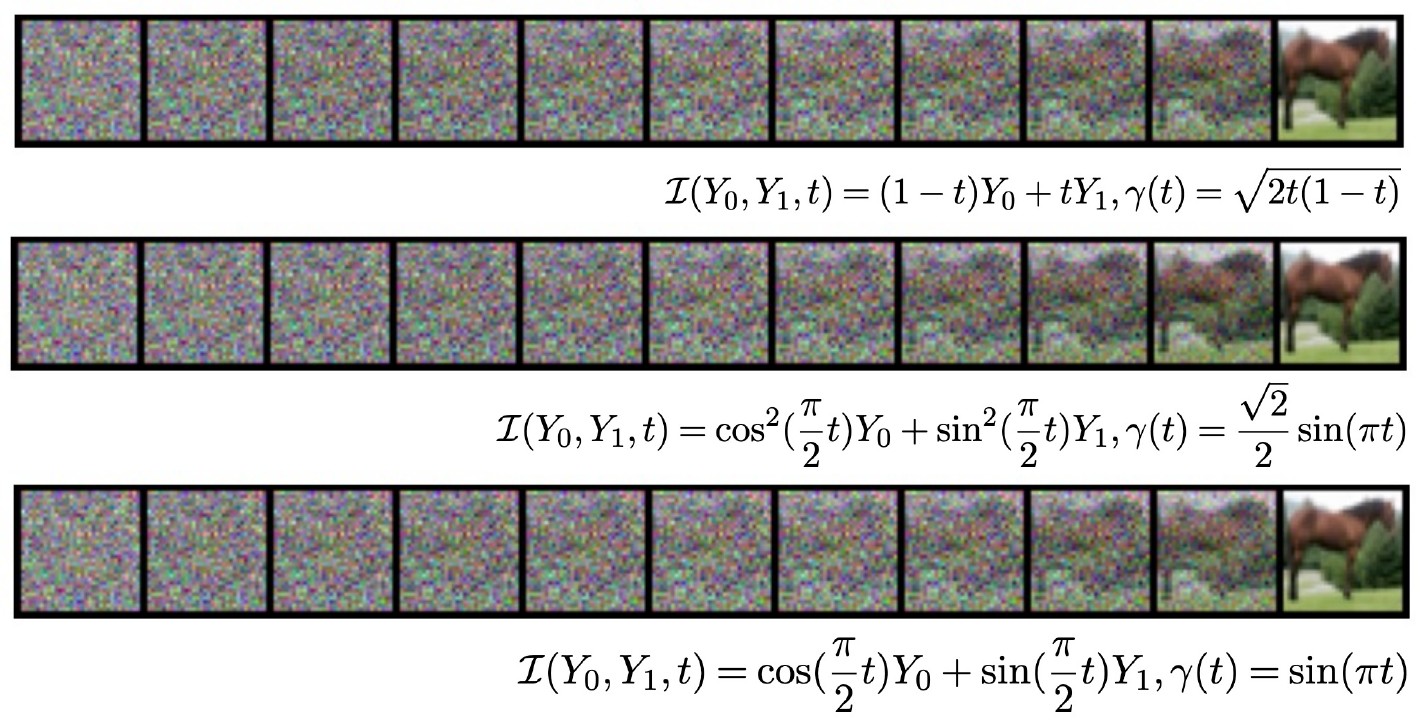}
\caption{Examples of the interpolation function $\cal{I}$ and the perturbation function $\gamma$.}
\label{eig}
\end{figure}

Below we consider two field functions, the velocity field (or drift function) and the distribution field (or score function),  that are important in characterizing a CSI process.
We need some additional assumptions on $\gamma$ and $\mathcal{I}.$

\begin{assumption}
\label{assumption0b}
(a) The interpolation function $\mathcal{I} (\mathbf{y}_0,\mathbf{y}_1,t)$ is  continuous with respect to $t$, $y_0$ and $y_1$ on $\mathcal{X}\times\mathbb{R}^d\times[0,1]$, and is second-order differentiable on $\mathcal{X}\times\mathbb{R}^d\times(0,1)$.  The function $\gamma(t)$ is continuous on $[0,1]$ and secondary differentiable in $(0,1)$ with respect to $t$. (b) There exists $M_1, M_2<\infty$ such that the random variables $Y_0$, $(X,Y_1)$ and the function $\mathcal{I}$ satisfy
$$\mathbb{E}\left[\left|\partial_t \mathcal{I}\left(Y_0, Y_1, t\right)\right|^4 \mid \mathbf{x}\right] \leq M_1, \  \mathbb{E}\left[\left|\partial_t^2 I\left(Y_0, Y_1, t\right)\right|^2 \mid \mathbf{x} \right] \leq M_2, \ (\mathbf{x}, t) \in \mathcal{X}\times [0, 1].
$$
\end{assumption}

Assumption \ref{assumption0b}(a) ensures that the drift function is well defined and can be easily satisfied by choosing appropriate $\mathcal{I}$ and $\gamma.$ Assumption \ref{assumption0b} (b) guarantees the continuity and smoothness of the density function, which is demonstrated in Theorem \ref{pro_b}. In the rest of the paper, we denote $\partial_t \mathcal{I}(\cdot,\cdot, t)$ and $\dot{\gamma}(t)$ by their derivatives with respect to $t\in(0,1)$.   In particular, this assumption is satisfied for the usual choices, as shown in examples in Figure \ref{eig}.

\begin{definition}[Conditional drift and score functions]
\label{defi1}
Denote the time-dependent conditional density of $Y_t \mid X = \mathbf{x}$ by
$\rho^*(\mathbf{x}, \mathbf{y}, t)$, which can be considered a map: $\rho^*: \mathcal{X} \times \mathcal{Y} \times [0,1] \rightarrow [0, + \infty).$
\begin{itemize}
\item[(a)]
The conditional drift function $\boldsymbol{b}^*: \mathcal{X} \times \mathbb{R}^{d}\times (0,1) \rightarrow \mathbb{R}^{d}$ is defined by
    \begin{align}\label{drif_equa}
        \boldsymbol{b}^*(\mathbf{x},\mathbf{y},t) := \mathbb{E}[\partial_t {\mathcal{I}}(Y_0,Y_1,t) + \dot{\gamma}(t)\boldsymbol{\eta} \mid X = \mathbf{x}, Y_t = \mathbf{y}], \ t \in (0, 1).
    \end{align}

\item[(b)]
The conditional score function $\boldsymbol{s}^*: \mathcal{X} \times \mathbb{R}^{d}\times (0,1) \rightarrow \mathbb{R}^{d}$ is defined as
    \begin{align*}
        \boldsymbol{s}^*(\mathbf{x},\mathbf{y},t):=\nabla_{\mathbf{y}} \log \rho^*(\mathbf{x},\mathbf{y},t), \ t \in (0, 1),
    \end{align*}
if the gradients $\nabla_{\mathbf{y}} \rho^*_{\mathbf{x}}(\mathbf{y}, t)$ exist for all $(\mathbf{x}, \mathbf{y}, t) \in \mathcal{X}\times \mathcal{Y} \times (0, 1)$.
\end{itemize}
\end{definition}

To emphasize that these functions are defined for a given $\mathbf{x} \in \mathcal{X},$ we also write
\begin{align}\label{cdt}
\boldsymbol{b}^{*}_{\mathbf{x}}(\mathbf{y},t):=\boldsymbol{b}^*(\mathbf{x},\mathbf{y},t),\
 \rho^*_{\mathbf{x}}(\mathbf{y},t):=\rho^*(\mathbf{x},\mathbf{y},t),
 \ \text{and} \ \boldsymbol{s}^{*}_{\mathbf{x}}(\mathbf{y},t)=\nabla_{\mathbf{y}} \log \rho^{*}_{\mathbf{x}}(\mathbf{y},t).
\end{align}
In subsequent sections, we will alternate between these notations as dictated by context and convenience, ensuring that their usage does not lead to ambiguity.

\subsection{Basic properties}
\label{sub_CSIb}
The conditional drift function represents the deterministic part of the process, which captures the long-term behavior of the process. It provides information of the mean velocity for the process at each location and time. For non-degenerate CSI ($\gamma(t) \not\equiv 0$), the conditional score function measures the sensitivity of the changes in the density function of the process. The conditional drift and score functions are used to construct equivalent processes described by ODE or SDEs.

We first prove that the time-dependent conditional density satisfies the transport equation or Liouville's equation with regard to the conditional drift function.

\begin{theorem}[Transport equation]
\label{pro_b}
Suppose that Assumptions \ref{assumption0} and \ref{assumption0b} are satisfied.
Given $\mathbf{x} \in \mathcal{X}$, the time-dependent conditional density $\rho^*_\mathbf{x}: \mathbb{R}^d \times [0,1]\to [0, \infty)$ is absolutely continuous with respect to the Lebesgue measure for all $t \in[0,1]$ and satisfies $\rho^*_\mathbf{x} \in C^1\left([0,1]; C^p\left(\mathbb{R}^d\right)\right)$ for any $p \in \mathbb{N}$, $(t, \mathbf{y}_t) \in[0,1] \times \mathbb{R}^d$. In addition, $\rho^*_{\mathbf{x}}$ satisfies the transport equation
\begin{align} \label{tp_equa}
        \partial_t \rho^*_{\mathbf{x}} +  \nabla_{\mathbf{y}} \cdot (\boldsymbol{b}^*_\mathbf{x} \rho^*_\mathbf{x}) = 0
\end{align}
in a weak sense, where $\partial_t \rho^*_{\mathbf{x}}$ denotes the partial derivative of $\rho^{*}_{\mathbf{x}}$ with respect to $t$, and $\nabla_{\mathbf{y}} \cdot $ denotes the divergence operator, i.e., $\nabla_{\mathbf{y}} \cdot (\boldsymbol{b}^{*}_{\mathbf{x}}\rho^{*}_{\mathbf{x}})=  \Sigma_{i=1}^d \frac{\partial}{\partial \mathbf{y}^{(i)}} (\boldsymbol{b}^{*}_{\mathbf{x}}\rho^{*}_{\mathbf{x}})^{(i)},$  with $\mathbf{y}^{(i)}$ representing the $i$-th component of $\mathbf{y}$.
\end{theorem}

{\color{black}Theorem \ref{pro_b} guarantees the continuity and smoothness of the conditional density function $\rho_{\mathbf{x}}(t,\mathbf{y})$ of $Y_t$ and that $\rho_{\mathbf{x}}(t, \mathbf{y})$ exists and is differentiable with respect to $t$ and $\mathbf{y}$. Also, the transport equation $ \partial_t \rho^*_{\mathbf{x}} +  \nabla_{\mathbf{y}} \cdot (\boldsymbol{b}^*_\mathbf{x} \rho^*_\mathbf{x}) = 0$ holds in a weak sense.
Theorem \ref{pro_b} characterizes the process of distributional changes over time during the stochastic interpolation process and serves as the foundation for demonstrating the equivalence of the conditional distributions with the corresponding ODE/SDE-based generators discussed later.}

In the following, we derive an expression for the conditional score function in terms of
a conditional expectation. Previously, \cite{albergo2023stochastic} focused on the unconditional case for explicit form of the score function under smoothness conditions on the data density function.
 Here, we also relax the smmothness condition and obtain the formulation of the conditional score function by employing a novel technique.

\begin{theorem}[Conditional score function]\label{pro_score}
 Suppose Assumptions \ref{assumption0} and \ref{assumption0b} are satisfied.
    If $\gamma(t) \neq 0$ for every $t \in (0, 1)$, then the conditional score function $\boldsymbol{s}^*$ can be expressed as
    \begin{align} \label{score_equa}
  \boldsymbol{s}^{*}_{\mathbf{x}}(\mathbf{y},t):=\boldsymbol{s}^*(\mathbf{x},\mathbf{y},t)
        = -\frac{1}{\gamma(t)} \mathbb{E}\left[\boldsymbol{\eta} \mid X= \mathbf{x}, Y_t=\mathbf{y} \right],\  t \in (0, 1).
    \end{align}
\end{theorem}
We note that expression (\ref{score_equa}) depends on the inclusion of a noise term in the interpolation (\ref{cond_stoch_interp_equa}), that is, $\gamma(t)>0, t \in (0, 1).$
Theorem \ref{pro_score} also highlights the potential for the score function to exhibit explosive behavior near the boundary points $t={0, 1},$ since $\gamma(0) = \gamma(1)=0.$ Such phenomena have been previously noted in models based on stochastic differential equations  \citep{kim2021soft}. The behavior of the score function at these boundaries is influenced not only by the perturbation function $\gamma(t)$ but also by the reference distribution and the data distribution. Let
\begin{align}
\label{kappa1}
\boldsymbol{\kappa}^*(\mathbf{x},\mathbf{y},t):= \mathbb{E}[\boldsymbol{\eta}\mid X = \mathbf{x},Y_t = \mathbf{y}], \ t\in (0, 1).
\end{align}
Similar to the denoising function in diffusion models \citep{ho2020denoising, song2020score}, we refer to $\boldsymbol{\kappa}^*$  as the conditional denoising function in CSI. Then (\ref{score_equa}) implies that
\begin{align}
\label{kappa2}
\boldsymbol{s}^*(\mathbf{x},\mathbf{y},t) =-\frac{1}{\gamma(t)} \boldsymbol{\kappa}^*(\mathbf{x},\mathbf{y},t), \ t \in (0, 1).
\end{align}
Clearly, $\boldsymbol{\kappa}^*$ has better regularity properties than $\boldsymbol{s}^*.$  We will first estimate $\boldsymbol{\kappa}^*$ and then
obtain an estimator of $\boldsymbol{s}^*$ based on (\ref{kappa2}), see Section \ref{estimation} below. We now provide sufficient conditions under which the score function and the drift function are stable near $t=0$.
\begin{theorem}[Boundary behavior at $t=0$]
\label{boundary_score}
 Suppose that Assumptions \ref{assumption0} and \ref{assumption0b} are satisfied and $Y_0 \sim \mathcal{N}(\boldsymbol{0},I_d)$ with its density function denoted by $p_{Y_0}.$
For a fixed $\mathbf{x}\in \mathcal{X}, $  let $p_{\mathcal{I}(Y_0,Y_1,t) \mid \mathbf{x}}$ denote the conditional density function of  $\mathcal{I}(Y_0,Y_1,t)$ given  $ X=\mathbf{x}.$
Suppose $\Vert p_{\mathcal{I}(Y_0,Y_1,t) \mid \mathbf{x}} - p_{Y_0}\Vert_\infty = o(\gamma(t))$ as $t\to 0$  for any given $\mathbf{x} \in\mathcal{X}$, then the limitation $\lim_{t \to 0}\boldsymbol{s}^*_{\mathbf{x}}(\mathbf{y},t)$ exits with
    $
    \boldsymbol{s}^*_{\mathbf{x}}(\mathbf{y},0) = -\mathbf{y},
    $
    and the value of the drift function at $t = 0$ is
    $    \boldsymbol{b}^*_{\mathbf{x}}(\mathbf{y},0) = \mathbb{E}[\partial_t {\mathcal{I}}(\mathbf{y},Y_1, 0) \mid  X = \mathbf{x}].
    $
\end{theorem}

Theorem \ref{boundary_score}
shows that the values of the score and drift functions remain bounded and integrable at the initial time $t=0$, which facilitates the use of numerical solvers. Techniques such as the Euler-Maruyama and stochastic Runge-Kutta methods \citep{platen2010numerical} can then be used to generate approximate trajectories from the ODEs/SDEs. In contrast, existing methods \citep{liu2022flow, albergo2022building, lipman2022flow, dao2023flow, song2020score} often truncate the integral domain near the boundary $t=0$ to circumvent the potential blow-up behavior of the integrand function, yet this introduces an unnecessary truncation bias.
Corollary \ref{boundary_score} addresses this problem by adopting a normal reference distribution, ensuring the integrals remain well-defined and free from introducing biases.

Theorem \ref{pro_score} gives expressions of the conditional drift and score functions in terms of conditional expectations. Based on these expressions, we can obtain estimators of $\boldsymbol{b}^* $ and $\boldsymbol{s}^* $ through nonparametric least squares regression.

 \begin{lemma} \label{lemma_loss}
 Assume that $\boldsymbol{b}^*$, $\boldsymbol{s}^* \in L^2(\mathbb{R}^{d+k},(0,1))^d$. Then the drift function $\boldsymbol{b}^*$ and score function $\boldsymbol{s}^*$ can be obtained by minimizing the corresponding loss functions,
\begin{align*}
        \boldsymbol{b}^* = \argmin_{\boldsymbol{b}} \mathcal{L}_b(Y_0, Y_1, X, \boldsymbol{\eta})
        \ \text{ and } \
        \boldsymbol{s}^* = \argmin_{\boldsymbol{s}} \mathcal{L}_s( Y_0, Y_1, X, \boldsymbol{\eta}),
\end{align*}
where
    \begin{align}
        \mathcal{L}_b(Y_0, Y_1, X, \boldsymbol{\eta}) &:=
  {\displaystyle \int_0^1 \mathbb{E} \| \partial_t {\mathcal{I}}(Y_0, Y_1, t)
  + \dot{\gamma}(t)\boldsymbol{\eta} - \boldsymbol{b}(X,Y_t,t) \|^2  \ \mathrm{d}t} ,\label{Lb} \\
        \mathcal{L}_s( Y_0, Y_1, X, \boldsymbol{\eta})&:=
        {\displaystyle \int_0^1 \mathbb{E} \| \gamma(t)^{-1} \boldsymbol{\eta} + \boldsymbol{s}(X, Y_t, t) \|^2  \ \mathrm{d}t}. \label{Ls}
    \end{align}
 \end{lemma}
This lemma shows that $\boldsymbol{b}^*$ and $\boldsymbol{s}^*$ can be identified as solutions to the minimization of two squares losses. This paves the way for using nonparametric regression to estimate these two functions.  We will present the details in Section \ref{estimation}.

\subsection{Additive interpolation} \label{sub_additive_interpolation}
We now focus on an important special case of the general interpolation function when it takes an additive form.

\begin{definition}[Additive interpolation]
An additive interpolation takes the form
\begin{align}\label{interpb}
Y_t = a(t) Y_0 + b(t) Y_1 + \gamma(t) \boldsymbol{\eta}, \ t \in [0, 1],
\end{align}
where $\boldsymbol{\eta} \sim N(\mathbf{0}, \mathbf{I}_d),$ the functions $a$ and $b$ are continuously differentiable on $[0,1]$ and satisfy
\[
a(0) =1, a(1) = 0; \ b(0)=0, b(1)=1,
\]
and the function $\gamma$ satisfies  Assumption 1(b). In addition, $Y_0$, $Y_1$ and $\boldsymbol{\eta}$ are mutually independent.
\end{definition}

The conditional drift function $\boldsymbol{b}^*$ corresponding to (\ref{interpb})
is given by
\begin{align}\label{drif_equa2}
 \boldsymbol{b}^*(\mathbf{x},\mathbf{y},t)
&= \mathbb{E}[\dot{a}(t) Y_0 + \dot{b}(t) Y_1 + \dot{\gamma}(t)\boldsymbol{\eta} \mid X = \mathbf{x}, Y_t = \mathbf{y}], \ t \in (0, 1).
\end{align}
We derive an expression of the score function in terms of the drift function.

\begin{corollary}[Conditional score function]\label{pro_score2}
Suppose that Assumptions \ref{assumption0} and \ref{assumption0b} hold, $Y_0\sim \mathcal{N}(\mathbf{0},\mathbf{I}_d) $ and the interpolation takes the additive form (\ref{interpb}). Denote
$$
A(t):= a(t)[a(t)\dot{b}(t)-\dot{a}(t)b(t)] + \gamma(t)[\gamma(t)\dot{b}(t)-\dot{\gamma}(t)b(t)].
$$
Then the conditional score function can be expressed as
     \begin{align} \label{score_b}
      \boldsymbol{s}^{*}_{\mathbf{x}}(\mathbf{y},t):= \boldsymbol{s}^{*}(\mathbf{x},\mathbf{y},t)
        =\frac{b(t)}{A(t)} \boldsymbol{b}^*(\mathbf{x},\mathbf{y},t) - \frac{\dot{b}(t)}{A(t)} \mathbf{y}, \ t \in (0, 1),
    \end{align}
provided that the function ${b(t)}/{A(t)}$ and ${\dot{b}(t)}/{A(t)}$ are well-defined in $(0,1).$
\end{corollary}

This corollary shows that, for a Gaussian initial value $Y_0$ and an additive interpolation function as given in (\ref{interpb}), an estimator of the conditional score function can be obtained directly based on (\ref{score_b}), provided that an estimator for the conditional drift function is available. We will use a numerical example in Section \ref{sec_experiment} to illustrate this point.

In addition to facilitating estimation, additive interpolation offers stability for the score function at the boundaries when the functions $a(t)$, $b(t)$ and $\gamma(t)$ are appropriately selected.  This can be elucidated through formula (\ref{score_b}) and Theorem \ref{boundary_score}. Notably, the applicability of Theorem \ref{boundary_score} is contingent upon the condition that $\Vert p_{\mathcal{I}(Y_0,Y_1,t) \mid X} - p_{Y_0}\Vert_\infty = o(\gamma(t))$. This condition  imposes restrictions on the support of $Y_1$ and the rate of change of the interpolation function $\mathcal{I}$ at the boundary.
We formalize these conditions as follows.

\begin{assumption} \label{assump_ab}
    (a) The supports of the conditional distributions of $Y_1\mid X=\mathbf{x}$ for every given $\mathbf{x} \in \mathcal{X}$ are uniformly bounded by a finite constant $B_1 > 0$. (b) Functions $a(t)$ and $b(t)$ satisfy $1-a(t)=o(\gamma(t))$ and $ b(t) =  o(\gamma(t))$ as $t$ tends to $0$.
\end{assumption}

Assumption \ref{assump_ab}(a) serves as a regularity condition for the support of $Y_1$ and is particularly applicable to bounded support data, such as image data.  To satisfy Assumption \ref{assump_ab}(b), the functions $a(t)$, $b(t)$ and $\gamma(t)$ must be appropriately selected, which also implies that the coefficients $b(t)/A(t)$ and $\dot{b}(t)/A(t)$ in formula (\ref{score_b}) are well-defined at $t=0$.  Consequently, we derive the following corollary.

\begin{corollary}[Boundary behavior at $t=0$] \label{corollary_boundary_score}
    Suppose that assumptions \ref{assumption0}, \ref{assumption0b}, \ref{assump_ab} are satisfied, $Y_0\sim \mathcal{N}(\mathbf{0},\mathbf{I}_d)$, and the interpolation takes the additive form (\ref{interpb}). Then, $\Vert p_{\mathcal{I}(Y_0, Y_1,t) \mid \mathbf{x}} - p_{Y_0}\Vert_\infty = o(\gamma(t))$ and $\boldsymbol{s}^*_{\mathbf{x}}(\mathbf{y},0) = -\mathbf{y}$, $    \boldsymbol{b}^*_{\mathbf{x}}(\mathbf{y},0) = \mathbb{E}[\partial_t {\mathcal{I}}(\mathbf{y},Y_1, 0) \mid  X = \mathbf{x}].
    $
\end{corollary}

Assuming $Y_1$ is bounded, we further provide examples to illustrate the selection of $a(t)$, $b(t)$ and $\gamma(t)$ and their impact on boundary behaviors, using corollary \ref{corollary_boundary_score}.

 \begin{example}
 \label{ex1}
Let $a(t)=1-t$, $b(t)=t$ and $\gamma(t) = \sqrt{2t(1-t)}, t \in [0, 1].$
The corresponding interpolation process is
\[
Y_t = (1-t) Y_0 + t Y_1 + \sqrt{2t(1-t)}\,\boldsymbol{\eta}, \ t \in [0, 1].
\]
Assumption \ref{assump_ab}(b) can be verified to be satisfied, which implies the stability of the score function at the boundary $t=0$. For the boundary $t=1$, straightforward calculations yield $A(t)=2t^2-t+1.$ Thus, the coefficients $b(t)/A(t)$ and $\dot{b}(t)/A(t)$ in formula (\ref{score_b}) are well-defined at $t=1$. Provided that $\boldsymbol{b}^*$ is well-defined at $t=1$, we conclude that the conditional score function of this interpolation process is stable using Corollary (\ref{pro_score2}).
\end{example}

\begin{example}
\label{ex2}
 Let $a(t)=\cos^2(\frac{\pi}{2}t), b(t)=\sin^2(\frac{\pi}{2} t)$ and
$\gamma(t)=\frac{\sqrt{2}}{2}\sin(\pi t), t \in [0,1].$
The corresponding interpolation process is
\[
Y_t = \cos^2(\frac{\pi}{2}t) Y_0 + \sin^2(\frac{\pi}{2} t) Y_1
+ \frac{\sqrt{2}}{2}\sin(\pi t) \boldsymbol{\eta}, \ t \in [0, 1].
\]

It can be verified that Assumption \ref{assump_ab}(b) is also satisfied, ensuring the stability of the score function at the boundary $t=0$. Additionally, some algebra yields $A(t)=\frac{\pi}{2}\sin(\pi t)[1+\frac{1}{2}\sin(\pi t) -\frac{5}{4}\sin^2(\pi t) + \sin^4(\frac{\pi}{2}t)].$  Provided that $\boldsymbol{b}^*$ is well-defined at $t=1$, we conclude that the conditional score function of this interpolation process is stable using Corollary (\ref{pro_score2}).
 \end{example}

 \begin{example}\label{ex3}
 Let $a(t)=\cos(\frac{\pi}{2}t), b(t)=\sin(\frac{\pi}{2} t)$ and
$\gamma(t)=\sin(\pi t), t \in [0,1].$
The corresponding interpolation process is
\[
Y_t = \cos(\frac{\pi}{2}t)   Y_0 + \sin(\frac{\pi}{2} t)  Y_1
+ \sin(\pi t) \boldsymbol{\eta}, \ t \in [0, 1].
\]
Some algebra yields $A(t)=\frac{\pi}{2}\cos(\frac{\pi}{2} t)[1+\sin^2(\pi t) -
4 \sin^2(\frac{\pi}{2} t) \cos(\pi t)]$. At the boundary $t=0$, Assumption \ref{assump_ab}(b) does not hold and $ |b(t)/A(t)| \to \infty$ as $t \to 1.$ Consequently, the stability of the conditional score function remains unclear.
 \end{example}

\section{Conditional generators} \label{sub_ode}
We propose two conditional generators based on the velocity and distribution field information of the CSI process: conditional stochastic interpolation flow (CSI Flow, an ODE-based generative model) and conditional stochastic interpolation diffusion (CSI Diffusion, an SDE-based generative model).
These models produce Markov processes $\{Z_{t,X}\}_{t\in[0,1]}$ under the condition $X \in \mathcal{X}$, ensuring that $Z_{0,X} \sim \mathbb{P}_{Y_0}$ and $Z_{1,X} \sim \mathbb{P}_{Y_1\mid X}$. The construction of these generative models uses estimators of the conditional drift and score functions. These generative models can be conceptualized as a two-step approach, including a training step and a sampling step.

\begin{figure}[H]
    \centering
    \includegraphics[width=\textwidth]{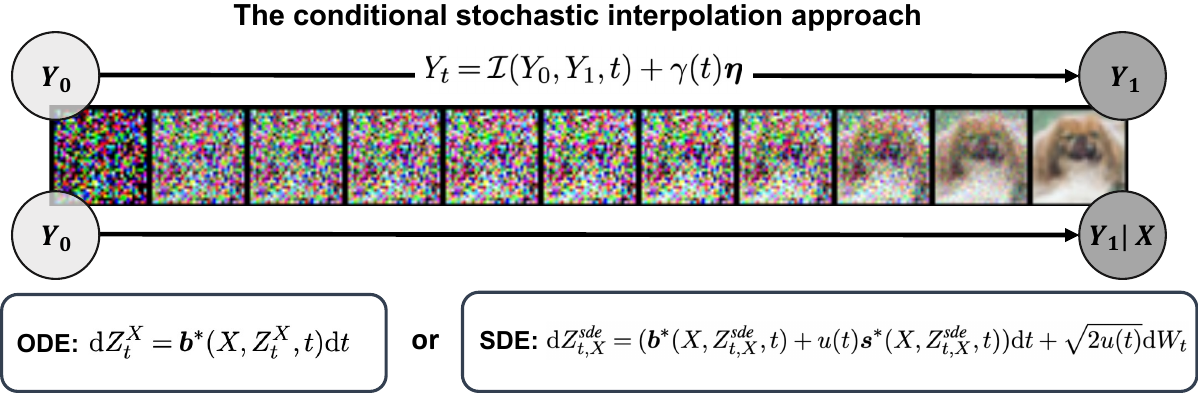}
    \caption{{\color{blue} Illustration of the proposed CSI approach.}}
    \label{fig_overview}
    \end{figure}

As illustrated in Figure \ref{fig_overview}, the training step of our CSI-based generative models involves using the CSI process to estimate the conditional drift and score functions. This pre-sampling step enables us to capture the key characteristics of the conditional distribution
$\mathbb{P}_{Y_1 \mid X}$.  In the sampling step, we construct a Markov process, including a diffusion process $Z_{t,X}^{\sde}$ or the corresponding probability flow $Z_{t,X}^{\ode}$, using either the estimated conditional drift or score function.

The processes $Z_{t,X}^{\sde}$ and  $Z_{t,X}^{\ode}$  possess the marginal preserving property, i.e., their time-dependent conditional distributions follow $\mathbb{P}_{Y_t|X}$, despite their different transfer paths. Consequently, the processes $Z_{1,X}^{\sde}$ and $Z_{1,X}^{\ode}$ can be used to generate samples from $\mathbb{P}_{Y_1 \mid X}$. In practice, the Euler-Maruyama and stochastic Runge-Kutta methods \citep{platen2010numerical} are commonly employed for solving the corresponding ODEs/SDEs.

\subsection{ODE-based generative model} We present an ODE-based conditional generative model with CSI flow. In the training step, we estimate the drift function by solving a nonparametric least squares regression problem (\ref{Lb}).
We then generate new samples by solving the ordinary differential equation (ODE) using the estimated drift function.

\begin{definition}[ODE Flow] \label{def_flow}
Suppose $Y_0 \sim \mathbb{P}_{Y_0}$, and $ Y_1 \mid X \sim \mathbb{P}_{Y_1 \mid X} $, where $Y_0$ is independent of $(X, Y_1)$. The conditional stochastic flow $Z_{t,X}^{\ode}$
is defined by an ordinary differentiable model,
\begin{align} \label{flow_equa}
 \mathrm{d} Z_{t,X}^{\ode} = \boldsymbol{b}^*(X, Z_{t,X}^{\ode}, t)\mathrm{d}t, \ t \in (0,1) \text{ with }  Z_{0,X}^{\ode}\sim P_{Y_0}.
\end{align}
Equation (\ref{flow_equa}) describes a deterministic motion, which starts from a random initial value
$Y_0$ and ends at a random value
that follows the conditional distribution
$\mathbb{P}_{Y_1|X}.$
\end{definition}

A simple example of CSI Flow is the conditional rectified flow.
The rectified flow is introduced by  \citet{liu2022flow} and
uses a linear interpolation between $Y_0$ and $Y_1$.

\begin{example}[Conditional rectified flow]\label{def_rectflow}
Suppose $Y_0 \sim \mathbb{P}_{Y_0}$, and $ Y_1 \mid X \sim \mathbb{P}_{Y_1 \mid X}$, where $Y_0$ is independent of $X$. The conditional rectified flow is realized by setting $\mathcal{I}(\mathbf{y}_0, \mathbf{y}_\mathbf{x}, t) = t\mathbf{y}_\mathbf{x} + (1-t)\mathbf{y}_0$ and $\gamma(t) = 0$.
The corresponding ODE Flow for sampling is
    \begin{align*}
        \mathrm{d} Z_{t,X} = \boldsymbol{b}_{\text{RF}}(X,Z_{t,X}, t) \mathrm{d}t,\
          t \in (0,1) \text{ with }
        Z_{0,X}\sim P_{Y_0},
    \end{align*}
    where the drift function $\boldsymbol{b}_{\text{RF}}$
can be obtained by solving
 \begin{align*}
        \min_{\boldsymbol{b}} {\int_0^1 \mathbb{E} \| Y_1 - Y_0 - \boldsymbol{b}(X, Y_t,t) \|^2   \ \mathrm{d}t}, \  {Y}_t = t {Y}_1 +(1-t) Y_0.
 \end{align*}
\end{example}

\subsection{SDE-based generative model} \label{sub_sde}

In this subsection, we present the SDE-based generative model, in which a stochastic noise is included for introducing randomness and capturing the inherent uncertainty in the system being modeled.  An adaptive function $u(t)$ is also incorporated to adjust the dynamic noise level during the generation process.

\begin{definition}[Conditional stochastic interpolation diffusion] \label{def_process}
    Suppose that $Y_0 \sim \mathbb{P}_{Y_0}$, and $ Y_1 \mid X \sim \mathbb{P}_{Y_1 \mid X}$, where $Y_0$ is independent of $X$.  The conditional stochastic interpolation diffusion
    $Z_{t,X}^{\sde}$  induced from $\mathbb{P}_{Y_0}$ and $\mathbb{P}_{Y_1|X}$
    is defined  as
    \begin{align} \label{sde_equ}
        \mathrm{d} Z_{t,X}^{\sde} = \boldsymbol{b}_u^*(X, Z_{t,X}^{\sde}, t) \mathrm{d}t + \sqrt{2u(t)} \mathrm{d} W_t, \  t \in (0,1) \text{ with }  Z_{0,X}^{\sde}\sim P_{Y_0},
    \end{align}
    where $\{W_t,t\in[0,1]\}$ is a Wiener process, and $\boldsymbol{b}_u^* : \mathcal{X} \times \mathbb{R}^{d} \times [0,1]  \rightarrow \mathbb{R}^{d}$ is defined as $$\boldsymbol{b}_u^*(X, Z_t, t) := \boldsymbol{b}^*( X, Z_t, t) + u(t)\boldsymbol{s}^*(X, Z_t,t).$$
\end{definition}
The SDE-based generative models incorporate the diffusion term $\sqrt{2u(t)}\mathrm{d}W_t$ into the process, which differs from the model proposed in \citet{albergo2023stochastic},  where the variance of noise remains constant. The diffusion function $u(t)$ in Definition \ref{def_process} offers an enhanced flexibility and adaptability to temporal changes. With a proper adaptive diffusion function $u(t)$, the drift function involved in SDE process
can be stabilized. Note that
\begin{align*}
	\boldsymbol{b}_u(\mathbf{x},\mathbf{y},t)
	=\boldsymbol{b}^*( \mathbf{x},\mathbf{y},t) - \frac{u(t)}{\gamma(t)} \mathbb{E}\left(\boldsymbol{\eta} \mid X= \mathbf{x}, Y_t=\mathbf{y} \right),
\end{align*}
then setting $u(t)=\gamma(t)$ yields
\begin{align*}
	\boldsymbol{b}_u(\mathbf{x},\mathbf{y},t) = \boldsymbol{b}^*(\mathbf{x},\mathbf{y},t) - \mathbb{E}[\boldsymbol{\eta} \mid X=\mathbf{x}, Y_t=\mathbf{y}],
\end{align*}
 which is more stable when $\mathbb{E}[\boldsymbol{\eta} \mid X=\mathbf{x}, Y_t=\mathbf{y}]$ is controllable.

The SDE given in (\ref{sde_equ}) describes a different motion from the process (\ref{cond_stoch_interp_equa}). This process also satisfies the Fokker-Planck equation \citep[Section 3.3.3]{ottinger2012stochastic}.

\begin{example}[Conditional stochastic linear interpolation diffusion]
Suppose $Y_0 \sim \mathbb{P}_{Y_0}$, and $ Y_1 \mid X \sim \mathbb{P}_{Y_1 \mid X}$,  where $Y_0$ is independent of $X$. Let $\gamma(t)=u(t) = \sqrt{2t(1-t)}.$
A conditional stochastic linear interpolation diffusion induced from $\mathbb{P}_{Y_0}$ and
$\mathbb{P}_{Y_1 \mid X}$ is
{\small
    \begin{align*}
        \mathrm{d}Z_{t,X} = \boldsymbol{b}(X,Z_{t,X}, t)\mathrm{d}t +\sqrt{2t(1-t)}\boldsymbol{s}(X,Z_{t,X}, t)\mathrm{d}t + (8 t(1-t))^{1/4} d W_t,\text{ } t \in (0,1),
    \end{align*}
} where  $Z_{0,X}\sim P_{Y_0}$ and the conditional drift and score functions $\boldsymbol{b} $ and $\boldsymbol{s} $ can be obtained by solving the minimization problems:
    \begin{align*}
        &\min_{\boldsymbol{b}} {\displaystyle \int_0^1 \mathbb{E} \| Y_1 - Y_0 + \frac{2-4t}{\sqrt{2 t(1-t)}}\boldsymbol{\eta} - \boldsymbol{b}(X, Y_t,t) \|^2  \mathrm{d}t},   \\
        &\min_{\boldsymbol{s}} {\displaystyle \int_0^1 \mathbb{E} \| \frac{1}{\sqrt{2t(1-t)}} \boldsymbol{\eta} + \boldsymbol{s}(X, Y_t,t) \|^2  \mathrm{d}t},
    \end{align*}where $Y_t = t Y_1 +(1-t) Y_0 + \sqrt{2 t (1-t)}\boldsymbol{\eta}$.
\end{example}
Below, we consider the estimation of the conditional drift and score functions based on
a finite sample.

\section{Estimation of the conditional drift and score functions}
\label{estimation}
We use a neural network function class for estimating the conditional drift and score functions based on a finite set of observations. At the population level, the true drift and score functions, denoted as $\boldsymbol{b}^*$ and $\boldsymbol{s}^*$ respectively, can be identified across the interval $(0,1)$, as demonstrated in Lemma \ref{lemma_loss}. However, when working with empirical data, accurately estimating the entire processes $\{\boldsymbol{b}^*(\cdot, t), t \in (0, 1)\}$ and $\{\boldsymbol{s}^*(\cdot, t), t \in (0, 1)\}$ simultaneously is difficult. This is particularly so for the score function $\boldsymbol{s}^*$, which may exhibit unstable behavior at the boundary points. Consequently, we adopt a pointwise estimation approach for both functions.

We define respectively the $L^2$ risks for any $\boldsymbol{b}$ and $\boldsymbol{s}$ at time $t$ by
\begin{align}
    R^b_t(\boldsymbol{b}) &:= \mathbb{E} \left\| \partial_t {\mathcal{I}}(Y_0, Y_1, t) + \dot{\gamma}(t) \boldsymbol{\eta} - \boldsymbol{b}(X, Y_t,t) \right\|^2, \label{Rbt} \\
    R^s_t(\boldsymbol{s}) &:= \mathbb{E}   \left\| \gamma(t)^{-1} \boldsymbol{\eta} + \boldsymbol{s}(X, Y_t ,t) \right\|^2, \text{ for all }t \in (0, 1),  \label{Rst}
\end{align}
where $Y_t=\mathcal{I}(Y_0, Y_1,t)+\gamma(t)\boldsymbol{\eta}$ and the expectation is taken with respect to $\{Y_0, X, Y_1, \boldsymbol{\eta}\}.$

Assume that $\boldsymbol{b}^*$, $\boldsymbol{s}^* \in L^2(\mathbb{R}^{d+k},(0,1))^d$. Based on the basic properties of conditional expectation,
 the minimums of (\ref{Rbt}) and (\ref{Rst}) are achieved at
\begin{align}
    &\boldsymbol{b}^*(\mathbf{x},\mathbf{y},t) = \mathbb{E}\left[ \partial_t \mathcal{I}(Y_0, Y_1, t) + \dot{\gamma}(t) \boldsymbol{\eta} \vert X = \mathbf{x}, Y_t = \mathbf{y}\right] = \arg \min_{\boldsymbol{b}} R^b_t(\boldsymbol{b}), \label{bstar} \\
    &\boldsymbol{s}^*(\mathbf{x},\mathbf{y},t)
    = \arg \min_{\boldsymbol{s}} R^b_t(\boldsymbol{s}),  \text{ for all }t \in (0, 1),  \label{sstar}
\end{align}
respectively for any $\mathbf{x} \in \mathcal{X}$ and $\mathbf{y} \in \mathbb{R}^d$.

For the $\boldsymbol{\kappa}^*$ defined in (\ref{kappa1}),
we define the population risk function for estimating $\boldsymbol{\kappa}^*$ as
\begin{align} \label{kstar}
    R^\kappa_t(\boldsymbol{\kappa}) := \mathbb{E} \| \boldsymbol{\eta} + \boldsymbol{\kappa}(X, Y_t ,t) \|^2, \ t \in (0, 1).
\end{align}
It is clear that $R^s_t(\boldsymbol{s}) = \gamma(t)^{-1} R^\kappa_t(\boldsymbol{\kappa})$
for $ \boldsymbol{\kappa} (\mathbf{x},\mathbf{y},t) := \gamma(t) \boldsymbol{s}(\mathbf{x},\mathbf{y},t).$
Therefore, for estimating the conditional score, we only need to estimate $\boldsymbol{\kappa}^*$, since $\gamma(t)$ is known.

Let $S:= \{(\mathbf{y}_{0,j}, (\mathbf{x}_{1,k}, \mathbf{y}_{1,k}), \boldsymbol{\eta}_{h})\}_{j,k,h}^{m,n,H}$ be a finite sample with independent and identically distributed observations, where $n$, $m$, and $H$ correspond to the number of samples of $Y_0$, $(X, Y_1)$, and $\boldsymbol{\eta}$, respectively. The sample sizes $m$, $n$, and $H$ may not be identical.
It can be shown that the effective sample size of $S$ in fact depends on $\min \{m, n, H\}$ in terms of the theoretical guarantee. Without loss of generality, we set $m = n = H$ as the sample size of $S$ and rewrite the indices of the sample as $S = \{D_k = (\mathbf{y}_{0,k}, (\mathbf{x}_k, \mathbf{y}_{1,k}), \boldsymbol{\eta}_k)\}_{k=1}^{n}$. Based on the effective sample $S$, for any given $t \in (0,1)$, we define the empirical risks for the drift and score functions as:
\begin{align} \label{loss}
    R_{t,n}^b(\boldsymbol{b}) &:= \frac{1}{n} \sum_{k=1}^n \left\| \partial_t \mathcal{I}(\mathbf{y}_{0,k}, \mathbf{y}_{1,k},t) + \dot{\gamma}(t) \boldsymbol{\eta}_k -\boldsymbol{b}(\mathbf{x}_k,\mathbf{y}_{t,k},t) \right\|^2, \\
    R_{t,n}^s(\boldsymbol{s}) &:= \frac{1}{n} \sum_{k=1}^n \gamma(t)^{-1} \left\| \boldsymbol{\eta}_k + \gamma(t)\boldsymbol{s}(\mathbf{x}_k, \mathbf{y}_{t,k},t) \right\|^2 \nonumber \\
    &= \gamma(t)^{-1}  \frac{1}{n} \sum_{k=1}^n \left\| \boldsymbol{\eta}_k + \boldsymbol{\kappa}(\mathbf{x}_k, \mathbf{y}_{t,k},t) \right\|^2 =: \gamma(t)^{-1} R_{t,n}^\kappa(\boldsymbol{\kappa}),
\end{align}
where $\mathbf{y}_{t,k} = \mathcal{I}(\mathbf{y}_{0,k}, \mathbf{y}_{1,k}, t) + \gamma(t) \boldsymbol{\eta}_k$. We minimize the empirical risks over a class of neural network functions
$\mathcal{F}_n = \{\boldsymbol{f}: \mathbb{R}^{k+d+1} \rightarrow \mathbb{R}^{d} \}$
described below.

We consider the commonly used feedforward neural networks, the multi-layer perceptrons (MLPs).
An MLP activated by the Rectified Linear Unit (ReLU) function can be described as a composite function $\boldsymbol{f}_{\boldsymbol{\theta}}: \mathbb{R}^{k+d+1} \rightarrow \mathbb{R}^d$, parameterized by $\boldsymbol{\theta},$ as follows:
$$\boldsymbol{f}_{\boldsymbol{\theta}} = \mathcal{L}_{\mathcal{D}} \circ \sigma \circ \mathcal{L}_{\mathcal{D}-1} \circ \sigma \circ \cdots \circ \sigma \circ \mathcal{L}_1 \circ \sigma \circ \mathcal{L}_0,$$
where $\sigma(x) = \max(0, x)$, $x \in \mathbb{R}$, represents the ReLU activation function applied element-wise to vectors. The function $\mathcal{L}_i(\mathbf{x}) = W_i \mathbf{x} + \mathbf{b}_i$ denotes the $i$-th linear function with a weight matrix $W_i \in \mathbb{R}^{p_{i+1} \times p_i}$ and a bias vector $\mathbf{b}_i \in \mathbb{R}^{p_{i+1}}$. Here, $p_i$ indicates the width (the number of neurons or computational units) of the $i$-th layer, for $i = 0, 1, \ldots, \mathcal{D}$. The notation $\boldsymbol{\theta}$ is used to collectively denote the weights and biases.

For the MLP $\boldsymbol{f}_{\boldsymbol{\theta}}$ defined above, we use $\mathcal{D}$ to denote its depth, $\mathcal{W}=\max \left\{p_1, \ldots, p_{\mathcal{D}}\right\}$ denote its width, $\mathcal{U}=\sum_{i=1}^{\mathcal{D}} p_i$ denote the number of neurons, and $\mathcal{S}=\sum_{i=0}^{\mathcal{D}-1}[p_i\times p_{i+1}+p_{i+1}]$ denote the size. We denote
the class of ReLU activated neural networks by
\begin{align*}
    \mathcal{F}_{\mathcal{D}, \mathcal{U}, \mathcal{W}, \mathcal{S}, \mathcal{B}(t)} := \{ \boldsymbol{f}_{\boldsymbol{\theta}}:& W_i\in \mathbb{R}^{p_{i+1} \times p_i},\mathbf{b}_i \in \mathbb{R}^{p_{i+1}}, i\in\{0,\ldots,\mathcal{D}\} 
    {\rm\ and\ } \Vert \boldsymbol{f}_{\boldsymbol{\theta}}(\cdot,t) \Vert_\infty \leq \mathcal{B}(t)  \},  
\end{align*}
$t \in [0,1],$ where $\mathcal{B}: [0,1] \rightarrow [1,+\infty)$ is a function of $t.$

In general, the estimation of  $\boldsymbol{b}^*$ and $\boldsymbol{s}^*$ is carried out separately. Specifically, the score function is estimated through $\gamma(t)^{-1}\boldsymbol{\kappa}^*$. So we denote the corresponding estimating neural networks by $\boldsymbol{b}_{\boldsymbol{\theta}}  \in \mathcal{F}_n$, $ \boldsymbol{s}_{\boldsymbol{\theta}^\prime} := \gamma(t)^{-1} \boldsymbol{\kappa}_{\boldsymbol{\theta}^\prime}$ with $  \boldsymbol{\kappa}_{\boldsymbol{\theta}^\prime} \in \mathcal{F}_n^\prime$, with parameters $\boldsymbol{\theta}$ and $\boldsymbol{\theta}^\prime,$ respectively, where $\mathcal{F}_n:= \mathcal{F}_{\mathcal{D}, \mathcal{U}, \mathcal{W}, \mathcal{S}, \mathcal{B}_b(t)}$ and $\mathcal{F}_n^\prime:= \mathcal{F}_{\mathcal{D}^\prime, \mathcal{U}^\prime,\mathcal{W}^\prime, \mathcal{S}^\prime, \mathcal{B}_\kappa(t)}$  for some integers $\mathcal{D}, \mathcal{U}, \mathcal{W}, \mathcal{S},\mathcal{D}^\prime, \mathcal{U}^\prime,\mathcal{W}^\prime, \mathcal{S}^\prime$ and
functions $\mathcal{B}_b(t)$ and $\mathcal{B}_\kappa(t)$. For ease of presentation, we omit the subscript for the dependence of these hyperparameters on the sample size $n$ in subsequent discussions.

With a finite sample $S$, we define the estimators of the drift and score functions as the empirical risk minimizers (ERMs):
\begin{align} \label{bhat}
    &\hat{\boldsymbol{b}}_n(\cdot,t) \in \arg \min_{\boldsymbol{f} \in \mathcal{F}_n} R_{t,n}^b(\boldsymbol{f}), \\ \label{shat}
    & \hat{\boldsymbol{s}}_n(\cdot,t) := \gamma(t)^{-1}\hat{\boldsymbol{\kappa}}_n(\cdot,t)
    \text{ with }
    \hat{\boldsymbol{\kappa}}_n(\cdot,t) \in \arg \min_{\boldsymbol{f} \in \mathcal{F}_n^{\prime}} R_{t,n}^\kappa(\boldsymbol{f}),
\end{align}
for any $ t \in (0,1)$. With the estimated drift and score functions, we generate samples based on the following CSI Flow or CSI Diffusion:
   \begin{align} \label{flow_equa1}
        \mathrm{d}\widehat{Z}_{t,X}^{\ode} = \hat{\boldsymbol{b}}_n(X, \widehat{Z}_{t,X}^{\ode}, t)\mathrm{d}t, \  t \in (0,1),
    \end{align}
or
       \begin{align} \label{sde_equ1}
        \mathrm{d} \widehat{Z}_{t,X}^{\sde} =
         \hat{\boldsymbol{b}}_n( X, \widehat{Z}_{t,X}^{\sde}, t) \mathrm{d}t+ u(t)\hat{\boldsymbol{s}}_n(X, \widehat{Z}_{t,X}^{\sde},t) \mathrm{d}t+ \sqrt{2u(t)} \mathrm{d} W_t, \ t \in (0,1),
    \end{align}
with  $\widehat{Z}_{0,X}^{\ode}, \widehat{Z}_{0,X}^{\sde} \sim \mathbb{P}_{Y_0}.$

Finally, according to Theorem \ref{pro_score},  for $\mathcal{I}(Y_0,Y_1,t) := a(t) Y_0 + b(t) Y_1$ with $Y_0 \sim \mathcal{N}(\boldsymbol{0}, I_d),$  if an estimator of the drift function $\boldsymbol{b}^*$ is available, we can then obtain an estimator of the conditional score function directly by using (\ref{score_b}). This circumvents the necessity of solving a nonparametric regression problem for estimating $\boldsymbol{s}^*$. Such an approach bypasses the challenges associated with estimating a potentially unstable score function and reduces computational costs. We will illustrate this approach in the numerical example in Section \ref{sec_experiment}.

\section{Error bounds for the estimated conditional drift and score functions}
\label{drift-score-error}

As can be seen from equations (\ref{flow_equa1}) and (\ref{sde_equ1}), the quality of the generated samples is contingent upon the accuracy of the estimated conditional drift and score functions $\hat{\boldsymbol{b}}_n$ and $\hat{\boldsymbol{s}}_n$. In this section, we establish error bounds for $\hat{\boldsymbol{b}}_n$ and $\hat{\boldsymbol{s}}_n$. These bounds lead to the results in Section \ref{sec_theory}, which pertain to the error bounds for the distributions of the generated samples from equations (\ref{flow_equa1}) and (\ref{sde_equ1}).

We evaluate the quality of $\hat{\boldsymbol{b}}_n$ or $\hat{\boldsymbol{s}}_n$ via the corresponding excess risk, defined as the difference between the $L^2$ risks of ground truth and the estimator,
\begin{align*}
    R^b_t(\hat{\boldsymbol{b}}_n)- R^b_t(\boldsymbol{b}^*)
    &= \mathbb{E} \left[\| \partial_t \mathcal{I}(Y_0, Y_1, t) + \dot{\gamma}(t) \boldsymbol{\eta} - \hat{\boldsymbol{b}}_n(X, Y_t,t) \|^2 \mid S \right]\\
    & \ \ \ \ \ \ - \mathbb{E} \left\| \partial_t \mathcal{I}(Y_0, Y_1, t) + \dot{\gamma}(t) \boldsymbol{\eta}
    - \boldsymbol{b}^*(X, Y_t,t) \right\|^2 \\
    &= \mathbb{E} \left[\Vert \hat{\boldsymbol{b}}_n(X,Y_t,t) - \boldsymbol{b}^*(X,Y_t,t) \Vert^2 \mid S\right] \\
    R^s_t(\hat{\boldsymbol{s}}_n)- R^s_t(\boldsymbol{s}^*)
    &= \mathbb{E} \left[\| \gamma(t)^{-1} \boldsymbol{\eta} + \hat{\boldsymbol{s}}_n(X, Y_t,t) \|^2 \mid S\right]  - \mathbb{E} \left\| \gamma(t)^{-1} \boldsymbol{\eta} + \boldsymbol{s}^*(X, Y_t,t) \right\|^2\\
    &= \mathbb{E} \left[\Vert \hat{\boldsymbol{s}}_n(X,Y_t,t) - \boldsymbol{s}^*(X,Y_t,t) \Vert^2 \mid S \right].
\end{align*}

To establish upper bounds for the excess risks,  we first decompose the excess risks into two terms, namely, the stochastic error and the approximation error.
\begin{lemma} \label{lemma1}
    Suppose that Assumptions \ref{assumption0}, \ref{assumption0b} and \ref{assump_bound} hold. Given random sample  $S = \{D_k = (\mathbf{y}_{0,k}, (\mathbf{x}_k, \mathbf{y}_{1,k}), \boldsymbol{\eta}_k)\}_{k=1}^{n}$ and two specific classes of functions $\mathcal{F}_n$, $\mathcal{F}_n^\prime$, we have
    \begin{align*}
        \mathbb{E}_S \{ R^b_t(\hat{\boldsymbol{b}}_n)-R^b_t(\boldsymbol{b}^*)\}  \leq& \mathbb{E}_S\{R^b_t(\hat{\boldsymbol{b}}_n) - 2R^b_{t,n}(\hat{\boldsymbol{b}}_n) + R^b_t(\boldsymbol{b}^*)\} \\
        &+ 2 d \inf_{\boldsymbol{f} \in \mathcal{F}_n }  \mathbb{E} \Vert \boldsymbol{f}(X,Y_t,t) - \boldsymbol{b}^*(X,Y_t,t) \Vert^2, \\
        \mathbb{E}_S \{ R^s_t(\hat{\boldsymbol{s}}_n)-R^s_t(\boldsymbol{s}^*)\}
        \leq& \gamma(t)^{-1}  \mathbb{E}_S\{R^\kappa_t(\hat{\boldsymbol{\kappa}}_n) - 2R^\kappa_{t,n}(\hat{\boldsymbol{\kappa}}_n) + R^\kappa_t(\boldsymbol{\kappa}^*)\} \\
        &+ 2 d \gamma(t)^{-1} \cdot \inf_{\boldsymbol{f} \in \mathcal{F}_n^\prime} \mathbb{E} \Vert \boldsymbol{f}(X,Y_t,t) - \boldsymbol{\kappa}^*(X,Y_t,t) \Vert^2
    \end{align*}
     for $t\in(0,1).$
\end{lemma}

Lemma \ref{lemma1} shows that the excess risk can be bounded by the sum of stochastic error, approximation error.
The stochastic error arises from the randomness of data. The approximation error results from using neural networks to approximate the target function in a unbounded domain. Decomposing the total error into stochastic error and approximation error terms has been employed in the context of deep nonparametric estimation studies \citep{schmidt2020nonparametric, shen2022approximation}, while our Lemma \ref{lemma1} enables random vectors $\mathcal{I}(Y_0, Y_1,t)$, $\partial_t \mathcal{I}(Y_0, Y_1,t)$ to be unbounded encompassing normal and other practical distributions.

\subsection{Stochastic error}
To analyze the stochastic error of the estimated conditional drift and score functions with multi-dimensional output, we propose a new method for augmenting a function class, which facilitates the transformation of high-dimensional problems into more manageable low-dimensional ones.

Specifically, for neural networks in $\mathcal{F}_{n}$ with multi-dimensional output, for $i=1,\ldots,d,$ we let $\mathcal{F}_{ni}:= \{\boldsymbol{f}^{(i)}: \mathbb{R}^{k+d+1} \rightarrow \mathbb{R} \vert \boldsymbol{f} \in \mathcal{F}_{n} \}$ denote the neural networks in $\mathcal{F}_{n}$ with univariate output restricted to the $i$-th component of the original output.
For $i=1,\ldots,d$, we let
$$\tilde{\mathcal{F}}_{ni}:= \{\boldsymbol{f}^{(i)}: \mathbb{R}^{k+d+1} \rightarrow \mathbb{R}\mid \boldsymbol{f}^{(i)}\in {\mathcal{F}}_{ni} \}$$
be the class of neural networks with the same architecture as those in $\mathcal{F}_{ni}$ with all possible weight matrices and bias vectors.  It is worth noting that $\mathcal{F}_{ni}$ is related with $\mathcal{F}_{nj}$ for $i\not=j$ in the sense that functions in these two classes may share the same parameter in layers excluding the output one.
In contrast, $\tilde{\mathcal{F}}_{ni}$ is independent of $\tilde{\mathcal{F}}_{nj}$ for $i\not=j$. As a consequence, $\mathcal{F}_{ni} \subset \tilde{\mathcal{F}}_{ni}$ and $\mathcal{F}_n\subseteq \mathcal{F}_{n1} \times \ldots \times \mathcal{F}_{nd} \subset \tilde{\mathcal{F}}_{n1} \times \ldots \times \tilde{\mathcal{F}}_{nd}$.

 Similarly, we can define function classes
 $$\tilde{\mathcal{F}}^\prime_{ni}:=
  \{\boldsymbol{f}^{(i)}: \mathbb{R}^{k+d+1} \rightarrow \mathbb{R}\mid 
  \boldsymbol{f}^{(i)}\in\tilde{\mathcal{F}}^\prime_{ni} \}
  $$
  with $\mathcal{F}_n^\prime \subset \tilde{\mathcal{F}}_{n1}^\prime \times \ldots \times \tilde{\mathcal{F}}_{nd}^\prime.$
This approach allows us to decompose a high-dimensional function into a collection of one-dimensional functions, each belonging to a corresponding augmented function class.

We need to make assumptions on the tail properties of the interpolation function $\mathcal{I}(X,Y_t,t).$
For this purpose, we first state the definition of a sub-Gaussian condition as given below.

\begin{definition}[Sub-Gaussian condition] \label{def_subgauss}
    Given a constant $B \in \mathbb{R}^+$, the random vector $Z=(Z_1,\ldots,Z_d)^\top \in \mathbb{R}^d$ satisfies \textit{$B$-sub-Gaussian condition}, if $\mathbb{E}[\exp(\lambda \vert Z_d\vert )] \leq 2\exp\left(\frac{1}{2} B^2 \lambda^2\right)$, for all $ \lambda \in \mathbb{R}$, $i = 1, ..., d$.
\end{definition}

We also need assumption on the boundedness of both the drift function and the score function, stipulating that for given $t \in (0,1)$, $\boldsymbol{b}(X,Y_t,t)$ and $\boldsymbol{\kappa}(X,Y_t,t)$ is bounded.

\begin{assumption} \label{assump_bound}
    (a) There exists functions $\mathcal{B}_I(t), \mathcal{B}_{I^\prime}(t) : [0,1] \rightarrow (0,+\infty)$ such that random vectors $\mathcal{I}(Y_0,Y_1,t)$, $\partial_t \mathcal{I}(Y_0,Y_1,t)$ satisfies $\mathcal{B}_I(t)$-sub-Gaussian and $\mathcal{B}_{I^\prime}(t)$-sub-Gaussian conditions for all $t \in [0,1]$.
    (b) The drift and the score function satisfy  $\|\boldsymbol{b}^{*}(\cdot,t)\|_{\infty} \leq \mathcal{B}_b(t)$ and $\|\boldsymbol{\kappa}^{*}(\cdot,t)\|_{\infty} \leq \mathcal{B}_\kappa(t)$ for all $t \in (0,1)$, with $\mathcal{B}_b(t) > |\dot{\gamma}(t)|+\mathcal{B}_{I^\prime}(t)$ where $\mathcal{B}_b(t)$ and $\mathcal{B}_\kappa(t)$ are the boundary functions for the network classes defined in Section 4.
\end{assumption}

It is important to note that the constants $B_I(t)$ and $B_{I^{\prime}}(t)$ in the sub-Gaussian conditions depend on $t$. Moreover, the upper bounds for the drift and score functions vary with $t$, encompassing scenarios where the drift and score functions may be unbounded and thus covering a fairly general case.

We now present  upper bounds for the stochastic errors of the estimators of $\boldsymbol{b}^*$ and $\boldsymbol{\kappa}^*.$

\begin{lemma}  \label{stoc_b}
    Let $\boldsymbol{b}^*$ and $\boldsymbol{s}^*$ be the target functions defined in (\ref{bstar}) and (\ref{sstar})  under the conditional stochastic interpolation model. Suppose that Assumptions \ref{assumption0}, \ref{assumption0b} and \ref{assump_bound} hold. Then for any $t \in (0,1)$, the ERMs $\hat{\boldsymbol{b}}_n$ and $\hat{\boldsymbol{\kappa}}_n$ defined in (\ref{bhat}) and (\ref{shat}) satisfy
    \begin{align*}
            \mathbb{E}_S\{R^b_t(\hat{\boldsymbol{b}}_n ) - 2R^b_{t,n}(\hat{\boldsymbol{b}}_n ) + R^b_t(\boldsymbol{b}^* )\}
        &\leq c_0 \mathcal{B}^5_b(t) d(\log n)^3 \frac{\mathcal{S D} \log (\mathcal{S})}{n},\\
            \mathbb{E}_S\{R^\kappa_t(\hat{\boldsymbol{\kappa}}_n ) - 2R^\kappa_{t,n}(\hat{\boldsymbol{\kappa}}_n ) + R^\kappa_t(\boldsymbol{\kappa}^* )\}
            &\leq c_0^\prime \mathcal{B}^5_\kappa(t) d(\log n)^3 \frac{\mathcal{S^\prime D^\prime} \log (\mathcal{S^\prime})}{n}
    \end{align*}
    for $n \geq \max_{i = 1, ..., d} \{\operatorname{Pdim}( \tilde{\mathcal{F}}_{ni}), \operatorname{Pdim}( \tilde{\mathcal{F}}_{ni}^\prime ) \} / 2$, where $c_0, c_0^\prime>0$ are constants independent of $d, n, \mathcal{D}, \mathcal{W}, \mathcal{S}, \mathcal{D}^\prime, \mathcal{W}^\prime$ , and $\mathcal{S}^\prime$.
\end{lemma}

\subsection{Approximation error}
 The approximation error is related to the expressive power of the neural networks and
 the properties such as smoothness of the target functions $\boldsymbol{b}^{*(i)}$, $\boldsymbol{\kappa}^{*(i)}$, $ i = 1,\ldots, d$.  We restrict our discussion to the case where the components of the target functions belong to the class of locally Hölder smooth functions with a smoothness index $\beta>0$.
 
{\color{black}\begin{definition} \label{holder}[Locally Hölder smooth class]
    Suppose $\beta=\lfloor\beta\rfloor+r>0, r \in(0,1]$, where $\lfloor\beta\rfloor$ denotes the largest integer strictly smaller than $\beta$ and $\mathbb{N}_0$ denotes the set of non-negative integers. For a finite constant $B_0>0$, the locally Hölder class of functions $\mathcal{H}^\beta_{\text{loc}}\left(\mathbb{R}^{k+d}, B_0\right)$ is defined as
{\small
    \begin{align*}
        \mathcal{H}^\beta_{\text{loc}}\left(\mathbb{R}^{k+d}, B_0\right)
    & = \Big\{f:\mathbb{R}^{k+d} \rightarrow \mathbb{R}: \max _{\|\boldsymbol{\alpha}\|_1 \leq \lfloor\beta\rfloor, \mathbf{x} \in E }\left\|D^{\boldsymbol{\alpha}} f(\mathbf{x})\right\|_{\infty} \leq B_0, \\
     &\max _{\|{\boldsymbol{\alpha}}\|_1=\lfloor\beta\rfloor} \sup _{\mathbf{x} \neq \mathbf{y}, \mathbf{x},\mathbf{y} \in E } \frac{\left\vert D^{\boldsymbol{\alpha}} f(\mathbf{x})-D^{\boldsymbol{\alpha}} f(\mathbf{y})\right\vert}{\|\mathbf{x}-\mathbf{y}\|^r} \leq B_0 \text{, for any compact set } E \in \mathbb{R}^{k+d} \Big\},
    \end{align*}
}
where $D^{\boldsymbol{\alpha}}$ is the differential operator and
${\boldsymbol{\alpha}}=({\boldsymbol{\alpha}}^{(1)}, \ldots, {\boldsymbol{\alpha}}^{(k+d)})$ is a multi-index of order $\Vert {\boldsymbol{\alpha}} \Vert_1$ with
$D^{\boldsymbol{\alpha}} f=\frac{\partial^{\Vert {\boldsymbol{\alpha}} \Vert_1} f}{\partial \mathbf{x}_1^{{\boldsymbol{\alpha}}^{(1)}} \ldots \partial \mathbf{x}_d^{{\boldsymbol{\alpha}}^{(k+d)}}}.$ Here $\Vert {\boldsymbol{\alpha}} \Vert_1=|{\boldsymbol{\alpha}}^{(1)}|+\cdots + |{\boldsymbol{\alpha}}^{(k+d)}|.$
\end{definition}}

In simple terms, a locally Hölder smooth function on $\mathbb{R}^{k+d}$
is one that is Hölder smooth on every compact subset of $\mathbb{R}^{k+d}.$
Functions such as conditional drift or conditional denoising functions can be locally Hölder smooth without being globally Hölder smooth.

We proceed to derive approximation results for the conditional drift and denoising functions, building on the work of \citet{jiao2023deep}, which established approximation error bounds for scalar-output neural networks applied to Hölder smooth functions. However, \citet{jiao2023deep} focused on non-parametric least squares estimations for real-valued target functions that are H\"{o}lder smooth
with a bounded domain, such as $[0, 1]^d.$
In contrast, our problem involves target functions that receive unbounded, time-varying inputs
that are only locally H\"{o}lder smooth and produce multi-dimensional vector outputs. Consequently, we cannot directly apply the results from \citet{jiao2023deep}; instead, we must extend their results to accommodate unbounded domains.

\begin{lemma} \label{theorem2}
    Suppose that Assumptions \ref{assumption0} and \ref{assumption0b} are satisfied and there exists a constant $B_X$ such that $\Vert X \Vert_\infty \leq B_X$. {\color{black}For any $A_n \geq \max\{B_I, B_X, 1\}$, $t \in (0,1)$, assume that the drift function $\boldsymbol{b}^{*(i)}(\cdot,t)$
    and denoising function $\boldsymbol{\kappa}^{*(i)}(\cdot,t)$ belong to the locally H\"{o}lder  smooth classes $\mathcal{H}^\beta_{\text{loc}}\left(\mathbb{R}^{k+d}, \mathcal{B}_b(t)\right)$ and $\mathcal{H}^{\beta^\prime}_{\text{loc}}\left(\mathbb{R}^{k+d}, \mathcal{B}_\kappa(t)\right)$ with $\beta>0$, $\beta^\prime>0$ for $i = 1,\ldots, d$. }
     \renewcommand{\theenumi}{\roman{enumi}}
     \begin{enumerate}
         \item For any $U, V  \in \mathbb{N}^{+}$,  let $\mathcal{F}_{n}=\mathcal{F}_{\mathcal{D}, \mathcal{W}, \mathcal{U}, \mathcal{S}, \mathcal{B}_b(t)}$ be a class of neural networks with width $\mathcal{W}= 38(\lfloor\beta\rfloor+1)^2 3^{(k+d)} (k+d)^{\lfloor\beta\rfloor+2} (3+\left\lceil \log _2 U\right\rceil)U$, and depth $\mathcal{D}=21(\lfloor\beta\rfloor+1)^2 (3+\left\lceil \log _2 V\right\rceil)V+2 (k+d)$,  then the approximation error
         \begin{align*}
            \inf_{\boldsymbol{f} \in \mathcal{F}_n }  \mathbb{E} \Vert \boldsymbol{f}(X,Y_t,t) - \boldsymbol{b}^*(X,Y_t,t) \Vert^2 \leq& \inf_{\boldsymbol{f} \in \mathcal{F}_n, \atop \Vert \mathbf{y} \Vert_\infty \leq A_n } \Vert \boldsymbol{f}(\cdot,t) - \boldsymbol{b}^*(\cdot,t)\Vert_\infty^2\\
             &+ 4 \mathcal{B}_b(t) \exp \left(- \frac{A_n^2}{2 (\mathcal{B}_I(t) + \gamma(t))^2}\right),
        \end{align*}
where
\begin{align*}
            \mathop{\inf}_{
                \boldsymbol{f} \in \mathcal{F}_n,
        \atop
        \Vert \mathbf{y} \Vert_\infty \leq A_n} \Vert \boldsymbol{f}(\cdot,t) - \boldsymbol{b}^*(\cdot,t) \Vert_\infty \leq  19 \mathcal{B}_b(t)
        C(\beta) (2A_n)^\beta(U V)^{\frac{-2 \beta}{k+d}},
\end{align*} with $C(\beta) =  (\lfloor\beta\rfloor+1)^2 (k+d+1)^{\lfloor\beta\rfloor+(\beta \vee 1) / 2},
\text{ and } \beta \vee 1 = \max\{\beta, 1\}.$

        \item For any $U^\prime, V^\prime  \in \mathbb{N}^{+}$, let $\mathcal{F}_{n}^\prime=\mathcal{F}^\prime_{\mathcal{D}^\prime, \mathcal{W}^\prime, \mathcal{U}^\prime, \mathcal{S}^\prime, \mathcal{B}_\kappa(t)}$ be a class of neural networks with width $\mathcal{W}^\prime= 38(\lfloor\beta^\prime\rfloor+1)^2 3^{(k+d)} (k+d)^{\lfloor\beta^\prime\rfloor+2} (3+\left\lceil \log _2 U^\prime\right\rceil)U^\prime$, and depth $\mathcal{D}^\prime=21(\lfloor\beta^\prime\rfloor+1)^2 (3+\left\lceil \log _2 V^\prime\right\rceil)V^\prime+2 (k+d)$,  then the approximation error
        \begin{align*}
            \mathop{\inf}_{
                \boldsymbol{f} \in \mathcal{F}_n^\prime}  \mathbb{E} \Vert \boldsymbol{f}(X,Y_t,t) - \boldsymbol{\kappa}^*(X,Y_t,t) \Vert^2 \leq&   \inf_{\boldsymbol{f} \in \mathcal{F}_n^\prime, \atop \Vert \mathbf{y} \Vert_\infty \leq A_n  } \Vert \boldsymbol{f}(\cdot,t) - \boldsymbol{\kappa}^*(\cdot,t) \Vert_\infty^2 \\
                & + 4 \mathcal{B}_\kappa(t) \exp \left(- \frac{A_n^2}{2 (\mathcal{B}_I(t) + \gamma(t))^2}\right),
        \end{align*}
        where
         \begin{align*}
            \mathop{\inf}_{
                \boldsymbol{f} \in \mathcal{F}_n^\prime,
        \atop
        \Vert \mathbf{y} \Vert_\infty \leq A_n} \Vert \boldsymbol{f}(\cdot,t) - \boldsymbol{\kappa}^*(\cdot,t) \Vert_\infty \leq  19 \mathcal{B}_\kappa(t)
    C(\beta^{\prime})
        (2A_n)^{\beta^\prime}(U^\prime V^\prime)^{\frac{-2 \beta^\prime}{k+d}}.
        \end{align*}
     \end{enumerate}
\end{lemma}
It is worth noting that the approximation error bounds in Lemma \ref{theorem2} are non-standard as the input vector $(X, Y_t, t)$ may not be bounded, which leads to an additional truncation error term in our analysis. In particular, we truncate the domain to handle the approximation error for unbounded random vectors $\mathcal{I}(Y_0, Y_1,t)$, $\partial_t \mathcal{I}(Y_0, Y_1,t)$, and control the error induced by truncation to obtain the bounds in Lemma \ref{theorem2}.

\subsection{Nonasymptotic upper bounds} \label{non-asym-bounds}
Combining the upper bounds of stochastic error, approximation error and truncation error, we  obtain the non-asymptotic upper bounds for the excess risks of the estimators $\hat{\boldsymbol{b}}_n$ and $\hat{\boldsymbol{s}}_n$.
\begin{lemma} \label{theorem_bound}
    Suppose the conditions in Lemmas \ref{stoc_b}, \ref{theorem2}  are satisfied.
    \renewcommand{\theenumi}{\roman{enumi}}
    \begin{enumerate}
        \item For any $U, V  \in \mathbb{N}^{+}$, let $\mathcal{F}_{n}=\mathcal{F}_{\mathcal{D}, \mathcal{W}, \mathcal{U}, \mathcal{S}, \mathcal{B}_b(t)}$ be a class of neural networks with width $\mathcal{W}= 38(\lfloor\beta\rfloor+1)^2 3^{(k+d)} (k+d)^{\lfloor\beta\rfloor+2} (3+\left\lceil \log _2 U\right\rceil)U$, and depth $\mathcal{D}=21(\lfloor\beta\rfloor+1)^2 (3+\left\lceil \log _2 V\right\rceil)V+2 (k+d)$,  then  for  $n \geq \max_{i = 1, ..., d} \operatorname{Pdim}( \tilde{\mathcal{F}}_{ni})/2$ and any $t \in (0, 1)$, we have
        \begin{align*}
    & \mathbb{E}_S \Vert \hat{\boldsymbol{b}}_n(X,Y_t,t) - \boldsymbol{b}^*(X,Y_t,t) \Vert^2 \\ &\leq   722 d \mathcal{B}_b(t)^2(\lfloor\beta\rfloor+1)^4 (k+d)^{2\lfloor\beta\rfloor+(\beta \vee 1)}(2 A_n)^{2 \beta}(U V)^{\frac{-4 \beta}{k+d}} \\
    &\ \ \ + c_0 d \mathcal{B}_b(t)^5 (\log n)^3 \frac{\mathcal{S D} \log (\mathcal{S})}{n} + 8 d \mathcal{B}_b(t) \exp \left(- \frac{A_n^2}{2 (\mathcal{B}_{I}(t) + \gamma(t))^2}\right),
        \end{align*}
        where $c_0>0$ is a universal constant.
        \vspace*{0.2cm}

        \item For any $U^\prime, V^\prime  \in \mathbb{N}^{+}$, let $\mathcal{F}_{n}^\prime=\mathcal{F}^\prime_{\mathcal{D}^\prime, \mathcal{W}^\prime, \mathcal{U}^\prime, \mathcal{S}^\prime, \mathcal{B}_\kappa(t)}$ be a class of neural networks with width $\mathcal{W}^\prime= 38(\lfloor\beta^\prime\rfloor+1)^2 3^{(k+d)} (k+d)^{\lfloor\beta^\prime\rfloor+2} (3+\left\lceil \log _2 U^\prime\right\rceil)U^\prime$, and depth $\mathcal{D}^\prime=21(\lfloor\beta^\prime\rfloor+1)^2 (3+\left\lceil \log _2 V^\prime\right\rceil)V^\prime+2 (k+d)$,  then for  $n \geq \max_{i = 1, ..., d} \operatorname{Pdim}( \tilde{\mathcal{F}}_{ni})/2$, we have
        \begin{align*}
    & \mathbb{E}_S \Vert \hat{\boldsymbol{s}}_n(X,Y_t,t) - \boldsymbol{s}^*(X,Y_t,t) \Vert^2 \\
    &\leq  722 d \mathcal{B}_\kappa(t)^2 \gamma(t)^{-1} (\lfloor\beta^\prime \rfloor+1)^4(k+d)^{2\lfloor\beta^\prime \rfloor+(\beta^\prime \vee 1)}(2A_n)^{2\beta^\prime}(U^\prime V^\prime)^{\frac{-4 \beta^\prime}{k+d}}\\
        &\ \ \ + c_0^\prime \gamma(t)^{-1}d \mathcal{B}_\kappa(t)^5 (\log n)^3 \frac{\mathcal{S}^\prime \mathcal{D}^\prime \log (\mathcal{S}^\prime)}{n} + 8 d \gamma(t)^{-1} \mathcal{B}_\kappa(t) \exp \left(- \frac{A_n^2}{2 (\mathcal{B}_{I}(t) + \gamma(t))^2}\right)
        \end{align*}
    \end{enumerate}
  and $c_0^\prime>0$ is a universal constant.
\end{lemma}

The nonasymptotic upper bounds in Lemma \ref{theorem_bound} are obtained directly by the summation of stochastic error and approximation error. To achieve the best trade-off of these errors, we can determine a proper network architecture in terms of its depth, width,  and size.
We now state the result on the error bounds for the estimated conditional drift and score functions.

\begin{theorem} \label{corollary_bound}
    Suppose the conditions in Lemmas \ref{stoc_b}, \ref{theorem2}  are satisfied.
    \renewcommand{\theenumi}{(\roman{enumi})}
        \begin{enumerate}
            \item Let $U= n^{(k+d)/(8\beta + 4k+4d)}$ and $V=n^{(k+d)/(8\beta + 4k+4d)},$ for $t \in (0,1)$,
            we have
            \begin{align*}
      \mathbb{E}_S \Vert \hat{\boldsymbol{b}}_n(X,Y_t,t) - \boldsymbol{b}^*(X,Y_t,t) \Vert^2 = H_b(t) O \left(n^{\frac{-2\beta}{2\beta+k+d}} \log^{\max\{8,\beta\}} n\right),
            \end{align*}
            where $H_b(t) := \mathcal{B}_b(t)^5 (\mathcal{B}_I(t)+\gamma(t))^{2\beta}$.
            \item  Let $U^\prime=n^{(k+d)/(8\beta + 4k+4d)}$ and $V^\prime= n^{(k+d)/(8\beta + 4k+4d)},$ for $t \in (0,1)$,
            we have
            \begin{align*}
                \mathbb{E}_S \Vert \hat{\boldsymbol{s}}_n(X, Y_t,t) - \boldsymbol{s}^*(X,Y_t,t) \Vert^2=\gamma(t)^{-1} H_\kappa(t) O \left(n^{\frac{-2\beta^\prime}{2\beta^\prime+k+d}} \log^{\max\{8,\beta^\prime\}} n\right),
            \end{align*}
            where $H_\kappa(t) := \mathcal{B}_\kappa(t)^5 (\mathcal{B}_I(t)+\gamma(t))^{2\beta^\prime}$.
        \end{enumerate}
\end{theorem}

Theorem \ref{corollary_bound} shows that the conditional drift function estimator $\hat{\boldsymbol{b}}_n$ and score function estimator $\hat{\boldsymbol{s}}_n$
achieve the nonparametric optimal minimax rate \citep{stone1982optimal} uniformly in $t \in (0, 1).$
There is no uniformity in the convergence of  $\hat{\boldsymbol{s}}_n$ due to the possible instability of the conditional score at the boundary points $t=0,1.$ This theorem is essential to establishing the error bounds for the distributions of the generated samples based on CSI Flow and CSI Diffusion.
It addresses the general case of the error incurred when estimating the conditional drift and score functions.

For the special case considered in Theorem \ref{pro_score}, where a linear relationship exists between the optimal score function $\boldsymbol{s}^*$ and the optimal drift $\boldsymbol{b}^*$, we can estimate the score function using the estimated drift $\hat{\boldsymbol{b}}_n$. In this case,
we can obtain better error bound for $\hat{\boldsymbol{s}}.$

 \begin{corollary} \label{corollary_linear}
    Suppose the conditions in Lemmas \ref{stoc_b}, \ref{theorem2}  are satisfied. Suppose CSI satisfies (\ref{interpb}),  $Y_0\sim \mathcal{N}(\mathbf{0},\mathbf{I}_d) $,
    and $\gamma(t) \neq 0, t \in (0, 1)$.
    Let $U^\prime=n^{(k+d)/(8\beta + 4k+4d)}$ and $V^\prime= n^{(k+d)/(8\beta + 4k+4d)}$. Define the conditional score estimator
    \begin{align} \label{estimator_s_2}
       \tilde{\boldsymbol{s}}_n(\mathbf{x},\mathbf{y},t) := \frac{b(t)}{A(t)}\hat{\boldsymbol{b}}_n(\mathbf{x},\mathbf{y},t) -\frac{\dot{b}(t)}{A(t)}\mathbf{y}.
    \end{align}
    We have
    \begin{align*}
        \mathbb{E}_S \Vert \tilde{\boldsymbol{s}}_n(X, Y_t,t) - \boldsymbol{s}^*(X,Y_t,t) \Vert^2=  \frac{H_b(t)b^2(t)}{A^2(t)} O \left(n^{\frac{-2\beta}{2\beta+k+d}} \log^{\max\{8,\beta\}} n \right),   \  t \in (0, 1).
    \end{align*}
    Suppose that $b(t)/A(t)$ is a bounded function in $(0, 1)$. Then
    \begin{align*}
    \sup_{t \in (0, 1)}     \mathbb{E}_S \Vert \tilde{\boldsymbol{s}}_n(X, Y_t,t) - \boldsymbol{s}^*(X,Y_t,t) \Vert^2= H_b(t)O \left(n^{\frac{-2\beta}{2\beta+k+d}} \log^{\max\{8,\beta\}}n \right).
    \end{align*}
 \end{corollary}

 A  simple example with a bounded $b(t)/A(t)$ is given in Example \ref{ex1} with the interpolation process $Y_t = (1-t) Y_0 + t Y_1 + \gamma(t)\boldsymbol{\eta}, \ t \in [0, 1],$
 where $\gamma(t) =  \sqrt{2t(1-t)}.$
 Comparing with the upper bound for $\hat{\boldsymbol{s}}_n$ in Theorem \ref{corollary_bound} (ii), the factor $\gamma^{-1}(t)$ is no longer needed in Corollary \ref{corollary_linear}.
 The bound in this corollary improves that given in Theorem \ref{corollary_bound} (ii), since
 for $\gamma(t) =  \sqrt{2t(1-t)},$
 $\gamma^{-1}(t)$ blows up near the boundary points $t=0, 1.$

\section{Sampling properties of CSI flows}
\label{sec_theory}

In this section, we study the sampling properties of the CSI Flow and CSI Diffusion models
and establish error bounds for the estimated distributions based on these models.

\subsection{Sampling distributions of CSI flows}
\label{sub_sampling}

In this subsection, we define the sampling distributions for data generated by CSI flows
with estimated conditional drift and score functions. We first summarize the definitions of the conditional drift and density functions associated with the CSI flow models.

\begin{itemize}
\item
\textbf{ODE Flow}:
For any given $\mathbf{x} \in \mathcal{X},$
the population ODE Flow defined in (\ref{flow_equa}) is
$$\mathrm{d} Z_{t,\mathbf{x}}^{\ode} = \boldsymbol{b}^*_\mathbf{x}(Z_{t,\mathbf{x}}^{\ode}, t)\mathrm{d}t, t \in (0,1) \text{ with } \ Z_{0,X}^{\ode} \sim \mathbb{P}_{Y_0},$$
where $\boldsymbol{b}^*_\mathbf{x}$ is the target conditional drift function.
The corresponding conditional density function  $\rho^{\text{ode}}_{\mathbf{x}}$ can be obtained by solving the transport equation:
\begin{align}
\label{fp_ode}
    \partial_t \rho+  \nabla_{\mathbf{z}} \cdot (\boldsymbol{b}_{\mathbf{x}}^*\rho)
    = 0, \  \rho (\cdot,0)= \rho^{*}_{\mathbf{x}}(\cdot,0).
\end{align}
The estimated ODE Flow defined in (\ref{flow_equa1}) is
   \begin{align*}
 \mathrm{d}\widehat{Z}_{t,X}^{\ode} =
  \hat{\boldsymbol{b}}_n(X, \widehat{Z}_{t,X}^{\ode}, t)\mathrm{d}t, , t \in (0,1) \text{ with }
      \   \widehat{Z}_{0,X}^{\ode} \sim \mathbb{P}_{Y_0},
    \end{align*}
    where $\hat{\boldsymbol{b}}_{\mathbf{x}}$ is the estimated conditional drift function.
  given in (\ref {bhat}).
Let $\hat{\rho}_{n,\mathbf{x}}^{\ode}$
be the solution to the Fokker-Planck equation
 \begin{align}
\label{fp_ode2}
 \partial_t \rho+\nabla_\mathbf{z} \cdot\left(\hat{\boldsymbol{b}}_{n,\mathbf{x}} \rho\right) = 0, \
 \rho (\cdot,0)= \rho^{*}_{\mathbf{x}}(\cdot,0).
 \end{align}

\item \textbf{SDE Diffusion}: For any given $\mathbf{x} \in \mathcal{X},$
the population CSI Diffusion defined in (\ref{sde_equ}),
$$
        \mathrm{d} Z_{t,\mathbf{x}}^{sde} = \boldsymbol{b}_{u,\mathbf{x}}^*(Z_{t,\mathbf{x}}^{sde}, t)\mathrm{d}t + \sqrt{2u(t)} \mathrm{d} W_t, \ t \in (0,1) \text{ with }   {Z}_{0,X}^{\sde} \sim \mathbb{P}_{Y_0},
$$
where
$\boldsymbol{b}_{u,\mathbf{x}}^*(\mathbf{z}, t) :=
\boldsymbol{b}_u^*(\mathbf{x}, \mathbf{z}, t) =
\boldsymbol{b}^*( \mathbf{x}, \mathbf{z}, t) + u(t)\boldsymbol{s}^*(\mathbf{x}, \mathbf{z},t)$.
The corresponding conditional density function
$\rho^{\sde}_{\mathbf{x}}$ satisfies the Fokker-Planck equation:
\begin{align} \label{fp_equa}
        \partial_t \rho +  \nabla_\mathbf{z} \cdot (\boldsymbol{b}_{u,\mathbf{x}}^*\rho) - u(t) \Delta_\mathbf{z} \rho = 0, \    \rho (\cdot,0)= \rho^{*}_{\mathbf{x}}(\cdot,0).
\end{align}
The estimated SDE Diffusion defined in (\ref{sde_equ1}) is
 \begin{align*}
        \mathrm{d} \widehat{Z}_{t,X}^{\sde} =
         \hat{\boldsymbol{b}}_{u,n,\mathbf{x}}( X, Z_{t,X}^{\sde}, t) \mathrm{d}t
        + \sqrt{2u(t)} \mathrm{d} W_t, \ t \in (0,1) \text{ with }   \widehat{Z}_{0,X}^{\sde} \sim \mathbb{P}_{Y_0},
         \
    \end{align*}
 where $\hat{\boldsymbol{b}}_{u,n,\mathbf{x}}(\mathbf{z},t) =\hat{\boldsymbol{b}}_{n,\mathbf{x}}(\mathbf{z},t) +u(t)  \hat{\boldsymbol{s}}_{n,\mathbf{x}}(\mathbf{z},t).$
Here $\hat{\boldsymbol{b}}_{\mathbf{x}}$  and $\hat{\boldsymbol{s}}_n$ are the estimated conditional drift and score functions given in (\ref {bhat}).
Let $\hat{\rho}_{n,\mathbf{x}}^{\sde}$ be the solution to the Fokker-Planck equation
    \begin{align}
\label{sde_equ2}
        \partial_t \rho+\nabla_\mathbf{z} \cdot\left(\hat{\boldsymbol{b}}_{u,n,\mathbf{x}} \rho \right) - u(t) \Delta \rho = 0, \  \rho(\cdot,0)= \rho^{*}_{\mathbf{x}}(\cdot,0).
    \end{align}
\end{itemize}

\begin{table}[H]
\caption{Functions and processes in the ODE-Flow and SDE-Diffusion. NA stands for not applicable.}
\begin{tabular}{llllc}
  \toprule
   &  Drift & Score & Flow & Conditional density \\
   \midrule
  Population Target & $\boldsymbol{b}^{*}_{\mathbf{x}}$ & $\boldsymbol{s}^{*}_{\mathbf{x}}$ & NA & $\rho^{*}_{\mathbf{x}} $  \\
  Population ODE-Flow &$ \boldsymbol{b}^{*}_{\mathbf{x}}$ & NA & $Z^{\ode}_{t,X}$ &
  $ \rho^{\ode}_{\mathbf{x}} $ \\
  Estimated ODE-Flow &$ \hat{\boldsymbol{b}}_{n,\mathbf{x}}$ & NA & $\widehat{Z}^{\ode}_{t,X}$ & $ \hat{\rho}^{\ode}_{n,\mathbf{x}}$  \\
  Population SDE-Diffusion&$\boldsymbol{b}^{*}_{\mathbf{x}}$  &
  $\boldsymbol{s}^{*}_{\mathbf{x}}$ &
  $Z^{\ode}_{t,X}$  & $ \rho^{\sde}_{\mathbf{x}} $  \\
  Estimated SDE-Diffusion&$\hat{\boldsymbol{b}}_{n,\mathbf{x}}$ &$\hat{\boldsymbol{s}}_{n,\mathbf{x}}$
  & $\widehat{Z}^{\sde}_{t,X}$  &  $\hat{\rho}^{\sde}_{\mathbf{x}} $\\
   \bottomrule
\end{tabular}
\end{table}

\subsection{The marginal preserving property} \label{sub_marginal}

The marginal preserving property means that both the CSI Flow and CSI Diffusion models follow the same conditional distribution of $Y_t \mid X$ for every $t \in [0,1]$ and $X \in \mathcal{X}$, despite traversing different paths. This equivalence ensures that, regardless of the method used, the generated samples follow the conditional distribution of $Y_t \mid X$.

The marginal preserving property for stochastic interpolations in the unconditional setting has been previously explored by \cite{albergo2023stochastic}. However, their analysis did not fully address the necessary assumptions for the existence and uniqueness of solutions to PDEs within the relevant function space. In this work, we extend their results to conditional stochastic interpolations, while providing standard regularity conditions on the drift and score functions associated with the process.
To ensure the existence and uniqueness of solutions to the Fokker-Planck equations, which describe the time evolution of probability densities under the influence of drift and diffusion processes, we invoke the following assumption on the drift and score functions. This assumption was used in Proposition 2 of \cite{bris2008existence}, which provides the theoretical underpinning for our analysis.
	
\begin{assumption} \label{assump_unique}
    For any $\mathbf{x} \in \mathcal{X}$, the functions $\boldsymbol{b}^*(\mathbf{x}, \cdot, \cdot)$ and $u(\cdot) \boldsymbol{s}^*(\mathbf{x}, \cdot, \cdot)$ satisfy
    \begin{align*}
        &\boldsymbol{b}^*(\mathbf{x}, \cdot, \cdot), u(\cdot) \boldsymbol{s}^*(\mathbf{x}, \cdot, \cdot) \in  (L^1(W^{1,1}_{loc}(\mathbb{R}^d), (0,1)))^d, \\
        &\nabla \cdot \boldsymbol{b}^*(\mathbf{x}, \cdot, \cdot), u(\cdot) \nabla \cdot \boldsymbol{s}^*(\mathbf{x}, \cdot, \cdot) \in L^1( L^\infty(\mathbb{R}^d), (0,1)), \\
        &\frac{\boldsymbol{b}^*(\mathbf{x}, \cdot, \cdot)}{1 + \Vert \mathbf{x} \Vert}, \frac{u(\cdot)  \boldsymbol{s}^*(\mathbf{x}, \cdot, \cdot)}{1 + \Vert \mathbf{x} \Vert} \in  (L^1(L^1 +L^\infty(\mathbb{R}^d), (0,1)))^d,
    \end{align*}
    where $\Vert \cdot \Vert$ is
    the Euclidean norm, the $u(t)$ satisfies the conditions in Definition \ref{def_process}, and the local Sobolev space $W^{1, 1}(\mathbb{R}^d)$ is defined as
    \begin{align*}
        W^{1, 1}_{loc}(\mathbb{R}^d):=\left\{f \in L^1_{loc}(\mathbb{R}^d): {D}^{\boldsymbol{\alpha}} f \in L^1_{loc} (\mathbb{R}^d), \forall \Vert {\boldsymbol{\alpha}} \Vert_1 \leq 1 \right\}.
    \end{align*}
    Here, $L_{\mathrm{loc}}^1(\mathbb{R}^d) :=\left\{f: \mathbb{R}^d \rightarrow \mathbb{R} \text { measurable }:\left.f\right|_K \in L^1(K), \forall K \subset \mathbb{R}^d, K \text { compact }\right\}$  is locally integrable function sapce.
\end{assumption}

\begin{lemma} \label{theorem_fp}
If Assumptions \ref{assumption0}, \ref{assumption0b} and  \ref{assump_unique} hold, then the marginal preserving property holds in the sense that, for any $\mathbf{x} \in \mathcal{X}$,
$$\rho^*_{\mathbf{x}}(\mathbf{z}, t) = \rho^{\text{ode}}_{\mathbf{x}}(\mathbf{z}, t) = \rho^{\text{sde}}_{\mathbf{x}}(\mathbf{z}, t), \ \text{ for any } \   ( \mathbf{z}, t) \in
 \mathbb{R}^d \times [0,1].
$$
\end{lemma}

Lemma \ref{theorem_fp} shows the conditional density functions corresponding to the ODE Flow and
SDE Diffusion is the same as the target conditional distribution. Therefore,  In the rest of the paper, we refer to either $\rho^{\ode}_{\mathbf{x}}$ or $\rho^{\sde}_{\mathbf{x}}$ as $\rho^{*}_{\mathbf{x}}.$

\subsection{Error bounds for generated conditional distributions}
In this section, we study the error bounds for the
conditional distributions of the samples generated by the ODE-based and SDE-based models.
We use the 2-Wasserstein distance and the Kullback-Leibler (KL) divergence to measure the
difference between the distributions produced by the ODE-based and SDE-based models and the target conditional distribution, respectively.

For ODE-based models (the CSI Flow), the error bounds for the estimated drift function
can lead to a control of the 2-Wasserstein distance between the estimated density and its target. To further clarify this connection, we first state a lemma that is particularly relevant to the behavior of the ODE-based generators.

\begin{lemma} \label{lemma_w2}
    Suppose that Assumptions \ref{assumption0} and \ref{assumption0b} are satisfied. For any drift function $\boldsymbol{b}(\mathbf{x}, \mathbf{z}, t)$ satisfying $\left\|\boldsymbol{b}(\mathbf{x}, \mathbf{z}, t)-\boldsymbol{b}(\mathbf{x}, \mathbf{y}, t)\right\|_{\infty} \leq l\|\mathbf{z}-\mathbf{y}\|_{\infty}$ for some $l >0.$
    Let $Z_{t,X} := Z_t \mid X$  be obtained through $ \mathrm{d} Z_{t,X} = \boldsymbol{b}(X,Z_{t,X}, t)\mathrm{d}t$ for $t \in (0,1)$ with $Z_{0,X} = Y_0.$ Denote the conditional density of $Z_t = \mathbf{z} \mid X= \mathbf{x}$ by ${\rho}(\mathbf{x}, \mathbf{z}, t)= \rho_{\mathbf{x}}(\cdot,t)$. Then, we have
       \begin{align} \label{w2}
        \mathbb{E}_{\mathbf{x}}[W_2^2(\rho^*_{\mathbf{x}}(\cdot,t), \rho_{\mathbf{x}}(\cdot,t))]  \leq \exp \left(2 l+1\right) \int_{0}^{t} \mathbb{E} \Vert \boldsymbol{b}^*(\mathbf{x}, Z_{s,\mathbf{x}}^{ode}, s)-\boldsymbol{b}(\mathbf{x}, Z_{s,\mathbf{x}}^{ode}, s) \Vert^2 ds,
       \end{align}
       where $W_2^2(\cdot,\cdot)$ denotes the 2-Wasserstein distance.
\end{lemma}

By combining Lemma \ref{lemma_w2} with  Theorem \ref{corollary_bound}, we are able to derive error bounds for the W2-distance  between the sample distribution generated by the ODE Flow model and the target distribution.

\begin{theorem} \label{theorem_w2}
    Suppose $H_b(t)$ is integrable on $[0,1]$. Assuming that the conditions in Theorem \ref{corollary_bound}  and Lemma \ref{lemma_w2} are satisfied, we have
    \begin{align}
 \mathbb{E}_{S,X}[W_2^2(\rho^{*}_{X}(\cdot,1), \hat{{\rho}}_{n,X}^{ode}(\cdot,1))] = O \left(n^{\frac{-2\beta}{2\beta+k+d}} \log^{\max\{8,\beta\}}n \right).
    \end{align}
\end{theorem}

For SDE-based models (the CSI Diffusion), the upper bound of the KL divergence between the induced density estimator and its target is determined by the error bounds associated with the estimators $\hat{\boldsymbol{b}}_n$ and $\hat{\boldsymbol{s}}_n$, which is stated in the following lemma.

\begin{lemma} \label{kl_lemma}
    Suppose that Assumptions \ref{assumption0}, \ref{assumption0b} are satisfied. Let $Z_{t,X}$
be defined according to $\mathrm{d} Z_{t,X} = \boldsymbol{b}_u(X,Z_{t,X}, t)\mathrm{d}t + \sqrt{2u(t)} \mathrm{d} W_t$ with $Z_{0,X} \sim P_{Y_0}$. Denote the corresponding time-dependent conditional density of $Z_{t,X}$ by ${\rho}_\mathbf{x}(\mathbf{z}, t)= {\rho}(\mathbf{x}, \mathbf{z}, t).$ Then, for any integrable functions $u $ and $u ^{-1}$ and any $t\in [0,1],$ we have
\begin{align} \label{kl}
        {\mathbb{E}_{\mathbf{x}}[\mathrm{KL}(\rho_{\mathbf{x}}^{*}(\cdot,t) \| {\rho}_{\mathbf{x}}(\cdot,t))]} \leq & \int_0^t \frac{1}{2u(s)}
        \mathbb{E} \Vert {\boldsymbol{b}(\mathbf{x},Z_{s,\mathbf{x}}^{sde}, s)} - \boldsymbol{b}^*(\mathbf{x},Z_{s,\mathbf{x}}^{sde}, s) \Vert^2 ds \\
        &+ \int_0^t \frac{u(s)}{2} \mathbb{E} \Vert {\boldsymbol{s}(\mathbf{x},Z_{s,\mathbf{x}}^{sde}, s)} - \boldsymbol{s}^*(\mathbf{x},Z_{s,\mathbf{x}}^{sde}, s) \Vert^2 ds. \nonumber
\end{align}
\end{lemma}

By integrating the results from Lemma \ref{kl_lemma} with those from Theorem \ref{corollary_bound}, we are able to derive error bounds for the Kullback-Leibler (KL) divergence between the sample distribution generated by the SDE Diffusion model and the target distribution.

\begin{theorem}\label{kl_theorem}
    Suppose that the conditions stated in Theorem \ref{corollary_bound} are satisfied.
    Furthermore, suppose that $H_b(t)u(t)^{-1}$ and $u(t)H_\kappa(t)\gamma(t)^{-1}$  are integrable on $[0,1],$ we have
        \begin{align*}
            {\mathbb{E}_{S,X}[\mathrm{KL}(\rho_{X}^{*}(\cdot,1) \| \hat{\rho}_{n, X}^{sde}(\cdot,1))]} = O \left(n^{\frac{-2\tilde{\beta}}{2\tilde{\beta}+k+d}} \log^{\max\{8,\tilde{\beta}\}}n \right) ,
        \end{align*} where $\tilde{\beta}=\min\{\beta,\beta^\prime\}$.
\end{theorem}

Theorems \ref{theorem_w2} and \ref{kl_theorem} show that the ODE-Flow and SDE-Diffusion-based distribution estimators are consistent and achieve a nonparametric rate of convergence. Nonetheless, these rates are not minimax optimal because the derived bounds for the distribution estimators
$\hat{r}^{ode}_{n,X}$ and $\hat{r}^{sde}_{n,X}$
are obtained indirectly through the control of the drift estimator
$\hat{b}_n$ and score estimator $\hat{s}_n.$ Importantly, the goal of generative learning is not to estimate the functional form of the target distribution but to learn a sampling procedure from the target distribution, which presents a more complex challenge.

\section{Numerical examples} 
\label{sec_experiment}
In this section, we use two examples to illustrate the proposed CSI in learning conditional distributions.

\subsection{Conditional distribution based on a regression model} 
\label{Reg1}
We consider the conditional distribution based on the  nonlinear regression model:
$Y = f(X) + \epsilon,$
where $\epsilon \sim N(0,1)$ and $f $ is a regression function. The reference distribution $\mathbb{P}_{Y_0}$ is set to be $N(0, 1)$ and we adopt additive interpolation process, i.e. $\mathcal{I}(Y_0,Y_1,t) = a(t)Y_0 + b(t)Y_1$. Under this setting,
the ground truths of conditional drift and score function are

\begin{align*}
    \boldsymbol{b}(\mathbf{x},y,t) = \phi(t)(y-b(t)f(\mathbf{x})) + \dot{b}(t)f(\mathbf{x}), \
    \boldsymbol{s}(\mathbf{x},y,t) = \frac{-(y-b(t)f(\mathbf{x}))}{a(t)^2 + b(t)^2 + \gamma(t)^2}, \
    t \in [0, 1],
\end{align*} where $\phi(t) :=  \frac{\dot{a}(t) 
a(t) + \dot{b}(t)  b(t) + \dot{\gamma}(t)  \gamma(t)}{a(t)^2 + b(t)^2 + \gamma(t)^2}.$
The details of the derivation are given in Section \ref{app_for_example} of the Appendix.
Based on Definitions \ref{def_flow} and \ref{def_process}, we can obtain the generators based on ODE and SDE:

\begin{align*}
    \mathrm{d} Z_{t,\mathbf{x}}^{\ode} &= \left(\phi(t)(Z_{t,\mathbf{x}}^{\ode}-b(t)f(\mathbf{x})) + \dot{b}(t)f(\mathbf{x})\right) \mathrm{d}t, \\
    \mathrm{d} Z_{t,\mathbf{x}}^{sde} &= \left(\phi(t)(Z_{t,\mathbf{x}}^{sde}-b(t)f(\mathbf{x})) + \dot{b}(t)f(\mathbf{x}) -u(t) \frac{Z_{t,\mathbf{x}}^{sde}-b(t)f(\mathbf{x})}{a(t)^2 + b(t)^2 + \gamma(t)^2} \right) \mathrm{d}t + \sqrt{2u(t)} \mathrm{d} W_t,
\end{align*} where $Z_{0,X}^{\ode}, Z_{0,X}^{\sde} \sim \mathbb{P}_{Y_0}.$
As an example, we take
$$f(\mathbf{x}) = \vert 2+ (\mathbf{x}^{(1)})^2/3 \vert - \vert \mathbf{x}^{(2)} \vert +  \max\{ (\mathbf{x}^{(3)})^3, \mathbf{x}^{(4)} \times \exp(\mathbf{x}^{(5)}/2)\},$$ and select $a(t)=\cos(\pi t/2)$ and $b(t)=\sin(\pi t/2)$. The perturbation function is given by $\gamma(t) = \log(t - t^2 + 1)$. And the diffusion function for the SDE-based generator is defined as $u(t) = t^2(1-t)^2/8.$

For simulation, we consider two distinct conditions: $X_0 = [0,0,0,0,0]^T$ and $X_1 = [2,2,2,2,2]^T$ and display the conditional distributions of $Y_t \mid X_0$ and $Y_t \mid X_1$ at intermediate times $t=0.2$, $0.4$, $0.6$, and $0.8$. The sample size is
$5000$. The estimated time-dependent conditional distributions based on the three process: conditional stochastic interpolation, ODE-based, and SDE-based generators are plotted in Figure \ref{fig1}. From the figure, the distributions of the three processes are similar at each moment. Moreover, at $t=1$, the conditional probability densities are similar. This experiment supports the result of Lemma \ref{theorem_fp}.

\begin{figure}[H]
\centering
\includegraphics[width=0.85\textwidth, height=1.8 in]{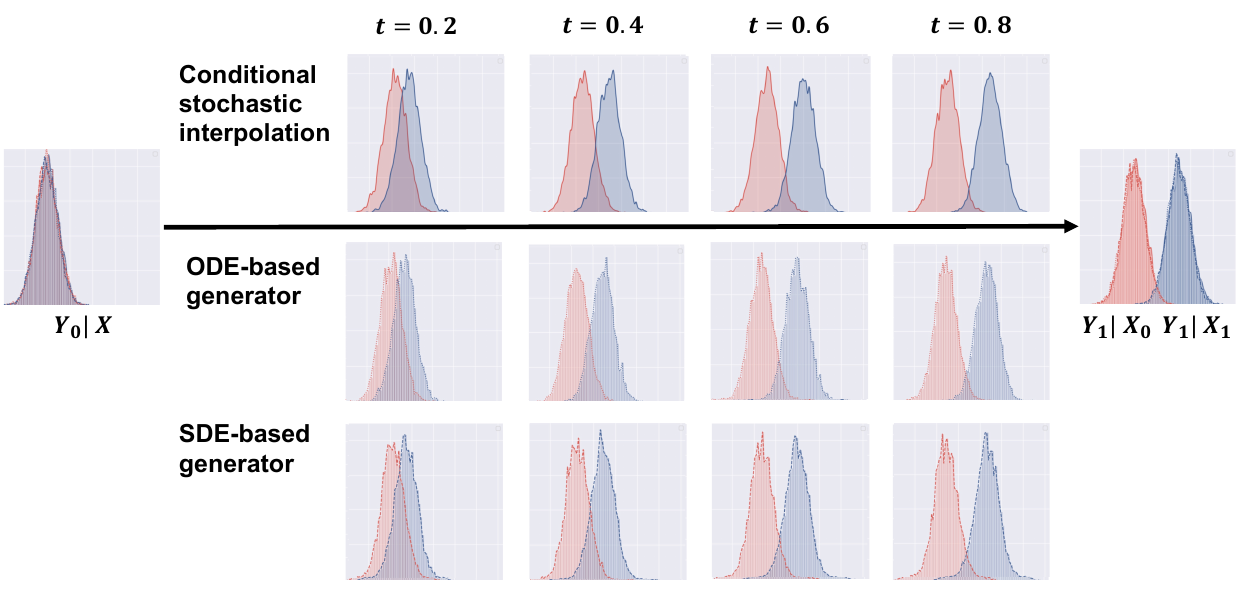}
\caption{The sampling distribution of three processes (interpolation, flow and stochastic process) at different moments. Dashed, dotted, and solid lines represent the time-dependent distribution obtained by conditional stochastic interpolation, ODE-based generator, and SDE-based generator, respectively. The red portion represents the conditional distribution of $Y_1 \mid X = \mathbf{0}$, and the blue portion represents the conditional distribution of $Y_1 \mid X = 2 \cdot \mathbf{1}$.}\label{fig1}
\end{figure}

\subsection{The STL-10 dataset} \label{sec_stl10}
We now illustrate the application of CSI  to a high-dimensional data problem. Our objective is to generate images conditionally on a given label. We provide an illustrative example using the STL-10 dataset \citep{coates2011analysis}, which contains $13,000$ color images.
These images are categorized into $10$ distinct classes (e.g., airplane, bird, car, cat, deer, etc.) and the class sizes are equal. Each image is stored as a $96 \times 96 \times 3$ tensor for RGB channels and paired with a label in $\{0,1,...,9\}$. We employ one-hot vectors in $\mathbb{R}^{10}$ to encode the ten categories. We use 10,200 images for
 training, and reserve the remaining images for testing and evaluation.

To construct a CSI process, we set  $\mathcal{I}(Y_0,Y_1,t)= tY_1 + (1-t)Y_0$ and $\gamma(t)= \sqrt{2t(1-t)}$ as given in Example \ref{ex1}(a).
The reference distribution $\mathbb{P}_{Y_0}$ is chosen to be $N(\mathbf{0},\mathbf{I}_d).$
We estimate both the ODE-based and SDE-based generators.
We adopt the diffusion coefficient $u(t)=0.1(1-t)$ for SDE-based generators. Furthermore, we test SDE-based generators constructed using two different estimation methods: one derives the score function through the formula in  (\ref{score_b}), while the other estimates the score function through the optimization procedure in   (\ref{loss}) using a neural network. We employ the U-net architecture \citep{ronneberger2015u} for the estimation of the conditional drift and score functions, and use the PyTorch platform to train the neural networks. Additional details on the hyperparameter settings for the neural networks and the experimental setup are given in Appendix C.

For sampling, we utilize the Euler-Maruyama method to approximate the trajectories of the ODE and SDE models, with a time step of $\Delta t = 1 \times 10^{-3}$. The generative results are shown in Figure \ref{fig2}. The synthesized images closely resemble the actual ones.
In addition,  CSI can effectively discern the distinguishing characteristics among different classes of image. Furthermore, using the score function derived from the formula in (\ref{score_b}) is computationally efficient, and its sample quality matches that of the score function trained by the neural network optimization procedure in   (\ref{loss}).

\begin{figure}[H]
\centering
\includegraphics[width=4.0 in, height=4.3 in]
{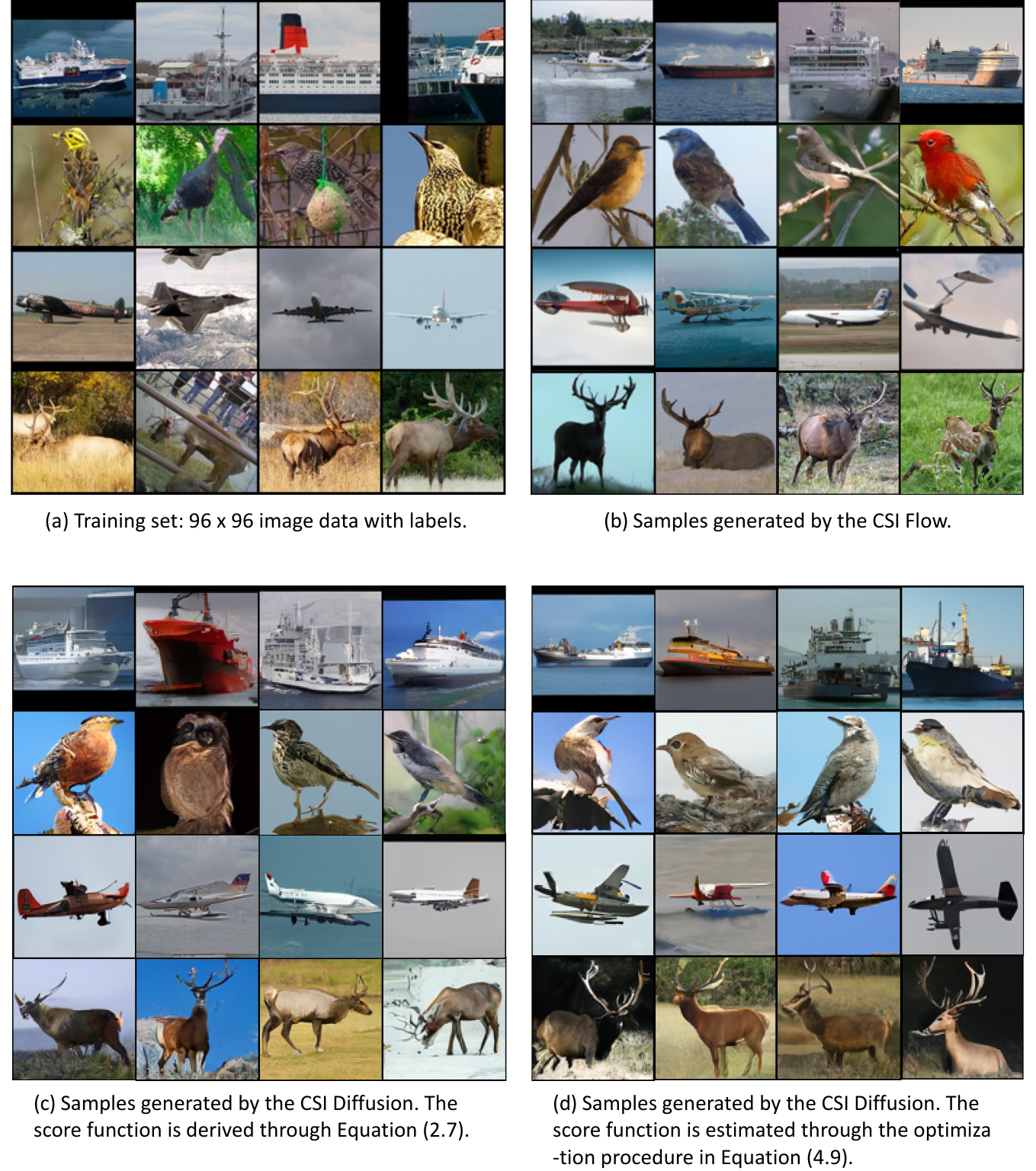}
\caption{The real images (top left) and samples generated through CSI generators.}\label{fig2}
\end{figure}

\section{Related work}\label{sec_relatedwork}
In this section, we discuss the connections and differences between
our work and the related papers in the literature. Since there is now a big literature on process-based generative modeling approaches, it is beyond the scope of this paper to have a complete review of the existing literature, we only consider the most related ones
below.

\subsection{Denoising diffusion probabilistic model}
Diffusion models have their roots in the concept of image denoising, known as denoising diffusion probabilistic model (DDPM) \citep{sohl2015deep,ho2020denoising}. This methodology initially involves a parameterized Markov chain trained through variational inference and generates samples matching the data within a discrete time frame $\{k\}_{k=1}^T$, where $T \in \mathbb{N}^+$.  Contrary to the models involved in CSI, diffusion models theoretically transport the sample $Y_0$ from the target distribution $\pi_{data}$ to a sample from
$\mathcal{N}(\mathbf{0}, \mathbf{I}_d)$ as time progresses from $0$ to infinity. The diffusion process or forward process $Y_k$ involves the incremental addition of noise to the data $Y_0$ according to a variance schedule $\{\beta_1, ..., \beta_T\} \subset \mathbb{R}^+ $,
\begin{align}\label{sec41}
    Y_k =\sqrt{1-\beta_k} Y_{k-1} + \beta_k \boldsymbol{\eta}.
\end{align}
This process can be sampled at any time step $k$ using a closed-form expression that combines the initial value $Y_0$ with the noise term, weighted by scaling factors $\sqrt{\bar{\alpha}_k}$ and $\sqrt{1-\bar{\alpha}_k}$:
\begin{align}\label{sec42}
    Y_k = \sqrt{\bar{\alpha}_k} Y_0 + \sqrt{1-\bar{\alpha}_k} \boldsymbol{\eta}, \text{ with } \bar{\alpha}_k := \prod_{i=1}^{k} (1- \beta_i).
\end{align}
By Tweedie's formula \citep{efron2011tweedie},
$
    \mathbb{E}[Y_0 \mid Y_k] = \bar{\alpha}_k^{-1/2} [Y_k + (1-\bar{\alpha}_k)\boldsymbol{s}(Y_k, k)],
$
where the score function $\boldsymbol{s}(y,k) := \nabla_y \log p(y,k).$

Transitions of this process, denoted as $Z_k$, is defined as a Markov chain with learned Gaussian transitions starting at $Z_T \sim \mathcal{N}(\mathbf{0},\mathbf{I}_d)$:
$
    Z_{k-1} = \boldsymbol{\mu}_{\boldsymbol{\theta}}(Z_k, k) + {\Sigma}_{\boldsymbol{\theta}}^{{1}/{2}}(Z_k, k) \boldsymbol{\eta},
$
where $\boldsymbol{\mu}_{\boldsymbol{\theta}}(Z_k,k) := (1- \beta_k)^{-1/2}[Z_k + {\beta_k} \boldsymbol{s}_{\boldsymbol{\theta}}(Z_k,k)]$ and ${\Sigma}_{\boldsymbol{\theta}} := \tilde{\beta}_k \mathbf{I}_d$. The score function $\boldsymbol{s}_{\boldsymbol{\theta}}(Z_k,k)$ and the covariance matrix are learned through reparameterization using neural networks, with $\boldsymbol{\theta}$ representing the network parameters.

As $T\to \infty,$ the distribution of $Y_T$ converges to $N(\mathbf{0}, \mathbf{I}_d).$
In practice, it takes a large number of function evaluations to reach Gasussian distribution. On the other hand, in a finite time interval, $Y_T$ and $Z_T$ are different, so there is an inevitable bias in the generative samples.

\subsection{Continuous process-based generative models}
The idea of using a continuous process for generative modeling was initially introduced by \citet{song2020score}.
%
By defining the scale function $\beta : [0,1] \rightarrow [0, T\beta_T]$ as $\beta(k/T) := T\beta_k, k \in \{0,1,..., T\}$. The equation (\ref{sec41}) converges to the following stochastic differential equation (SDE) when $T \rightarrow +\infty$:
\begin{align} \label{ddpm_sde}
    \mathrm{d} Y_\tau = - \frac{1}{2} \beta (\tau) Y_\tau \mathrm{d}\tau + \sqrt{\beta(\tau)} \mathrm{d} W_\tau, \ Y_0 \sim \pi_{data},  \tau \in [0,1],
\end{align} where $W_\tau$ is a Wiener process on $[0,1]$.
This model relies on reverse SDEs for sampling, which is different from the CSI approach. Specifically, for SDE (\ref{ddpm_sde}), the reverse SDE can be expressed as:
\begin{align*}
    \mathrm{d}Z_{1-t} = \frac{1}{2} \beta(1-t) Z_{1-t} + \beta(1-t) \boldsymbol{s}_{\boldsymbol{\theta}}(Z_{1-t},1-t) \mathrm{d}t + \sqrt{\beta(1-t)} \mathrm{d} V_{1-t}, t \in [1, 0],
\end{align*}
where $V_{1-t}$ is a reversed Wiener process.
It should be noted that $Z_1$ and $Y_1$ obey different distributions, which introduces bias.
All these sampling procedures rely on the score function, which makes them known as score-based diffusion methods (SBDM).

 In comparison, interpolation-type generative models can circumvent the bias issue of SBMDs
 \citep{albergo2023stochastic, lipman2022flow}.
 \citet{liu2022flow} propose a linear interpolation to connect the target distribution and a reference distribution and a non-linear interpolation were further proposed \citep{albergo2022building}.

In this study, we build upon the stochastic interpolation method introduced by \citet{albergo2023stochastic}, expanding it to encompass conditional process-based modeling. Our proposed CSI approach sheds light on how conditioning influences the support set of $Y_t$ during its transition, which in turn affects the resulting distribution. A limitation of the model proposed by \citet{albergo2023stochastic} is its reliance on a fixed variance scale for the stochastic term. This fixed scale can lead to an explosive behavior of the score function near the data distribution, introducing truncation bias in practical applications. To address this issue, our approach integrates an adaptive diffusion term, specifically the term $\sqrt{2 u(t)}$ as detailed in (\ref{sde_equ}) in Definition \ref{def_process}. This adaptation effectively counters the adverse effects associated with the explosion of the score function during the training phase.

\subsection{Theoretical studies of process based generative models}
Theoretical studies of SBDM with SDE have aimed to establish convergence rates, assuming a given estimation error for the score function. These studies rely on strong assumptions concerning the data distribution, such as functional inequalities \citep{lee2022convergence, wibisono2022convergence} or uniform Lipschitz condition \citep{lee2022convergence, chen2022sampling}.
Some other works relax the smoothness constraint on the score function but introduce assumptions regarding infinite KL-divergence or Fisher information between the data distribution and a multivariate standard normal distribution \citep{chen2022sampling, conforti2023score}.
Both of these conditions imply that the support of the data distribution fills the entire  $\mathbb{R}^d$ space. Violation of this assumption leads to the explosion of score functions.
Theoretical results on DDPM are limited. To the best of our knowledge, only \citep{li2023towards} derive a convergence rate for the DDPM-type sampler but
it is obtained under a specific setting of hyperparameters $\{\beta_t \}_{t=1}^T$.

Under a manifold support assumption,
\citet{pidstrigach2022score} shows that the approximate backward process converges to a random variable whose distribution is supported on the manifold of interest. \citet{de2022convergence} complement these results by studying the discretization scheme and providing quantitative bounds between the output of the diffusion model and the target distribution. However, these techniques either suffer from exponential dependence on problem parameters \citep{de2022convergence} or high iteration complexity \citep{lee2023convergence,chen2023improved}. A recent work by \citet{benton2023linear} improves upon the result of \citet{chen2023improved} and obtains
error bounds that are linear in the data dimension.

For interpolation based generative models, \citet{liu2022flow} present a derivation of the transport equation within the framework of deterministic and linear interpolation, elucidating the connection between the velocity field and the distribution field. \citet{chen2023restoration} further generalize the results to encompass broader scenarios involving stochastic and nonlinear situations using Fourier transform techniques. However, their approach is constrained by the requirements of function space and smoothness conditions imposed on the data density function to ensure the validity of the inverse Fourier Transform. In contrast, we employ a different technique to obtain the transport equation and circumvent the need for these strict assumptions about data distribution.

Moreover, we obtain the closed form of the conditional score function using a novel technique, without the need for additional assumptions. These theoretical developments form the foundation for ensuring that the distribution transfer achieved through stochastic interpolation is consistent with the Markov processes generated by the corresponding drift and score functions. This consistency is a critical factor in the realm of interpolation-based generative models.
In this context, it is imperative to investigate the existence and uniqueness of solutions to the Fokker-Planck equation, as highlighted by \citet {liu2022flow}. To guarantee the existence and uniqueness of these solutions, we invoke the regularity condition stated in Assumption (\ref{assump_unique}). This assumption is essential for the theoretical validation of the models we consider.

Regarding the convergence of the  distribution of generated samples, the existing results were established based on pre-specified estimators of drift function and score function with a given accuracy of estimation. The upper bounds expressed in terms of the pre-sepcified estimation accuracy were derived for  the ODE sampler in terms of 2-Wasserstein distance \citep{benton2023error}, and  the SDE sampler in terms of  KL divergence \citep{albergo2023stochastic}. In addition, the former work relies on a Lipschitz type condition, while the latter neglects the impact of an unstable score function near the data distribution on the transportation process, leading to an unavoidable introduction of early stopping. Our proposed CSI approach effectively mitigates the influence of score function explosion and bias resulting from imprecise transportation.

For interpolation-type models, we first provide the learning guarantee based on the neural network function class, where we establish upper bounds for both stochastic error and approximation error in the estimation of the drift and score function, respectively. The approximation capability of neural networks is linked to the smoothness of functions, and our analysis relies on this assumption solely for assessing approximation errors. It is worth highlighting that the estimation error bound for the score function depends on $t.$ Therefore, our error analysis considers the difficulty of approximating real functions at different time points, particularly when the score function approaches explosion. This aspect aligns more closely with practical scenarios, where obtaining accurate estimates at the boundary is more challenging. Furthermore, we establish non-asymptotic upper error bounds for the learned conditional distribution in terms of the 2-Wasserstein distance and the KL divergence.

\section{Conclusions and discussion} \label{sec_conclusion}
In this work, we have proposed a CSI  method for conditional generative learning, which provides a unified basis for using conditional ODE-based and SDE-based models.
We have provided learning guarantees for the estimation accuracy of the CSI generative models. Under mild conditions, we have established the non-asymptotic error bounds for the estimated drift and score functions based on the conditional stochastic interpolation process. In particular, we have derived an upper bound on the KL divergence between the distributions of generated samples and the truth for the SDE-based generative model.

We have also conducted numerical experiments to evaluate CSI,  which show that it works well in the examples we have considered, including a nonparametric conditional density estimation problems to a more complex image generation problem. However, the scope of our numerical experiments are limited, more extensive studies with simulated and real data are needed to further examine the performance of CSI.

Several questions deserve further investigation.
First, it would be interesting to establish a stronger version of the marginal preserving property, or under weaker conditions. For generative purposes, it is sufficient that the sample distribution at the terminal $t=1$ of the generator is consistent with the target, while marginal preserving property ensures time-dependent conditional distributions of the three processes (interpolation, flow, and stochastic process) are equivalent at each moment. Second, extending the proposed CSI approach to infinite integral fields would be interesting since score-based diffusion models can be regarded as a variant in CSI framework, {namely the infinite-horizon conditional stochastic interpolation.} This extension would enrich the theory and improve the understanding and application of diffusion models. We leave these research questions for the future work.


\bibliographystyle{imsart-nameyear}
\bibliography{CSIarXivRev2}

\newpage
\appendix
\setcounter{section}{0}
\renewcommand{\thesection}{\Alph{section}}
\setcounter{equation}{0}
\renewcommand{\theequation}{S.\arabic{equation}}
\setcounter{page}{1}
\renewcommand{\thepage}{\arabic{page}}

\def\thetable{S\arabic{table}}
\def\thefigure{S\arabic{figure}}
\begin{center}
\textbf{\LARGE Appendix} 
\end{center}

\medskip
The Appendix contains the technical details and proofs of the theoretical results stated in the paper.

\section{Definitions} \label{app_defi}
\begin{definition}[Covering number] \label{covering_number}
    For a given sequence $\mathbf{x}=\left(x_1, \ldots, x_n\right) \in \mathcal{Y}$, let $\mathcal{F}_n\vert_\mathbf{x}=\{(f(x_1), \ldots, f(x_n)): f \in \mathcal{F}_n\}$ be a subset of $\mathbb{R}^n$. For a positive number $\delta$, let $\mathcal{N}(\delta,\Vert \cdot \Vert_{\infty},\mathcal{F}_n \vert_\mathbf{x})$ denote the covering number of $\mathcal{F}_n\vert_\mathbf{x}$ under the norm $\Vert\cdot\Vert_{\infty}$ with radius $\delta$. Define the uniform covering number $\mathcal{N}_n(\delta, \Vert \cdot \Vert_{\infty}, \mathcal{F}_n)$ by the maximum of the covering number $\mathcal{N}(\delta,\Vert\cdot\Vert_{\infty},\mathcal{F}_n\vert_\mathbf{x})$ over all $\mathbf{x} \in \mathcal{Y}$, i.e.

    \begin{align*}
        \mathcal{N}_n\left(\delta,\Vert\cdot\Vert_{\infty}, \mathcal{F}_n\right)=\max \left\{\mathcal{N}\left(\delta,\Vert\cdot\Vert_{\infty},\mathcal{F}_n\vert_\mathbf{x}\right): \mathbf{x} \in \mathcal{Y} \right\}.
    \end{align*}
\end{definition}

\begin{definition}
     Let $\mathcal{F}$ be a set of functions mapping from a domain $X$ to $\mathbb{R}$ and suppose that $S=\left\{x_1, x_2, \ldots, x_m\right\} \subseteq X$. Then $S$ is pseudo-shattered by $\mathcal{F}$ if there are real numbers $r_1, r_2, \ldots, r_m$ such that for each $b \in\{0,1\}^m$ there is a function $f_b$ in $F$ with $\operatorname{sgn}\left(f_b\left(x_i\right)-r_i\right)=b_i$ for $1 \leq i \leq m$. We say that $r=\left(r_1, r_2, \ldots, r_m\right)$ witnesses the shattering.
\end{definition}

\begin{definition}[Pseudo-dimension]
    Suppose that $\mathcal{F}$ is a set of functions from a domain $X$ to $\mathbb{R}$. Then $\mathcal{F}$ has pseudo-dimension $d$ if $d$ is the maximum cardinality of a subset $S$ of $X$ that is pseudo-shattered by $\mathcal{F}$. If no such maximum exists, we say that $\mathcal{F}$ has infinite pseudo-dimension. The pseudo-dimension of $\mathcal{F}$ is denoted $\operatorname{Pdim}(\mathcal{F})$.
\end{definition}

\begin{definition}[Truncated normal distribution]
    Suppose $X \in \mathbb{R}^d$ follows the multivariate normal distribution with mean $\boldsymbol{\mu}$ and variance $\boldsymbol{\Sigma}$. Let $(a, b)$ be an interval with $-\infty \leq a<b \leq \infty$. Then $X$ conditional on $a<\Vert X \Vert_\infty<b$ has a truncated normal distribution with probability density function $f$ given by
    $$
    f(\mathbf{x} ; \boldsymbol{\mu}, \boldsymbol{\Sigma}, a, b)= \det(\boldsymbol{\Sigma})^{-1/2} \frac{\varphi\left(\boldsymbol{\Sigma}^{-1/2}(\mathbf{x}-\boldsymbol{\mu})\right)}{\Phi\left(\boldsymbol{\Sigma}^{-1/2}(b \mathbf{1}-\boldsymbol{\mu})\right)-\Phi\left(\boldsymbol{\Sigma}^{-1/2}({a \mathbf{1}-\boldsymbol{\mu}})\right)}, \ a<\Vert X \Vert_\infty<b
    $$
    and $f=0$ otherwise.
    Here,
    $$
    \varphi(\boldsymbol{\xi})=(2 \pi)^{-2/d} \exp \left(-\frac{1}{2} \Vert \boldsymbol{\xi} \Vert^2\right), \boldsymbol{\xi} \in \mathbb{R}^d
    $$
    is the probability density function of the multivariate standard normal distribution and $\Phi $ is its cumulative distribution function.
\end{definition}

\section{Proofs}\label{secA1}
For ease of reference, we first restate the results in the main text,
and then give the proofs.

\subsection{Proof of Theorem \ref{pro_b}}

\bigskip
\noindent
\newline

\noindent
Theorem \ref{pro_b} (Transport equation). \textit{
Suppose that Assumptions \ref{assumption0} and \ref{assumption0b} are satisfied.
Given $\mathbf{x} \in \mathcal{X}$, the time-dependent conditional density $\rho^*_\mathbf{x}: \mathbb{R}^d \times [0,1]\to [0, \infty)$ is absolutely continuous with respect to the Lebesgue measure for all $t \in[0,1]$ and satisfies $\rho^*_\mathbf{x} \in C^1\left([0,1]; C^p\left(\mathbb{R}^d\right)\right)$ for any $p \in \mathbb{N}$, $(t, \mathbf{y}_t) \in[0,1] \times \mathbb{R}^d$. In addition, $\rho^*_x$
satisfies the transport equation
\begin{align} \label{tp_equa}
        \partial_t \rho^*_{\mathbf{x}} +  \nabla_{\mathbf{y}} \cdot (\boldsymbol{b}^*_\mathbf{x} \rho^*_\mathbf{x}) = 0
\end{align}
in a weak sense, where $\partial_t \rho^*_{\mathbf{x}}$ denotes the partial derivative of $\rho^{*}_{\mathbf{x}}$ with respect to $t$, and $\nabla_{\mathbf{y}} \cdot $ denotes the divergence operator, i.e., $\nabla_{\mathbf{y}} \cdot (\boldsymbol{b}^{*}_{\mathbf{x}}\rho^{*}_{\mathbf{x}})=  \Sigma_{i=1}^d \frac{\partial}{\partial \mathbf{y}^{(i)}} (\boldsymbol{b}^{*}_{\mathbf{x}}\rho^{*}_{\mathbf{x}})^{(i)},$  with $\mathbf{y}^{(i)}$ representing the $i$-th component of $\mathbf{y}$.
}

\begin{proof}
For a given $\mathbf{x} \in \mathcal{X}$, we first prove the existence of the differentiability of $\rho_\mathbf{x}$. Following the idea in the proof in \citet{albergo2023stochastic}, let $g(t, \mathbf{\lambda})=\mathbb{E} [e^{i \mathbf{\lambda}^\top \mathbf{y}_t} \mid \mathbf{x}] $, where$ \mathbf{\lambda} \in \mathbb{R}^d$, be the characteristic function of $\rho_\mathbf{x}( \mathbf{y}_t,t)$.

Using the Definition \ref{sub_CSI} and the independence between $Y_0, \left(X, Y_1\right)$ and $\boldsymbol{\eta}$, we have

\begin{align} \label{equa_g}
    g_\mathbf{x}(t, \mathbf{\lambda})=\mathbb{E}[e^{i \mathbf{\lambda}^\top \mathcal{I}\left(t, \mathbf{y}_0, \mathbf{y}_1\right)} \mid \mathbf{x}] \mathbb{E}[e^{i \gamma(t) \mathbf{\lambda}^\top \boldsymbol{\eta}}] =: \varphi(t, \mathbf{\lambda}) e^{-\frac{1}{2} \gamma^2(t)\|\mathbf{\lambda}\|^2}
\end{align}
where The function $\varphi_\mathbf{x}(t, \lambda)$ is the characteristic function of random variable $\mathcal{I}\left(t, Y_0, Y_1\right)$ given $\mathbf{x}$, i.e., $\varphi_\mathbf{x}(t, \mathbf{\lambda}):=\mathbb{E}[e^{i \mathbf{\lambda}^\top  \mathcal{I}\left(t, x_0, x_1\right)} \mid \mathbf{x}]$.

Then, we have
$$
|g_\mathbf{x}(t, \mathbf{\lambda})|=\left|\varphi_\mathbf{x}(t, \mathbf{\lambda})\right| e^{-\frac{1}{2} \gamma^2(t)\|\lambda\|^2} \leq e^{-\frac{1}{2} \gamma^2(t)\|\mathbf{\lambda}\|^2}
$$

Since $\gamma(t)>0$ for all $t \in(0,1)$ by assumption \ref{assumption0}, this shows that $\forall p \in \mathbb{N} \text { and } t \in(0,1) \quad$

$$
\int_{\mathbb{R}^d}\|\mathbf{\lambda}\|^p|g_\mathbf{x}(t, \mathbf{\lambda})| d \mathbf{\lambda}<\infty
$$
implying that $\rho_\mathbf{x}(t, \cdot)$ is in $C^p\left(\mathbb{R}^d\right)$ for any $p \in \mathbb{N}$ and all $t \in(0,1)$. From (\ref{equa_g}), we also have
$$
\begin{aligned}
\left|\partial_t g_\mathbf{x}(t, \mathbf{\lambda})\right|^2 & =\left|\mathbb{E}\left[\left(i \lambda^\top \partial_t \mathcal{I}\left(t, \mathbf{y}_0, \mathbf{y}_1\right)-\gamma(t) \dot{\gamma}(t)\|\mathbf{\lambda}\|^2\right) e^{i \mathbf{\lambda}^\top \mathcal{I}\left(t, \mathbf{y}_0, \mathbf{y}_1\right)} \mid \mathbf{x}\right]\right|^2 e^{-\gamma^2(t)\|\mathbf{\lambda}\|^2} \\
& \leq 2\left(\|\mathbf{\lambda}\|^2 \mathbb{E}\left[\left|\partial_t \mathcal{I}\left(\mathbf{y}_0, \mathbf{y}_1,t\right)\right|^2 \mid \mathbf{x}\right]+|\gamma(t) \dot{\gamma}(t)|^2\|\mathbf{\lambda}\|^4\right) e^{-\gamma^2(t)\|\mathbf{\lambda}\|^2} \\
& \leq 2\left(\|\mathbf{\lambda}\|^2 M_1+4|\gamma(t) \dot{\gamma}(t)|^2\|\mathbf{\lambda}\|^4\right) e^{-\gamma^2(t)\|\mathbf{\lambda}\|^2}
\end{aligned}
$$

and

$$
\begin{aligned}
\left|\partial_t^2 g_\mathbf{x}(t, \mathbf{\lambda})\right|^2 \leq & 4\left(\|\mathbf{\lambda}\|^2 \mathbb{E}\left[\left|\partial_t^2 \mathcal{I}\left(\mathbf{y}_0, \mathbf{y}_1,t \right)\right|^2 \mid \mathbf{x}\right]+\left(|\dot{\gamma}(t)|^2+\gamma(t) \ddot{\gamma}(t)\right)^2\|\mathbf{\lambda}\|^4\right) e^{-\gamma^2(t)\|\mathbf{\lambda}\|^2} \\
& +8\left(\|\mathbf{\lambda}\|^2 \mathbb{E}\left[\left|\partial_t \mathcal{I}\left(\mathbf{y}_0, \mathbf{y}_1,t\right)\right|^4\mid \mathbf{x} \right]+(\gamma(t) \dot{\gamma}(t))^4\|\mathbf{\lambda}\|^8\right) e^{-\gamma^2(t)\|\mathbf{\lambda}\|^2} \\
\leq & 4\left(\|\mathbf{\lambda}\|^2 M_2+\left(|\dot{\gamma}(t)|^2+\gamma(t) \ddot{\gamma}(t)\right)^2\|\mathbf{\lambda}\|^4\right) e^{-\gamma^2(t)\|\mathbf{\lambda}\|^2} \\
& +8\left(\|\mathbf{\lambda}\|^2 M_1+(\gamma(t) \dot{\gamma}(t))^4\|\mathbf{\lambda}\|^8\right) e^{-\gamma^2(t)\|\mathbf{\lambda}\|^2}
\end{aligned}
$$
where in both cases we used Assumption \ref{assumption0b}(b) to get the last inequalities. These imply that

$$
\int_{\mathbb{R}^d}\|\mathbf{\lambda}\|^p\left|\partial_t g_\mathbf{x}(t, \mathbf{\lambda})\right| d \mathbf{\lambda}<\infty ; \quad \int_{\mathbb{R}^d}\|\mathbf{\lambda}\|^p\left|\partial_t^2 g_\mathbf{x}(t, \mathbf{\lambda})\right| d \mathbf{\lambda}<\infty
$$ for all $p \in \mathbb{N}, t \in(0,1)$ and $\mathbf{x} \in \mathcal{X}$, which indicates that $\partial_t \rho_\mathbf{x}(t, \cdot)$ and $\partial_t^2 \rho_\mathbf{x}(t, \cdot)$ are in $C^p\left(\mathbb{R}^d\right)$ for any $p \in \mathbb{N}$, i.e. $\rho_\mathbf{x} \in C^1\left((0,1) ; C^p\left(\mathbb{R}^d\right)\right)$ as claimed.

We now prove that the conditional density $\rho^*_\mathbf{x}(\mathbf{z},t)$ solves the continuity equation for $\rho: \mathbb{R}^d \times (0,1)$,
     \begin{align} \label{B0_1}
         \partial_t \rho+\nabla_\mathbf{z} \cdot\left(\boldsymbol{b}^*_\mathbf{x}\rho\right)=0,
     \end{align}
 where  $\boldsymbol{b}$ is the conditional drift function. Similar to the proof idea of \citet{liu2022flow}, we introduce a continuously differentiable test function  $\varphi: \mathbb{R}^d \rightarrow \mathbb{R}$ with compact support. Then for any given $\mathbf{x}\in\mathcal{X}$ and $t\in(0,1)$, we have
     \begin{align*}
         \int \left(\partial_t \rho^*_\mathbf{x} +\nabla_\mathbf{z} \cdot\left(\boldsymbol{b}^*_\mathbf{x} \rho^*_\mathbf{x} \right)\right) \varphi \mathrm{d} \mathbf{z} &= \int (\varphi \partial_t \rho^*_\mathbf{x} -\nabla \varphi \cdot \boldsymbol{b}^*_\mathbf{x} \rho^*_\mathbf{x}) \mathrm{d} \mathbf{z} \\
         &= \partial_t \mathbb{E}\left[ \varphi \left(Y_t\right) \mid X= \mathbf{x} \right]-\mathbb{E}\left[\nabla \varphi \left(Y_t\right) \cdot \boldsymbol{b}^*_\mathbf{x}(Y_t,t) \mid X=\mathbf{x}\right],
     \end{align*}
 where the first equation follows from integration by parts that $\int \varphi \nabla_\mathbf{z} \cdot\left(\boldsymbol{b}^*_\mathbf{x}\rho^*_\mathbf{x} \right)=-\int \nabla \varphi \cdot \boldsymbol{b}^*_\mathbf{x} \rho^*_\mathbf{x} $. Note that
     \begin{align} \label{B0_2}
         \partial_t \mathbb{E}\left[\varphi \left(Y_t\right) \mid X=\mathbf{x} \right]
         &=\mathbb{E}\left[\nabla \varphi \left(Y_t\right) \cdot \dot{Y}_t \mid X=\mathbf{x} \right]\\
         &= \mathbb{E}\left[\nabla \varphi \left(Y_t\right) \cdot \mathbb{E}\left[ \dot{Y}_t \mid X=\mathbf{x}, Y_t \right] \mid X=\mathbf{x}  \right]  \nonumber \\
        & =\mathbb{E}\left[\nabla \varphi \left(Y_t\right) \cdot \boldsymbol{b}^*_\mathbf{x}\left(Y_t, t\right) \mid X\right],\nonumber
     \end{align}
 where $\boldsymbol{b}^*_\mathbf{x}(\mathbf{z},t):=\mathbb{E}\left[\dot{Y}_t \mid X=\mathbf{x}, Y_t = \mathbf{z}\right]$ by definition. Then
 $$\int \left(\partial_t \rho^*_\mathbf{x} +\nabla_\mathbf{z} \cdot\left(\boldsymbol{b}^*_\mathbf{x} \rho^*_\mathbf{x} \right)\right) \varphi \mathbf{dz}=0$$
    for any $\varphi.$ This completes the proof.
\end{proof}

\subsection{Proof of Theorem \ref{pro_score}}
Before proving this theorem, we first prove a conditional version of  Tweedie's formula,
which is a straightforward extension of the original Tweedy's formula \citep{herbert1956empirical, efron2011tweedie}.

\begin{lemma}[Conditional Tweedie’s formula] \label{condition_tf}
    If random variables $Y \in
    \mathbb{R}^d$, $\boldsymbol{\mu} \in
    \mathbb{R}^d$ satisfies
    $$ \boldsymbol{\mu} \sim \mathbb{P}_{\boldsymbol{\mu}} \text{ and } Y \mid \boldsymbol{\mu} \sim \mathcal{N}(\boldsymbol{\mu}, \sigma^2 I_d),
    $$ and $\mathcal{Y}$ is the sample space of the exponential family, then for any condition $X \in \mathcal{X}$ satisfying $X \perp Y \mid \boldsymbol{\mu} $, the following formula holds:
    $$
    \mathbb{E}[\boldsymbol{\mu} \mid X=\mathbf{x}, Y=\mathbf{y}] = \mathbf{y}+ \sigma^2 \nabla \log p(\mathbf{y} \mid \mathbf{x})
    $$ and
    $$
    \operatorname{Cov}[\boldsymbol{\mu} \mid Y=\mathbf{y},X=\mathbf{x}[= I_d + \sigma^2 \mathbf{H}( \log{p(\mathbf{y} \mid \mathbf{x})}),
    $$ where $\mathbf{H}(f)$ represents Hessian matrix  of function $f: \mathbb{R}^d \rightarrow \mathbb{R}$, and $p(\mathbf{y} \mid \mathbf{x})$ denotes the conditional density function of $Y$ given $X=\mathbf{x}$.
\end{lemma}

\begin{proof}
 Let $\boldsymbol{\upsilon}:= \frac{1}{\sigma^2} \boldsymbol{\mu}$.
The conditional density function of $Y$ given $\boldsymbol{\upsilon}$ is
\begin{align} \label{lemma_p1}
        q(\mathbf{y}\mid \boldsymbol{\upsilon})=\exp({\boldsymbol{\upsilon}^\top \mathbf{y} -\psi(\boldsymbol{\upsilon})}) f_0(\mathbf{y}).
\end{align}
Here, $\boldsymbol{\upsilon}$ is the natural or canonical parameter of the family, $\psi(\boldsymbol{\upsilon})$ is the cumulant generating function (which makes $q(\mathbf{y}\mid \boldsymbol{\upsilon})$ integrate to 1 ) and $\psi(\boldsymbol{\upsilon})=\frac{1}{2} \sigma^2 \Vert \boldsymbol{\upsilon} \Vert^2$, $f_0(\mathbf{y})$ is $\mathcal{N}(0, \sigma^2 I_d)$ density function, which yields the normal translation family $\mathcal{N}(\boldsymbol{\mu}, \sigma^2 I_d)$. 
By Bayes' theorem,   the posterior density $g(\boldsymbol{\upsilon} \mid \mathbf{y}, \mathbf{x})$,
\begin{align} \label{lemma_p2}
    g(\boldsymbol{\upsilon} \mid \mathbf{y}, \mathbf{x})= \frac{\tilde{q}(\mathbf{y}\mid \boldsymbol{\upsilon}, \mathbf{x}) g(\boldsymbol{\upsilon} \mid \mathbf{x})}{p(\mathbf{y} \mid  \mathbf{x})} = \frac{q(\mathbf{y}\mid \boldsymbol{\upsilon}) g(\boldsymbol{\upsilon} \mid \mathbf{x})}{p(\mathbf{y} \mid  \mathbf{x})},
\end{align}
where $\tilde{q}(\mathbf{y}\mid \boldsymbol{\upsilon}, \mathbf{x})$ denotes the conditional density function of $Y$ given $(\boldsymbol{\upsilon}, X)$, $g(\boldsymbol{\upsilon} \mid \mathbf{x})$ is the conditional density function of $\boldsymbol{\upsilon}$ given $X$.  Since $X \perp Y \mid \boldsymbol{\mu} $, $\tilde{q}(\mathbf{y}\mid \boldsymbol{\upsilon}, \mathbf{x}) = q(\mathbf{y}\mid \boldsymbol{\upsilon})$  and
\begin{align} \label{lemma_p3}
    p(\mathbf{y} \mid \mathbf{x}) = \int \tilde{q}(\mathbf{y} \mid \boldsymbol{\upsilon}, \mathbf{x}) g(\boldsymbol{\upsilon} \mid \mathbf{x}) \mathbf{d} \mathbf{\boldsymbol{\upsilon}}  =  \int q(\mathbf{y} \mid \boldsymbol{\upsilon}) g(\boldsymbol{\upsilon} \mid \mathbf{x}) \mathbf{d} \mathbf{\boldsymbol{\upsilon}}.
\end{align}
Then (\ref{lemma_p1}) and (\ref{lemma_p2}) imply
\begin{align}
    g(\boldsymbol{\upsilon} \mid \mathbf{y}, \mathbf{x}) = \exp(\boldsymbol{\upsilon}^\top\mathbf{y} - \lambda_\mathbf{x}(\mathbf{y}))[g(\boldsymbol{\upsilon} \mid \mathbf{x})\exp(-\psi(\boldsymbol{\upsilon}))],
\end{align}
where $\lambda_\mathbf{x}(\mathbf{y}) := \log({p(\mathbf{y} \mid \mathbf{x})}/{f_0(\mathbf{y})})$. This is an exponential family with canonical parameters $\mathbf{y}$, $\mathbf{x},$ and the cumulant generating function  $\lambda_\mathbf{x}(\mathbf{y}).$
It follows from the properties of an exponential family that
$$
E[\boldsymbol{\upsilon} \mid \mathbf{y}, \mathbf{x}]=\nabla \lambda_\mathbf{x}(\mathbf{y}) = \nabla \log{p(\mathbf{y} \mid \mathbf{x})} - \nabla \log{f_0(\mathbf{y})} = \nabla \log{p(\mathbf{y} \mid \mathbf{x})} + \frac{1}{\sigma^2} \mathbf{y}
$$ and
$$
\operatorname{Cov}[\boldsymbol{\upsilon} \mid \mathbf{y},\mathbf{x}]= \mathbf{H}( \lambda_\mathbf{x}(\mathbf{y})) = \mathbf{H}(\log{p(\mathbf{y} \mid \mathbf{x})}) + \frac{1}{\sigma^2} I_d.
$$
It follows that
$$
E[\boldsymbol{\mu} \mid \mathbf{y}, \mathbf{x}] = \mathbf{y} + \sigma^2  \nabla \log{p(\mathbf{y} \mid \mathbf{x})}
$$ and
$$
\operatorname{Cov}[\boldsymbol{\mu} \mid \mathbf{y},\mathbf{x}]= I_d + \sigma^2 \mathbf{H}(\log{p(\mathbf{y} \mid \mathbf{x})}).
$$
This completes the proof.
\end{proof}

\noindent
Theorem \ref{pro_score} (Conditional score function).
\textit{Suppose that Assumptions \ref{assumption0} and \ref{assumption0b} are satisfied.
If $\gamma(t) \neq 0$ for every $t \in (0, 1)$, then the conditional score function $\boldsymbol{s}^*$ can be expressed as
\begin{align} \label{score_equa}
\boldsymbol{s}^{*}_{\mathbf{x}}(\mathbf{y},t):=\boldsymbol{s}^*(\mathbf{x},\mathbf{y},t)
    = -\frac{1}{\gamma(t)} \mathbb{E}\left[\boldsymbol{\eta} \mid X= \mathbf{x}, Y_t=\mathbf{y} \right],\  t \in (0, 1).
\end{align}
}

\begin{proof}
First, we have
    \begin{align*}
        Y_t \mid \mathcal{I}(Y_0,Y_1,t) \sim \mathcal{N}(\mathcal{I}(Y_0,Y_1,t), \gamma(t)^2 I_d).
    \end{align*}
    Since $X \perp Y_t \mid \mathcal{I}(Y_0,Y_1,t)$, by Lemma \ref{condition_tf} we have
    \begin{align}
    \label{s_00a}
        \mathbb{E}[\mathcal{I}(Y_0,Y_1,t) \mid X, Y_t] = Y_t + \gamma(t)^2 \boldsymbol{s}^*(X,Y_t,t).
    \end{align}
Taking the conditional expectation of both sides of  (\ref{cond_stoch_interp_equa}) with respect to $X$ and $Y_t$, we obtain
    \begin{align} \label{s_00}
        Y_t = \mathbb{E}[\mathcal{I}(Y_0,Y_1,t) \mid X, Y_t] + \gamma(t) \mathbb{E}[\eta \mid X, Y_t].
    \end{align}
Combining (\ref{s_00a}) and (\ref{s_00}), we have
    \begin{align} \label{s_00000}
        \boldsymbol{s}^*(X,Y_t,t) = - \frac{1}{\gamma(t)}\mathbb{E}[\eta \mid X, Y_t], \ t \in (0, 1).
    \end{align}
This completes the proof.

\end{proof}

\subsection{Proof of Theorem \ref{boundary_score}}

\noindent
\newline
Theorem \ref{boundary_score} (Boundary problem). \textit{
Suppose that Assumptions \ref{assumption0} and \ref{assumption0b} are satisfied and $Y_0 \sim \mathcal{N}(\boldsymbol{0},I_d)$ with its density function denoted by $p_{Y_0}.$
For a fixed $\mathbf{x}\in \mathcal{X}, $  let $p_{\mathcal{I}(Y_0,Y_1,t) \mid \mathbf{x}}$ denote the conditional density function of  $\mathcal{I}(Y_0,Y_1,t)$ given  $ X=\mathbf{x}.$
Suppose $\Vert p_{\mathcal{I}(Y_0,Y_1,t) \mid \mathbf{x}} - p_{Y_0}\Vert_\infty = o(\gamma(t))$ as $t\to 0$  for any given $\mathbf{x} \in\mathcal{X}$, then the limitation $\lim_{t \to 0}\boldsymbol{s}^*_{\mathbf{x}}(\mathbf{y},t)$ exits with
    $
    \boldsymbol{s}^*_{\mathbf{x}}(\mathbf{y},0) = -\mathbf{y},
    $
    and the value of the drift function at $t = 0$ is
    $    \boldsymbol{b}^*_{\mathbf{x}}(\mathbf{y},0) = \mathbb{E}[\partial_t {\mathcal{I}}(\mathbf{y},Y_1, 0) \mid  X = \mathbf{x}].
    $
}
\begin{proof}
 {Let $p_{\mathcal{I}(Y_0,Y_1,t)\mid \mathbf{x}} $ denote the conditional density function of $\mathcal{I}(Y_0,Y_1,t)$ given $X=\mathbf{x}$. By definition of the score function (\ref{score_equa}), and the independence between $Y_0$ and  $(X, Y_1)$, we have}
    \begin{align*}
        \boldsymbol{s}(\mathbf{x},\mathbf{y},t) &= -\frac{1}{\gamma(t)} \mathbb{E}\left(\boldsymbol{\eta} \mid X= \mathbf{x}, \mathcal{I}(Y_0,Y_1,t)+\gamma(t)\boldsymbol{\eta}= \mathbf{y} \right)\\
        &= -\frac{1}{\gamma(t)} \displaystyle \int \boldsymbol{\eta} \frac{p_{\mathcal{I}(Y_0,Y_1,t)\mid \mathbf{x}}(\mathbf{y}-\gamma(t)\boldsymbol{\eta})p_{\boldsymbol{\eta}}(\boldsymbol{\eta})} { \int p_{\mathcal{I}(Y_0,Y_1,t)\mid \mathbf{x}}(\mathbf{y}-\gamma(t)\mathbf{z})p_{\boldsymbol{\eta}} (\mathbf{z}) \mathbf{dz}} d \boldsymbol{\eta} \\
        &=  -\frac{1}{\gamma(t)} \frac{ \int {\boldsymbol{\eta}} p_{\mathcal{I}(Y_0,Y_1,t)\mid \mathbf{x}}(\mathbf{y}-\gamma(t){\boldsymbol{\eta}})p_{\boldsymbol{\eta}}({\boldsymbol{\eta}}) d{\boldsymbol{\eta}} }{\int p_{\mathcal{I}(Y_0,Y_1,t)\mid \mathbf{x}}(\mathbf{y}-\gamma(t)\mathbf{z})p_{\boldsymbol{\eta}}(\mathbf{z}) \mathbf{dz} }.
    \end{align*}
 Next, we derive bounds for the components of the score function ${\boldsymbol{s}} $. Let
 \begin{align*}
 	A(\mathbf{x},\mathbf{y},t,{\boldsymbol{\eta}})&=p_{Y_0}(\mathbf{y}-\gamma(t){\boldsymbol{\eta}}) p_{\boldsymbol{\eta}}({\boldsymbol{\eta}})\\
 	B(\mathbf{x},\mathbf{y},t,{\boldsymbol{\eta}})&=(p_{\mathcal{I}(Y_0,Y_1,t)\mid \mathbf{x}}(\mathbf{y}-\gamma(t){\boldsymbol{\eta}}) - p_{Y_0}(\mathbf{y}-\gamma(t){\boldsymbol{\eta}})) p_{\boldsymbol{\eta}}({\boldsymbol{\eta}}).
 \end{align*}
 We note that
  \begin{align} \label{b3_0}
 	\left\Vert \int B(\mathbf{x},\mathbf{y},\boldsymbol{\eta},t) \mathbf{d}\boldsymbol{\eta} \right\Vert_\infty\le \Vert p_{\mathcal{I}(Y_0,Y_1,t)\mid \mathbf{x}} - p_{Y_0}\Vert_\infty
 \end{align}
and
 \begin{align*}
 	\left\Vert\int \boldsymbol{\eta} B(\mathbf{x},\mathbf{y},\boldsymbol{\eta},t) \mathbf{d}\boldsymbol{\eta} \right\Vert_\infty
 	&=\left\Vert  \int \boldsymbol{\eta} \left(p_{\mathcal{I}(Y_0,Y_1,t)\mid \mathbf{x}}(\mathbf{y}-\gamma(t)\boldsymbol{\eta}) - p_{Y_0}(\mathbf{y}-\gamma(t)\boldsymbol{\eta})\right) p_{\boldsymbol{\eta}}(\boldsymbol{\eta}) d \boldsymbol{\eta} \right\Vert_\infty \\
 	&\leq \int \Vert \boldsymbol{\eta} \Vert_1 \left\vert p_{\mathcal{I}(Y_0,Y_1,t)\mid \mathbf{x}}(\mathbf{y}-\gamma(t)\boldsymbol{\eta}) - p_{Y_0}(\mathbf{y}-\gamma(t)\boldsymbol{\eta})\right\vert p_{\boldsymbol{\eta}}(\boldsymbol{\eta}) d \boldsymbol{\eta} \\
 	& \leq C_{\boldsymbol{\eta}} \times \Vert p_{\mathcal{I}(Y_0,Y_1,t)\mid \mathbf{x}} - p_{Y_0}\Vert_\infty,
 \end{align*}
 where $C_{\boldsymbol{\eta}} :=  \int \Vert \boldsymbol{\eta} \Vert_\infty p_{\boldsymbol{\eta}}(\boldsymbol{\eta}) \mathbf{d}\boldsymbol{\eta}<\infty$ is a finite constant given $\boldsymbol{\eta} \sim \mathcal{N}(0,I_d)$. Given $\gamma(t)\geq 0$, then for any $j\in\{1,\ldots,d\}$, the $j$th component of $\boldsymbol{s}(\mathbf{x},\mathbf{y},t)$ satisfies
 {\small
    \begin{align} \notag
    &    \boldsymbol{s}(\mathbf{x},\mathbf{y},t)_j  = -\frac{1}{\gamma(t)}\frac{(\int{\boldsymbol{\eta}} [A(\mathbf{x},\mathbf{y},t,{\boldsymbol{\eta}})+B(\mathbf{x},\mathbf{y},t,{\boldsymbol{\eta}})]d{\boldsymbol{\eta}})_j }{\int [A(\mathbf{x},\mathbf{y},t,\mathbf{z})+B(\mathbf{x},\mathbf{y},t,\mathbf{z})]\mathbf{dz}}\\ \notag
        &=-\frac{1}{\gamma(t)} \frac{\int A(\mathbf{x},\mathbf{y},t,\mathbf{z})\mathbf{dz}}{\int [A(\mathbf{x},\mathbf{y},t,\mathbf{z})+B(\mathbf{x},\mathbf{y},t,\mathbf{z})]\mathbf{dz}} \left(\frac{ (\int {\boldsymbol{\eta}} A(\mathbf{x},\mathbf{y},t,{\boldsymbol{\eta}})d{\boldsymbol{\eta}})_j + (\int{\boldsymbol{\eta}} B(\mathbf{x},\mathbf{y},t,{\boldsymbol{\eta}})d{\boldsymbol{\eta}})_j}{\int A(\mathbf{x},\mathbf{y},t,\mathbf{z})\mathbf{dz}}\right)\\ \label{b3_1}
        &\leq -\frac{1}{\gamma(t)}\frac{\int A(\mathbf{x},\mathbf{y},t,\mathbf{z})\mathbf{dz}}{\int A(\mathbf{x},\mathbf{y},t,\mathbf{z})\mathbf{dz}+\Vert \int B(\mathbf{x},\mathbf{y},t,\mathbf{z})\mathbf{dz}\Vert_\infty}
        \left(\frac{(\int {\boldsymbol{\eta}} A(\mathbf{x},\mathbf{y},t,{\boldsymbol{\eta}})d{\boldsymbol{\eta}})_j - \Vert \int{\boldsymbol{\eta}} B(\mathbf{x},\mathbf{y},t,{\boldsymbol{\eta}})d{\boldsymbol{\eta}}\Vert_1}{\int A(\mathbf{x},\mathbf{y},t,\mathbf{z})\mathbf{dz}}\right).
    \end{align}
 }
Based on the condition that $\Vert p_{\mathcal{I}(Y_0,Y_1,t) \mid X} - p_{Y_0}\Vert_\infty = o(\gamma(t))$ and inequality (\ref{b3_0}), we know there exists a neighborhood $(0, \epsilon)$ of $t=0$ such that ${\int A(\mathbf{x},\mathbf{y},t,\mathbf{z})\mathbf{dz}-\Vert \int B(\mathbf{x},\mathbf{y},t,\mathbf{z})\mathbf{dz}\Vert_\infty} >0$. Then for $t\in(0,\epsilon)$, we can similarly obtain
{\small
    \begin{align} \label{b3_2}
    &    \boldsymbol{s}(\mathbf{x},\mathbf{y},t)_j \nonumber \\
    & \geq -\frac{1}{\gamma(t)}\frac{\int A(\mathbf{x},\mathbf{y},t,\mathbf{z})\mathbf{dz}}{\int A(\mathbf{x},\mathbf{y},t,\mathbf{z})\mathbf{dz}-\Vert \int B(\mathbf{x},\mathbf{y},t,\mathbf{z})\mathbf{dz}\Vert_\infty} \left(\frac{(\int {\boldsymbol{\eta}} A(\mathbf{x},\mathbf{y},t,{\boldsymbol{\eta}})d{\boldsymbol{\eta}})_j + \Vert \int{\boldsymbol{\eta}} B(\mathbf{x},\mathbf{y},t,{\boldsymbol{\eta}})d{\boldsymbol{\eta}}\Vert_1}{\int A(\mathbf{x},\mathbf{y},t,\mathbf{z})\mathbf{dz}}\right).
    \end{align}
}
    Now recall that $Y_0$ and ${\boldsymbol{\eta}}$ both follow standard Gaussian distributions and we find that
    \begin{align*}
        \frac{A(\mathbf{x},\mathbf{y},t,{\boldsymbol{\eta}})}{\int A(\mathbf{x},\mathbf{y},t,\mathbf{z}) \mathbf{dz}}&=\frac{p_{Y_0}(\mathbf{y}-\gamma(t){\boldsymbol{\eta}})p_{\boldsymbol{\eta}}({\boldsymbol{\eta}})}{ \int p_{Y_0}(\mathbf{y}-\gamma(t)\mathbf{z})p_{\boldsymbol{\eta}}(z) \mathbf{dz} }\\
        &\propto \exp \left\{ -\frac{1}{2}\left( \Vert \mathbf{y} - \gamma(t){\boldsymbol{\eta}} \Vert_2^2 + \Vert {\boldsymbol{\eta}} \Vert_2^2 \right) \right\}\\
        &= \exp \left\{ -\frac{1}{2}\left( (\gamma(t)^2+1) \Vert {\boldsymbol{\eta}} \Vert_2^2 + \Vert \mathbf{y} \Vert_2^2 -2\gamma(t)\mathbf{y} \cdot {\boldsymbol{\eta}}  \right) \right\} \\
        &\propto \exp \left(-\frac{\gamma(t)^2+1}{2} \Vert {\boldsymbol{\eta}} - \frac{\gamma(t)}{\gamma(t)^2+1}\mathbf{y} \Vert_2^2 \right),
    \end{align*}
	{which is a kernel of a Gaussian density function. Therefore, we have}
    \begin{align} \label{b3_3}
        \int {\boldsymbol{\eta}}\frac{ A(\mathbf{x},\mathbf{y},t,{\boldsymbol{\eta}})}{\int A(\mathbf{x},\mathbf{y},t,\mathbf{z}) \mathbf{dz}}d{\boldsymbol{\eta}}= \int {\boldsymbol{\eta}} \frac{p_{Y_0}(\mathbf{y}-\gamma(t){\boldsymbol{\eta}})p_{\boldsymbol{\eta}}({\boldsymbol{\eta}})}{\int p_{Y_0}(\mathbf{y}-\gamma(t)\mathbf{z})p_{\boldsymbol{\eta}}(\mathbf{z}) \mathbf{dz} } d{\boldsymbol{\eta}} = \frac{\gamma(t)}{\gamma(t)^2+1}\mathbf{y}.
    \end{align}
    Also, we can calculate the term
{\small
    \begin{align} \label{b3_4}
        &\int A(\mathbf{x},\mathbf{y},t,\mathbf{z}) \mathbf{dz} =  \int p_{Y_0}(\mathbf{y}-\gamma(t)\mathbf{z})p_{\boldsymbol{\eta}}(\mathbf{z}) \mathbf{dz} \nonumber \\
        =& \displaystyle \int \frac{1}{(2\pi)^d} \exp \left\{ -\frac{1}{2}\left( (\gamma(t)^2+1) \Vert \mathbf{z} \Vert_2^2 + \Vert \mathbf{y} \Vert_2^2 -2\gamma(t)\mathbf{y} \cdot \mathbf{z}  \right) \right\} \mathbf{dz} \nonumber \\
        =& \frac{1}{(2\pi)^d} \exp \left(-\frac{\Vert \mathbf{y} \Vert_2^2}{2}\right) \displaystyle \int \exp\left\{-\frac{\gamma(t)^2+1}{2} \left(\Vert \mathbf{z} - \frac{\gamma(t)}{\gamma(t)^2+1}\mathbf{y} \Vert_2^2 - \left( \frac{\gamma(t)}{\gamma(t)^2+1}\right)^2 \Vert \mathbf{y} \Vert_2^2\right) \right\} \mathbf{dz}
        \nonumber \\
        =& \frac{1}{(2\pi)^d} \exp \left(-\frac{\Vert \mathbf{y} \Vert_2^2}{2(\gamma(t)^2+1)}\right) \displaystyle \int \exp\left(-\frac{\gamma(t)^2+1}{2} \Vert \mathbf{z} - \frac{\gamma(t)}{\gamma(t)^2+1}\mathbf{y} \Vert_2^2  \right) \mathbf{dz} \nonumber \\
        =& (2\pi)^{-\frac{d}{2}} \left(\gamma(t)^2+1\right)^{-\frac{d}{2}} \exp \left(-\frac{\Vert \mathbf{y} \Vert_2^2}{2(\gamma(t)^2+1)}\right).
    \end{align}
}
    Plugging (\ref{b3_3}), (\ref{b3_4}) into (\ref{b3_1}) and (\ref{b3_2}), for $j\in\{1,\ldots,d\}$ and $t$ near 0, we have
    \begin{align*}
      &  \boldsymbol{s}(\mathbf{x},\mathbf{y},t)_j\\
       &\leq - \left( \frac{1}{\gamma(t)^2+1}\mathbf{y}_j - C_{\boldsymbol{\eta}} (2\pi)^{\frac{d}{2}}(\gamma(t)^2+1)^{\frac{d}{2}} \exp \left(\frac{\Vert \mathbf{y} \Vert_2^2}{2(\gamma(t)^2+1)}\right) \frac{\Vert p_{\mathcal{I}(Y_0,Y_1,t)\mid \mathbf{x}} - p_{Y_0}\Vert_\infty}{\gamma(t)} \right)\\
        &\quad\quad\quad \cdot  \frac{(2\pi)^{-\frac{d}{2}} \left(\gamma(t)^2+1\right)^{-\frac{d}{2}} \exp \left(-\frac{\Vert \mathbf{y} \Vert_2^2}{2(\gamma(t)^2+1)}\right)}{(2\pi)^{-\frac{d}{2}} \left(\gamma(t)^2+1\right)^{-\frac{d}{2}} \exp \left(-\frac{\Vert \mathbf{y} \Vert_2^2}{2(\gamma(t)^2+1)}\right)+\Vert p_{\mathcal{I}(Y_0,Y_1,t)\mid \mathbf{x}} - p_{Y_0}\Vert_\infty}
    \end{align*}
    and
    \begin{align*}
     &   \boldsymbol{s}(\mathbf{x},\mathbf{y},t)_j \\
     &\geq - \left( \frac{1}{\gamma(t)^2+1}\mathbf{y}_j + C_{\boldsymbol{\eta}} (2\pi)^{\frac{d}{2}}(\gamma(t)^2+1)^{\frac{d}{2}} \exp \left(\frac{\Vert \mathbf{y} \Vert_2^2}{2(\gamma(t)^2+1)}\right) \frac{\Vert p_{\mathcal{I}(Y_0,Y_1,t)\mid \mathbf{x}} - p_{Y_0}\Vert_\infty}{\gamma(t)} \right)\\
        &\quad\quad\quad \times  \frac{(2\pi)^{-\frac{d}{2}} \left(\gamma(t)^2+1\right)^{-\frac{d}{2}} \exp \left(-\frac{\Vert \mathbf{y} \Vert_2^2}{2(\gamma(t)^2+1)}\right)}{(2\pi)^{-\frac{d}{2}} \left(\gamma(t)^2+1\right)^{-\frac{d}{2}} \exp \left(-\frac{\Vert \mathbf{y} \Vert_2^2}{2(\gamma(t)^2+1)}\right)-\Vert p_{\mathcal{I}(Y_0,Y_1,t)\mid \mathbf{x}} - p_{Y_0}\Vert_\infty}.
    \end{align*}
    Recall that  $\gamma(0)=0$ and $\Vert p_{\mathcal{I}(Y_0,Y_1,t)\mid \mathbf{x}} - p_{Y_0}\Vert_\infty = o(\gamma(t))$ as $t$ tends toward $0$. Let $t$ in above two inequalities approach to $0$ and we can obtain
    \begin{align} \label{b3_5}
        -\mathbf{y}_j \le \lim_{t \to 0} \boldsymbol{s}(\mathbf{x},\mathbf{y},t)_j\le -\mathbf{y}_j,
    \end{align}
    for any $j\in\{1,\ldots,d\}$.
    Therefore, the limit exists and $\boldsymbol{s}(\mathbf{x},\mathbf{y},0) = -\mathbf{y}$ is well-defined. Based on the equation (\ref{b3_5}) (\ref{drif_equa}), we can also obtain the value of drift function at $t = 0$, i.e.,
\begin{align*}
    \boldsymbol{b}(\mathbf{x},\mathbf{y},0) &= \mathbb{E}[\partial_t \mathcal{I}(Y_0,Y_1,0) \mid X = \mathbf{x}, Y_0 = \mathbf{y}] + \dot{\gamma}(0){\gamma}(0)\mathbf{y} \\
    &= \mathbb{E}[\partial_t \mathcal{I}(\mathbf{y},Y_1,0) \mid X = \mathbf{x}],
\end{align*} where the last equality holds because  $\gamma(0)=0.$
\end{proof}

\subsection{Proof of Corollary \ref{pro_score2}}
\noindent
\newline

\noindent
Corollary \ref{pro_score2}. \textit{
Suppose that Assumptions \ref{assumption0} and \ref{assumption0b} hold, $Y_0\sim \mathcal{N}(\mathbf{0},\mathbf{I}_d) $ and the interpolation takes the additive form (\ref{interpb}). Denote
$$
A(t):= a(t)[a(t)\dot{b}(t)-\dot{a}(t)b(t)] + \gamma(t)[\gamma(t)\dot{b}(t)-\dot{\gamma}(t)b(t)].
$$
Then the conditional score function can be expressed as
     \begin{align} \label{score_b}
      \boldsymbol{s}^{*}_{\mathbf{x}}(\mathbf{y},t):=
        \frac{b(t)}{A(t)} \boldsymbol{b}^*(\mathbf{x},\mathbf{y},t) - \frac{\dot{b}(t)}{A(t)} \mathbf{y}, \ t \in (0, 1),
    \end{align}
provided that the function ${b(t)}/{A(t)}$ and ${\dot{b}(t)}/{A(t)}$ are well-defined in $(0,1).$
}

\begin{proof}
    Since $\mathcal{I}(Y_0,Y_1,t) = a(t) Y_0 + b(t) Y_1$ and $Y_0 \sim \mathcal{N}(\mathbf{0},I_d),$ then we have $Y_t = a(t) Y_0 + b(t) Y_1 + \gamma(t)\boldsymbol{\eta}$ and  $ a(t) Y_0 + \gamma(t)\boldsymbol{\eta} \sim \mathcal{N}(\mathbf{0}, (a(t)^2 + \gamma(t)^2) I_d).$

    Using Lemma \ref{condition_tf} again, we have
    \begin{align} \label{s_1}
        b(t)\mathbb{E}[Y_1 \mid X, Y_t] = Y_t + (a(t)^2 + \gamma(t)^2) \boldsymbol{s}^*(X,Y_t,t).
    \end{align}

    In this case, the drift function is
    \begin{align} \label{s_000}
        \boldsymbol{b}^*(X,Y_t,t) &= \mathbb{E}[\dot{a}(t) Y_0 + \dot{b}(t)Y_1 + \dot{\gamma}(t)\boldsymbol{\eta} \mid X,Y_t,t] \nonumber\\
        &= \dot{a}(t) \mathbb{E}[ Y_0 \mid X,Y_t,t] + \dot{b}(t) \mathbb{E}[ Y_1 \mid X,Y_t,t] + \dot{\gamma}(t) \mathbb{E}[ \boldsymbol{\eta} \mid X,Y_t,t],
    \end{align}
and (\ref{s_00}) can be written as
    \begin{align} \label{s_0000}
        Y_t = {a}(t) \mathbb{E}[ Y_0 \mid X,Y_t] + {b}(t) \mathbb{E}[ Y_1 \mid X,Y_t] + {\gamma}(t) \mathbb{E}[ \boldsymbol{\eta} \mid X,Y_t].
    \end{align}

    By solving the system of linear equations given by (\ref{s_00000}), (\ref{s_1}), (\ref{s_000}), and (\ref{s_0000}), we obtain
    \begin{align*}
        \boldsymbol{s}^*(\mathbf{x},\mathbf{y},t) = \frac{b(t)}{A(t)} \boldsymbol{b}^*(\mathbf{x},\mathbf{y},t) - \frac{\dot{b}(t)}{A(t)} \mathbf{y},
    \end{align*}
    when factors ${b(t)}/{A(t)}$ and ${\dot{b}(t)}/{A(t)}$ are well-defined on $[0,1]$.
\end{proof}

\subsection{Proof of Corollary \ref{corollary_boundary_score}}

\noindent
\newline

\noindent
    Corollary \ref{corollary_boundary_score}. \textit{
    Suppose that assumptions \ref{assumption0}, \ref{assumption0b}, \ref{assump_ab} are satisfied, $Y_0\sim \mathcal{N}(\mathbf{0},\mathbf{I}_d)$, and the interpolation takes the additive form (\ref{interpb}). Then, $\Vert p_{\mathcal{I}(Y_0, Y_1,t) \mid \mathbf{x}} - p_{Y_0}\Vert_\infty = o(\gamma(t))$ and $\boldsymbol{s}^*_{\mathbf{x}}(\mathbf{y},0) = -\mathbf{y}$, $    \boldsymbol{b}^*_{\mathbf{x}}(\mathbf{y},0) = \mathbb{E}[\partial_t {\mathcal{I}}(\mathbf{y},Y_1, 0) \mid  X = \mathbf{x}].
    $
}
\begin{proof}
   We show that the conditions in Corollary \ref{corollary_boundary_score} are sufficient to establish $\Vert p_{\mathcal{I}(Y_0,Y_1,t)\mid \mathbf{x}} - p_{Y_0}\Vert_\infty = o(\gamma(t))$ as $t$ tends toward $0$.

   Firstly, the conditions $1-a(t)=o(\gamma(t))$, $b(t) = o(\gamma(t))$ and $\gamma(0)=0$ imply that $a(0)=1$ and $b(0)=0$. Since $a(t)$ is continuous at $0$ and $a(0)>0$, then there exists a neighborhood $[0,\epsilon)$ such that $a(t)>0$ for all $t\in[0,\epsilon)$. We also note that $\mathcal{I}(Y_0,Y_1,t) = a(t)Y_0 + b(t)Y_1$ and $\Vert Y_1 \Vert \leq B_1$. Then for all $t\in[0,\epsilon)$, we have
    \begin{align} \label{b4_1}
     &   \Vert p_{\mathcal{I}(Y_0,Y_1,t)\mid \mathbf{x}}(\mathbf{y}) - p_{Y_0}(\mathbf{y}) \Vert_\infty \nonumber \\
     &= \max_\mathbf{y} \left\vert \displaystyle \int_{\Vert \mathbf{z} \Vert\leq b(t)B_1} p_{Y_0}\left(\frac{\mathbf{y}-\mathbf{z}}{a(t)}\right)p_{b(t)Y_1 \mid \mathbf{x}}(\mathbf{z}) \mathbf{dz} - p_{Y_0}(\mathbf{y}) \right\vert \nonumber\\
        &\leq \displaystyle \int_{\Vert \mathbf{z} \Vert\leq b(t)B_1} \max_\mathbf{y} \left\vert p_{Y_0}\left(\frac{\mathbf{y}-\mathbf{z}}{a(t)}\right) - p_{Y_0}(\mathbf{y})  \right\vert p_{b(t)Y_1 \mid \mathbf{x}}(\mathbf{z}) \mathbf{dz}\nonumber\\
        &\leq \displaystyle \int_{\Vert \mathbf{z} \Vert\leq b(t)B_1} \max_\mathbf{y} \max_{\Vert \mathbf{u} \Vert\leq b(t)B_1} \left\vert p_{Y_0}\left(\frac{\mathbf{y}-\mathbf{u}}{a(t)}\right) - p_{Y_0}(\mathbf{y})  \right\vert p_{b(t)Y_1 \mid \mathbf{x}}(\mathbf{z}) \mathbf{dz}\nonumber\\
        &\leq \max_y \max_{\Vert \mathbf{u} \Vert\leq b(t)B_1} \left\vert p_{Y_0}\left(\frac{\mathbf{y}-\mathbf{u}}{a(t)}\right) - p_{Y_0}(\mathbf{y})  \right\vert.
    \end{align}

  Define $\mathbf{v} = (\mathbf{y}-\mathbf{u})/a(t)-\mathbf{y}$, then (\ref{b4_1}) can be rewritten as
    \begin{align} \label{b4_2}
        \Vert p_{\mathcal{I}(Y_0,Y_1,t)\mid \mathbf{x}}(\mathbf{y}) - p_{Y_0}(\mathbf{y}) \Vert_\infty \leq \max_\mathbf{y} \max_{\Vert \mathbf{v} \Vert\leq r_t(\mathbf{y})} \left\vert p_{Y_0}\left(\mathbf{y}+\mathbf{v}\right) - p_{Y_0}(\mathbf{y})  \right\vert,
    \end{align} where $r_t(\mathbf{y})$ is defined by
    \begin{align*}
        r_t(\mathbf{y}):= \frac{\vert 1-a(t) \vert}{a(t)}\Vert \mathbf{y} \Vert + \frac{\vert b(t)\vert}{a(t)} B_1\ge \left\Vert \frac{\mathbf{y}-\mathbf{u}}{a(t)} -\mathbf{y} \right\Vert.
    \end{align*}
By the mean value theorem for multivariate function, we have
    \begin{align*}
    &    \max_{\Vert \mathbf{v} \Vert\leq r_t(\mathbf{y})}
    \left\vert p_{Y_0}\left(\mathbf{y}+\mathbf{v}\right) - p_{Y_0}(\mathbf{y})  \right\vert
     \nonumber \\
    &= \max_{\Vert \mathbf{v} \Vert\leq r_t(\mathbf{y})} (2\pi)^{-d/2}\left\vert \exp\left(-\frac{1}{2}\left\Vert \mathbf{y}+\mathbf{v} \right\Vert^2\right) - \exp\left(-\frac{1}{2}\Vert \mathbf{y} \Vert^2\right) \right\vert \\
        &\leq (2\pi)^{-d/2}  \max_{\mathbf{u}: \Vert \mathbf{u}-\mathbf{y}\Vert\le r_t(\mathbf{y})} \Vert u\Vert \exp\left(-\frac{1}{2}\Vert \mathbf{u} \Vert^2\right)  \\
        &\leq (2\pi)^{-d/2} r_t(\mathbf{y}) (\Vert \mathbf{y}\Vert + r_t(\mathbf{y}))\exp\left(-\frac{1}{2}\max(0,\Vert \mathbf{y}\Vert-r_t(\mathbf{y}))^2\right),
    \end{align*}
    where $l_t(\mathbf{y}):= (2\pi)^{-d/2} (\Vert \mathbf{y}\Vert + r_t(\mathbf{y}))\exp\left(-\frac{1}{2}\max(0,\Vert \mathbf{y}\Vert-r_t(\mathbf{y}))^2\right)$ in the last row is the locally Lipschitz constant of function $p_{Y_0} $. For notation simplification, we introduce $\alpha(t):= \vert 1-a(t)\vert/{a(t)}$. Integrating the results above, for $t\in[0,\epsilon)$ we have
    \begin{align*}
   &     \max_\mathbf{y} \max_{\Vert \mathbf{v} \Vert\leq r_t(\mathbf{y})} \vert p_{Y_0}(\mathbf{y}+\mathbf{v}) - p_{Y_0}(\mathbf{y})  \vert \leq \max_\mathbf{y} r_t(\mathbf{y})l_t(\mathbf{y}) \\
        &\leq (2\pi)^{-d/2}\max_{\Vert \mathbf{y}\Vert} \left(\alpha(t)\Vert \mathbf{y}\Vert+\frac{\vert b(t)\vert}{a(t)}B_1 \right) \left((\alpha(t)+1)\Vert \mathbf{y}\Vert+\frac{\vert b(t)\vert}{a(t)}B_1 \right)\\
        & \quad \quad \quad \cdot \exp\left(-\frac{1}{2}\max(0,(1-\alpha(t))\Vert \mathbf{y} \Vert-\frac{\vert b(t)\vert}{a(t)}B_1 )^2\right)\\
        & \leq (2\pi)^{-d/2} \max_{\Vert \mathbf{y}\Vert} (\alpha(t)+1)\left(\alpha(t)\Vert \mathbf{y}\Vert+\frac{\vert b(t)\vert}{a(t)}B_1 \right) \left(\Vert \mathbf{y}\Vert+\frac{\vert b(t)\vert B_1}{a(t)(\alpha(t)+1)} \right)\\
        & \quad \quad \quad \cdot \exp\left(-\frac{1}{2}\max(0,(1-\alpha(t))\Vert \mathbf{y}\Vert-\frac{\vert b(t)\vert}{a(t)}B_1 )^2\right)\\
        &\leq (2\pi)^{-d/2} \max_{z \ge \frac{\vert b(t)\vert B_1}{a(t)(\alpha(t)+1)}} (\alpha(t)+1) \left(\alpha(t)z + \frac{\vert b(t)\vert B_1}{a(t)(\alpha(t)+1)}\right)z \\
          & \quad \quad \quad \cdot \exp\left(-\frac{1}{2}\max(0,(1-\alpha(t))z-\frac{2\vert b(t)\vert B_1}{a(t)(\alpha(t)+1)} )^2\right)\\
        &\leq (2\pi)^{-d/2}\left[ (\alpha(t)+1)\alpha(t)\max_{z\geq 0}z^2 \exp(-f(z)) + \frac{\vert b(t)\vert B_1}{a(t)}\max_{z\geq 0}  z \exp(-f(z))\right] \\
        &\leq c_1 (\alpha(t)+1)\alpha(t) + c_2 \frac{\vert b(t)\vert}{a(t)} = o(\gamma(t)),
    \end{align*}
    where $f(z) := \max(0,(1-\alpha(t))z-\frac{2\vert b(t)\vert B_1}{a(t)(\alpha(t)+1)} )^2/2$, $c_1:=\max_{z\geq 0}z^2 \exp(-f(z))<\infty$ and $c_1:=\max_{z\geq 0}z^2 \exp(-f(z))<\infty$. Note that $\alpha(t) = o(\gamma(t))$ and $\vert b(t) \vert = o(\gamma(t))$ as $t$ tends towards $0$. This completes the proof.
\end{proof}

\subsection{Proof of Lemma \ref{lemma1}}
\bigskip
\noindent
\newline

\noindent
Lemma \ref{lemma1}. \textit{
Suppose that Assumptions \ref{assumption0}, \ref{assumption0b} and \ref{assump_bound} hold. Given random sample  $S = \{D_k = (\mathbf{y}_{0,k}, (\mathbf{x}_k, \mathbf{y}_{1,k}), \boldsymbol{\eta}_k)\}_{k=1}^{n}$ and two specific classes of functions $\mathcal{F}_n$, $\mathcal{F}_n^\prime$, we have
    \begin{align*}
        \mathbb{E}_S \{ R^b_t(\hat{\boldsymbol{b}}_n)-R^b_t(\boldsymbol{b}^*)\}  \leq& \mathbb{E}_S\{R^b_t(\hat{\boldsymbol{b}}_n) - 2R^b_{t,n}(\hat{\boldsymbol{b}}_n) + R^b_t(\boldsymbol{b}^*)\} \\
        &+ 2 d \inf_{\boldsymbol{f} \in \mathcal{F}_n }  \mathbb{E} \Vert \boldsymbol{f}(X,Y_t,t) - \boldsymbol{b}^*(X,Y_t,t) \Vert^2, \\
        \mathbb{E}_S \{ R^s_t(\hat{\boldsymbol{s}}_n)-R^s_t(\boldsymbol{s}^*)\}
        \leq& \gamma(t)^{-1}  \mathbb{E}_S\{R^\kappa_t(\hat{\boldsymbol{\kappa}}_n) - 2R^\kappa_{t,n}(\hat{\boldsymbol{\kappa}}_n) + R^\kappa_t(\boldsymbol{\kappa}^*)\} \\
        &+ 2 d \gamma(t)^{-1} \cdot \inf_{\boldsymbol{f} \in \mathcal{F}_n^\prime} \mathbb{E} \Vert \boldsymbol{f}(X,Y_t,t) - \boldsymbol{\kappa}^*(X,Y_t,t) \Vert^2,
    \end{align*} for $t \in(0,1).$
}

\begin{proof}
    Firstly, by the independence between $\mathbf{\eta}$ and $(X,Y_0,Y_1)$, for any $\boldsymbol{b}$ it is easy to check
    \begin{align} \label{R_b1}
     R^b_t(\boldsymbol{b})- R^b_t(\boldsymbol{b}^*) = &\mathbb{E} \left\| \partial_t \mathcal{I}(Y_0, Y_1, t) + \dot{\gamma}(t) {\boldsymbol{\eta}} - \boldsymbol{b}(X,Y_t,t) \right\|^2  \nonumber \\
    &- \mathbb{E} \left\| \partial_t \mathcal{I}(Y_0, Y_1, t) + \dot{\gamma}(t) {\boldsymbol{\eta}}  - \boldsymbol{b}^*(X,Y_t,t) \right\|^2 \nonumber \\
    &=  \mathbb{E} \Vert \boldsymbol{b}(X,Y_t,t) - \boldsymbol{b}^*(X,Y_t,t) \Vert^2
    \end{align}
    Recall that $\hat{\boldsymbol{b}}_n(\cdot, t)$ is empirical risk minimizer for $t \in [0,1]$ based on the effective sample $S$. Then for any $\boldsymbol{b} \in \mathcal{F}_n$ we have
    \begin{align} \label{lemma1_proof}
      \mathbb{E}_S \{ R^b_t(\hat{\boldsymbol{b}}_n)- &R^b_t(\boldsymbol{b}^*)\}
       \leq  \mathbb{E}_S \{R^b_t(\hat{\boldsymbol{b}}_n)-R^b_t(\boldsymbol{b}^*) \} + 2 \mathbb{E}_S \{R^b_{t,n}(\boldsymbol{b}) - R^b_{t,n}(\hat{\boldsymbol{b}}_n)\} \nonumber \\
        =& \mathbb{E}_S \{R^b_t(\hat{\boldsymbol{b}}_n)-R^b_t(\boldsymbol{b}^*)\} + 2 \mathbb{E}_S\{R^b_{t,n}(\boldsymbol{b}) - R^b_{t,n}(\boldsymbol{b}^*) + R^b_{t,n}(\boldsymbol{b}^*) - R^b_{t,n}(\hat{\boldsymbol{b}}_n)\} \nonumber \\
        =& \mathbb{E}_S\{R^b_t(\hat{\boldsymbol{b}}_n) - R^b_t(\boldsymbol{b}^*) + 2R^b_{t,n}(\boldsymbol{b}^*) - 2R^b_{t,n}(\hat{\boldsymbol{b}}_n)\} + 2 \mathbb{E}\{R^b_{t,n}(\boldsymbol{b}) - R^b_{t,n}(\boldsymbol{b}^*)\} \nonumber\\
        =& \mathbb{E}_S\{R^b_t(\hat{\boldsymbol{b}}_n) - 2R^b_{t,n}(\hat{\boldsymbol{b}}_n) + R^b_t(\boldsymbol{b}^*)\} + 2 \mathbb{E}\{ R^b_{t,n}(\boldsymbol{b}) - R^b_{t,n}(\boldsymbol{b}^*)\}, \nonumber\\
        =& \mathbb{E}_S\{R^b_t(\hat{\boldsymbol{b}}_n) - 2R^b_{t,n}(\hat{\boldsymbol{b}}_n) + R^b_t(\boldsymbol{b}^*)\} + 2 \mathbb{E} \Vert \boldsymbol{b}(X,Y_t,t) - \boldsymbol{b}^*(X,Y_t,t) \Vert^2,
    \end{align}
    where the first inequality follows from $R^b_{t,n}(\hat{\boldsymbol{b}}_n) \leq R^b_{t,n}(\boldsymbol{b})$ based on the definition of empirical risk minimizer, and the last row follows from (\ref{R_b1}). 

     Since (\ref{lemma1_proof}) holds for any $\boldsymbol{b} \in \mathcal{F}_n$, we then have
    \begin{align*}
     &   \mathbb{E}_S \{ R^b_t(\hat{\boldsymbol{b}}_n)-R^b_t(\boldsymbol{b}^*)\}\\
       &\leq \mathbb{E}_S\{R^b_t(\hat{\boldsymbol{b}}_n) - 2R^b_{t,n}(\hat{\boldsymbol{b}}_n) + R^b_t(\boldsymbol{b}^*)\} + 2 d \inf_{\boldsymbol{b} \in \mathcal{F}_n } \mathbb{E} \Vert \boldsymbol{b}(X,Y_t,t) - \boldsymbol{b}^*(X,Y_t,t) \Vert^2.
    \end{align*}
    Similarly to (\ref{lemma1_proof}), we can obtain
    \begin{align*}
    &    \mathbb{E}_S \{ R^s_t(\hat{\boldsymbol{s}}_n)-R^s_t(\boldsymbol{s}^*)\}\\
     &\leq \mathbb{E}_S\{R^s_t(\hat{\boldsymbol{s}}_n) - 2R^s_{t,n}(\hat{\boldsymbol{s}}_n) + R^s_t(\boldsymbol{s}^*)\} + 2 \mathbb{E} \Vert \boldsymbol{s}(X,Y_t,t) - \boldsymbol{s}^*(X,Y_t,t) \Vert^2,
      \end{align*} where $\boldsymbol{s}(\mathbf{x},\mathbf{y},t) = \gamma(t)^{-1} \boldsymbol{\kappa}(\mathbf{x},\mathbf{y},t)$ and  $\boldsymbol{\kappa}(\mathbf{x},\mathbf{y},t) \in \mathcal{F}_n^\prime$. Recall that $\hat{\boldsymbol{s}}_n(\mathbf{x},\mathbf{y},t) = \gamma(t)^{-1}\hat{\boldsymbol{\kappa}}_n(\mathbf{x},\mathbf{y},t)$ and ${\boldsymbol{s}}^*(\mathbf{x},\mathbf{y},t) = \gamma(t)^{-1}{\boldsymbol{\kappa}}^*(\mathbf{x},\mathbf{y},t)$. We have,
      \begin{align*}
        \mathbb{E}_S\{R^s_t(\hat{\boldsymbol{s}}_n) - 2R^s_{t,n}(\hat{\boldsymbol{s}}_n) + R^s_t(\boldsymbol{s}^*)\} &= \gamma(t)^{-1} \mathbb{E}_S\{R^\kappa_t(\hat{\boldsymbol{\kappa}}_n) - 2R^\kappa_{t,n}(\hat{\boldsymbol{\kappa}}_n) + R^\kappa_t(\boldsymbol{\kappa}^*)\} \\
        \mathbb{E} \Vert \boldsymbol{s}(X,Y_t,t) - \gamma(t)^{-1}  \boldsymbol{s}^*(X,Y_t,t) \Vert^2 &= \gamma(t)^{-1}  \mathbb{E} \Vert \boldsymbol{\kappa}(X,Y_t,t) - \boldsymbol{\kappa}^*(X,Y_t,t) \Vert^2.
      \end{align*}
It follows that
    \begin{align*}
        \mathbb{E}_S \{ R^s_t(\hat{\boldsymbol{s}}_n)-R^s_t(\boldsymbol{s}^*)\}
   \leq& \gamma(t)^{-1}  \mathbb{E}_S\{R^\kappa_t(\hat{\boldsymbol{\kappa}}_n) - 2R^\kappa_{t,n}(\hat{\boldsymbol{\kappa}}_n) + R^\kappa_t(\boldsymbol{\kappa}^*)\} \\
   &+ 2 d \gamma(t)^{-1} \inf_{\boldsymbol{\kappa} \in \mathcal{F}_n^\prime} \mathbb{E} \Vert \boldsymbol{\kappa}(X,Y_t,t) - \boldsymbol{\kappa}^*(X,Y_t,t) \Vert^2.
    \end{align*}
This completes the proof.
\end{proof}

\subsection{Proof of Lemma \ref{stoc_b}}
Before proving this lemma, we first introduce a lemma to handle some simple calculations.

\begin{lemma} \label{lemma_gauss}
    If a random variable $z \in \mathbb{R}^d$ follows a normal distribution $\mathcal{N}(\mu,\sigma^2)$, then we have
    \begin{align*}
        \mathbb{E}[\exp(\vert z \vert)] = e^{\mu+\frac{\sigma^2}{2}} \Phi\left[\frac{\mu+\sigma^2}{\sigma}\right] + e^{-\mu+\frac{\sigma^2}{2}} \Phi\left[\frac{-\mu+\sigma^2}{\sigma}\right].
    \end{align*}
\end{lemma}

\begin{proof}
    We introduce a log-normal random variable $y = \exp(z)$. Then, the conditional expectations of $y$ with respect to a threshold $\lambda$ are:
    \begin{align}
        E[y \mid y \geqslant \lambda]=e^{\mu+\frac{\sigma^2}{2}} \cdot \frac{\Phi\left[\frac{\mu+\sigma^2-\ln (\lambda)}{\sigma}\right]}{1-\Phi\left[\frac{\ln (\lambda)-\mu}{\sigma}\right]} \label{cal}
    \end{align}where $\Phi$ is the cumulative distribution function of the standard normal distribution (i.e., $\mathcal{N}(0,1)$).

    Then, we have
    \begin{align*}
        \mathbb{E}[\exp(\vert z \vert)] &= P(z \geqslant 0) \mathbb{E}[\exp(z) | z \geqslant 0] +  P(z < 0) \mathbb{E}[\exp(-z) | z<0] \\
        &= \Phi(\frac{\mu}{\sigma}) \mathbb{E}[y| y \geqslant 1] +  \Phi(-\frac{\mu}{\sigma}) \mathbb{E}[\exp(-z) | -z > 0] \\
        &= \Phi(\frac{\mu}{\sigma}) \mathbb{E}[y| y \geqslant 1] +  \Phi(-\frac{\mu}{\sigma}) \mathbb{E}[y^\prime | y^\prime > 1],
    \end{align*}
where $y^\prime = \exp(-z)$ and $-z \sim \mathcal{N}(-\mu, \sigma^2)$. Combined with Equation \ref{cal} and setting $\lambda$ to 1, the expectation can be derived as
    \begin{align*}
        \mathbb{E}[\exp(\vert z \vert)] &= \Phi(\frac{\mu}{\sigma}) e^{\mu+\frac{\sigma^2}{2}} \cdot \frac{\Phi\left[\frac{\mu+\sigma^2}{\sigma}\right]}{1-\Phi\left[\frac{-\mu}{\sigma}\right]} + \Phi(-\frac{\mu}{\sigma}) e^{-\mu+\frac{\sigma^2}{2}} \cdot \frac{\Phi\left[\frac{-\mu+\sigma^2}{\sigma}\right]}{1-\Phi\left[\frac{\mu}{\sigma}\right]} \\
        &= e^{\mu+\frac{\sigma^2}{2}} \Phi\left[\frac{\mu+\sigma^2}{\sigma}\right] + e^{-\mu+\frac{\sigma^2}{2}} \Phi\left[\frac{-\mu+\sigma^2}{\sigma}\right],
    \end{align*} which completes the proof.

\end{proof}
 We now prove Lemma \ref{stoc_b}.

\noindent
\noindent
\newline
\noindent
Lemma \ref{stoc_b}. \textit{
Let $\boldsymbol{b}^*$ and $\boldsymbol{s}^*$ be the target functions defined in (\ref{bstar}) and (\ref{sstar})  under the conditional stochastic interpolation model. Suppose that Assumptions \ref{assumption0}, \ref{assumption0b} and \ref{assump_bound} hold. Then for any $t \in (0,1)$, the ERMs $\hat{\boldsymbol{b}}_n$ and $\hat{\boldsymbol{\kappa}}_n$ defined in (\ref{bhat}) and (\ref{shat}) satisfy
\begin{align*}
        \mathbb{E}_S\{R^b_t(\hat{\boldsymbol{b}}_n ) - 2R^b_{t,n}(\hat{\boldsymbol{b}}_n ) + R^b_t(\boldsymbol{b}^* )\}
    \leq c_0 \mathcal{B}_b(t)^5 \log n^3 \frac{\mathcal{S D} \log (\mathcal{S})}{n}
\end{align*}
    and
\begin{align*}
        \mathbb{E}_S\{R^\kappa_t(\hat{\boldsymbol{\kappa}}_n ) - 2R^\kappa_{t,n}(\hat{\boldsymbol{\kappa}}_n ) + R^\kappa_t(\boldsymbol{\kappa}^* )\}
		\leq c_0^\prime \mathcal{B}_\kappa(t)^5 \log n^3 \frac{\mathcal{S^\prime D^\prime} \log (\mathcal{S^\prime})}{n^\prime}.
\end{align*}
for $n \geq \max_{i = 1, ..., d} \{\operatorname{Pdim}( \tilde{\mathcal{F}}_{ni}), \operatorname{Pdim}( \tilde{\mathcal{F}}_{ni}^\prime ) \} / 2$, where $c_0, c_0^\prime>0$ are constants independent of $d, n, \mathcal{D}, \mathcal{W}, \mathcal{S}, \mathcal{D}^\prime, \mathcal{W}^\prime$ , and $\mathcal{S}^\prime$.
}
\begin{proof}
	We present the proof in two parts. We derive stochastic error bounds for $\hat{\boldsymbol{b}}_n$ in Part (I) and $\hat{\boldsymbol{\kappa}}_n$ in Part (II).
	
{\noindent \bf Part (I)}: For any $t \in (0,1)$,  the stochastic error of the estimator $\hat{\boldsymbol{b}}_n$ is given by:

\begin{align*}
\mathbb{E}_S\left\{R^b_t(\hat{\boldsymbol{b}}_n) - 2R^b_{t,n}(\hat{\boldsymbol{b}}_n) + R^b_t(\boldsymbol{b})\right\}.
\end{align*}

Given the sample $S = \{D_k = (\mathbf{y}_{0,k}, (\mathbf{x}_k, \mathbf{y}_{1,k}), {\boldsymbol{\eta}}_k)\}_{k=1}^{n}$. Defining
$$\mathcal{A}_{t}^{(i)}(D_k) := \partial_t {\mathcal{I}(\mathbf{y}_{0,k}, \mathbf{y}_{1,k},t )}^{(i)} +\dot{\gamma}(t){\boldsymbol{\eta}}_k^{(i)},$$
we consider coordinate-wise scalar expressions of the risks
\begin{align*}
    R_t^{b,(i)}(\boldsymbol{b}) &:= \mathbb{E}_{S} \vert \mathcal{A}_{t}^{(i)}(D_k) - \boldsymbol{b}^{(i)}(\mathbf{x}_k,\mathbf{y}_{t,k},t)  \vert^2 \\
    R_{t,n}^{b,(i)}(\boldsymbol{b}) &:= \frac{1}{n} \sum_{k=1}^n \left\vert \mathcal{A}_{t}^{(i)}(D_k) -\boldsymbol{b}^{(i)}(\mathbf{x}_k,\mathbf{y}_{t,k},t)  \right\vert^2,
\end{align*}
where $\mathbf{v}^{(i)}$ denotes the $i$th component of a vector $\mathbf{v}$.
It is evident that $R^b_t(\boldsymbol{b}) = \sum_{i = 1}^d R_t^{b,(i)}(\boldsymbol{b})$ and $R^b_{t,n}(\boldsymbol{b}) = \sum_{i = 1}^d R_{t,n}^{b,(i)}(\boldsymbol{b})$. Let $S^\prime := \{D_k^\prime=(\mathbf{y}^\prime_{0,k}, \mathbf{y}^\prime_{1,k}, t^\prime_k)\}_{k=1}^n$ be an independent ghost sample of $S$. We then have
\begin{align} \label{bound0}
    \mathbb{E}_S\{R^b_t(\hat{\boldsymbol{b}}_n )& - 2R^b_{t,n}(\hat{\boldsymbol{b}}_n ) + R^b_t(\boldsymbol{b}^* )\}
    = \mathbb{E}_S\left\{R^b_t(\hat{\boldsymbol{b}}_n ) - R^b_t(\boldsymbol{b}^* ) - 2\{R^b_{t,n}(\hat{\boldsymbol{b}}_n ) - R^b_{t,n}(\boldsymbol{b}^* )\}\right\}  \nonumber\\
    =& \sum_{i = 1}^d \mathbb{E}_S\left\{R_t^{b,(i)}(\hat{\boldsymbol{b}}_n ) - R_t^{b,(i)}(\boldsymbol{b}^* ) + 2\{R_{t,n}^{b,(i)}(\boldsymbol{b}^* ) - R^{b,(i)}_{t,n}(\hat{\boldsymbol{b}}_n )\}\right\} \nonumber\\
    =& \sum_{i = 1}^d \mathbb{E}_S \left\{ \Bigg[
    \mathbb{E}_{S^\prime} \left\vert \mathcal{A}_{t}^{(i)}(D_k^\prime)  - \hat{\boldsymbol{b}}_n^{(i)}(Y_{t}^\prime, t)\right\vert^2 -\mathbb{E}_{S^\prime} \left\vert \mathcal{A}_{t}^{(i)}(D_k^\prime)  - \boldsymbol{b}^{*(i)}(Y_{t}^\prime, t)\right\vert^2 \Bigg]\right.\nonumber \\
    &\left. -\frac{2}{n}\sum_{k=1}^n \Bigg[ \left\vert \mathcal{A}_{t}^{(i)}(D_k) - \hat{\boldsymbol{b}}_n^{(i)}(\mathbf{x}_k,\mathbf{y}_{t,k},t)   \right\vert^2  - \left\vert \mathcal{A}_{t}^{(i)}(D_k) - {\boldsymbol{b}}^{*(i)}(\mathbf{x}_k,\mathbf{y}_{t,k},t)   \right\vert^2  \Bigg] \right\} \nonumber \\
    := &\sum_{i = 1}^d \mathbb{E}_S \frac{1}{n} \sum_{k=1}^n \left[
    \mathbb{E}_{S^\prime} [\boldsymbol{g}^{(i)}(\hat{\boldsymbol{b}}_n , D_k^\prime)]- 2 \boldsymbol{g}^{(i)}(\hat{\boldsymbol{b}}_n ,D_k)\right] \nonumber \\
    := &\sum_{i = 1}^d \mathbb{E}_S \frac{1}{n} \sum_{k=1}^n \boldsymbol{G}^{(i)}(\hat{\boldsymbol{b}}_n ,D_k),
\end{align}
where $Y_t^\prime = \mathcal{I}(Y_0^\prime, Y_1^\prime, t) + \gamma(t){\boldsymbol{\eta}}^\prime$, $\boldsymbol{g}^{(i)}(\boldsymbol{b},D_k) := \vert \mathcal{A}_{t}^{(i)}(D_k) - {\boldsymbol{b}}^{(i)}(\mathbf{x}_k,\mathbf{y}_{t,k},t)   \vert^2- \vert \mathcal{A}_{t}^{(i)}(D_k) - \boldsymbol{b}^{*(i)}(\mathbf{x}_k,\mathbf{y}_{t,k},t)  \vert^2$, $\boldsymbol{G}^{(i)}(\hat{\boldsymbol{b}}_n ,D_k) := \mathbb{E}_{S^\prime} [\boldsymbol{g}^{(i)}(\hat{\boldsymbol{b}}_n , D_k^\prime)]- 2 \boldsymbol{g}^{(i)}(\hat{\boldsymbol{b}}_n ,D_k)$ for $i=1,\ldots,d$ and $k=1,\ldots,n$.

Given time $t$, let $\beta_n(t) := \mathcal{B}_b(t) \sqrt{\log n}$ be a positive number and $\mathcal{T}_{\beta_n(t)}: \mathbb{R}^d \rightarrow \mathbb{R}$ be the truncation function with threshold $\beta_n(t)$, i.e., for any $y \in \mathbb{R}, \mathcal{T}_{\beta_n(t)}(y)=y$ if $\vert y\vert \leq \beta_n(t)$ and $\mathcal{T}_{\beta_n(t)} (y)=\beta_n(t) \cdot \operatorname{sign}(y)$ otherwise. We define $\boldsymbol{b}^{*(i)}_{\beta_n(t)}(\mathbf{x},\mathbf{y},t) := \mathbb{E}[\mathcal{T}_{\beta_n(t)}(\partial_t \mathcal{I}^{(i)}(Y_0,Y_1,t) + \dot{\gamma}(t)\eta^{(i)}) \mid X = \mathbf{x}, Y_t=\mathbf{y}]$, $\boldsymbol{g}^{(i)}_{\beta_n(t)} (\boldsymbol{b},D_k) := \vert \mathcal{T}_{\beta_n(t)}(\mathcal{A}_{t}^{(i)}(D_k)) - {\boldsymbol{b}}^{(i)}(\mathbf{x}_k,\mathbf{y}_{t,k},t)   \vert^2- \vert \mathcal{T}_{\beta_n(t)}(\mathcal{A}_{t}^{(i)}(D_k)) - \boldsymbol{b}^{*(i)}_{\beta_n(t)}(\mathbf{x}_k,\mathbf{y}_{t,k},t)  \vert^2$ for $i=1,\ldots,d$ and $k=1,\ldots,n$. Then for any $\boldsymbol{b} \in \mathcal{F}_n$, we have
\begin{align*}
  &  |\boldsymbol{g}^{(i)}(\boldsymbol{b}, D_k)- \boldsymbol{g}^{(i)}_{\beta_n(t)}(\boldsymbol{b}, D_k)|\\
  &  = \vert 2 \{\boldsymbol{b}^{(i)}(\mathbf{x}_k, \mathbf{y}_{t,k}, t)-\boldsymbol{b}^{*(i)}(\mathbf{x}_k, \mathbf{y}_{t,k}, t)\} (\mathcal{T}_{\beta_n(t)}(\mathcal{A}_t^{(i)}(D_k))-\mathcal{A}_t^{(i)}(D_k)) \\
  &\ \ \ +(\boldsymbol{b}^{*(i)}_{\beta_n(t)}(\mathbf{x}_k, \mathbf{y}_{t,k},t)-\mathcal{T}_{\beta_n(t)} \mathcal{A}_t^{(i)}(D_k))^2-
  (\boldsymbol{b}^{*(i)}(\mathbf{x}_k,\mathbf{y}_{t,k},t)-\mathcal{T}_{\beta_n(t)}(\mathcal{A}_t^{(i)}(D_k)) )^2 \vert \\
 &   \leq |2\{\boldsymbol{b}^{(i)}(\mathbf{x}_k,\mathbf{y}_{t,k},t)-\boldsymbol{b}^{*(i)}(\mathbf{x}_k,\mathbf{y}_{t,k},t)\}(\mathcal{T}_{\beta_n(t)}(\mathcal{A}_t^{(i)}(D_k))-\mathcal{A}_t^{(i)}(D_k))| \\
    &\ \ \ + |(\boldsymbol{b}^{*(i)}_{\beta_n(t)}(\mathbf{x}_k,\mathbf{y}_{t,k},t)-\mathcal{T}_{\beta_n(t)}(\mathcal{A}_t^{(i)}(D_k)))^2-(\boldsymbol{b}^{*(i)}(\mathbf{x}_k,\mathbf{y}_{t,k},t)-\mathcal{T}_{\beta_n(t)}(\mathcal{A}^{(i)}_t(D_k)))^2| \\
 &   \leq  4\mathcal{B}_b(t)|\mathcal{T}_{\beta_n(t)}(\mathcal{A}_t^{(i)}(D_k))-\mathcal{A}_t^{(i)}(D_k)| +|\boldsymbol{b}^{*(i)}_{\beta_n(t)}(\mathbf{x}_k,\mathbf{y}_{t,k},t)-\boldsymbol{b}^{*(i)}(\mathbf{x}_k,\mathbf{y}_{t,k},t)| \\
    &\ \ \ \times  |\boldsymbol{b}^{*(i)}_{\beta_n(t)}(\mathbf{x}_k,\mathbf{y}_{t,k},t) +\boldsymbol{b}^{*(i)}(\mathbf{x}_k,\mathbf{y}_{t,k},t)-2 \mathcal{T}_{\beta_n(t)}(\mathcal{A}_t^{(i)}(D_k))| \\
 &   \leq  4\mathcal{B}_b(t)|\mathcal{A}_t^{(i)}(D_k)| \chi(|\mathcal{A}_t^{(i)}(D_k)|>\beta_n(t))+(3 \beta_n(t) + \mathcal{B}_b(t)) |\boldsymbol{b}^{*(i)}_{\beta_n(t)}(\mathbf{x}_k,\mathbf{y}_{t,k},t)-\boldsymbol{b}^{*(i)}(\mathbf{x}_k,\mathbf{y}_{t,k},t)| \\
 &   \leq 4\mathcal{B}_b(t)|\mathcal{A}_t^{(i)}(D_k)| \chi(|\mathcal{A}_t^{(i)}(D_k)|>\beta_n(t)) \\
    & \ \ \ +(3 \beta_n(t) + \mathcal{B}_b(t)) \mathbb{E}[|\mathcal{T}_{\beta_n(t)}(\mathcal{A}_t^{(i)}(D_k)) - \mathcal{A}_t^{(i)}(D_k) | X = \mathbf{x}_k, Y_t = \mathbf{y}_{t,k} ] \\
&    \leq 4\mathcal{B}_b(t)|\mathcal{A}_t^{(i)}(D_k)| \chi(|\mathcal{A}_t^{(i)}(D_k)|>\beta_n(t)) \\
    &\ \ \ +(3 \beta_n(t) + \mathcal{B}_b(t)) \mathbb{E}[|\mathcal{A}_t^{(i)}(D_k)| \chi(|\mathcal{A}_t^{(i)}(D_k)|>\beta_n(t)) | X = \mathbf{x}_k, Y_t = \mathbf{y}_{t,k} ] \\
&    \leq  4\mathcal{B}_b(t) |\mathcal{A}_t^{(i)}(D_k)| \chi(|\mathcal{A}_t^{(i)}(D_k)|>\beta_n(t))+(3 \beta_n(t) + \mathcal{B}_b(t))|\mathcal{A}_t^{(i)}(D_k)| \chi(|\mathcal{A}_t^{(i)}(D_k)|>\beta_n(t)),
\end{align*} where $\chi$ is the indicator function. Then, we obtain
\begin{align*}
   \mathbb{E}_S \vert \boldsymbol{g}^{(i)}(\boldsymbol{b}, D_k) - \boldsymbol{g}_{\beta_n(t)}^{(i)}(\boldsymbol{b}, D_k) \vert \leq  (3 \beta_n(t)+5\mathcal{B}_b(t))\mathbb{E}_S[|\mathcal{A}_t^{(i)}(D_k)| \chi(|\mathcal{A}_t^{(i)}(D_k)|>\beta_n(t))].  
\end{align*}

Recall that random variable $\partial_t \mathcal{I}(Y_0,Y_1,t)$ satisfies $\mathcal{B}_{I^\prime}(t)$-light tail condition.  Using Lemma \ref{lemma_subgauss} and \ref{lemma_gauss}, we find that $\mathcal{A}_t^{(i)}(D_k)$  satisfies $|\dot{\gamma}(t)|+ \mathcal{B}_{I^\prime}(t)$-light tail condition. Recall that $\mathcal{B}_b(t) > |\dot{\gamma}(t)|+\mathcal{B}_{I^\prime}(t)$, then we have
\begin{align*}
 &   \mathbb{E}_S[|\mathcal{A}_t^{(i)}(D_k)| \chi(|\mathcal{A}_t^{(i)}(D_k)|>\beta_n(t))] \\
 &= \int_{\beta_n(t)}^\infty x P(|\mathcal{A}_t^{(i)}(D_k)|>x)  dx \\
    &\leq 4 \int_{\beta_n(t)}^\infty x  \exp\left( -\frac{x^2}{2((\mathcal{B}_{I^\prime}(t))+|\dot{\gamma}(t)|)^2}\right)  dx \\
    &= 4 (\mathcal{B}_{I^\prime}(t)+|\dot{\gamma}(t)|)^2 \exp\left(-\frac{\beta_n(t)^2}{(\mathcal{B}_{I^\prime}(t)+|\dot{\gamma}(t)|)^2}\right) \leq  \frac{4 \mathcal{B}_b(t)^2}{n} .
\end{align*}

We introduce $\boldsymbol{G}_{\beta_n(t)}^{(i)} (\boldsymbol{b},D_k) :=\mathbb{E}_{S^\prime} [\boldsymbol{g}^{(i)}_{\beta_n(t)}(\hat{\boldsymbol{b}}_n , D_k^\prime)]- 2 \boldsymbol{g}^{(i)}_{\beta_n(t)}(\hat{\boldsymbol{b}}_n ,D_k)$ for $i=1,\ldots,d$ and $k=1,\ldots,n$, then it is easy to obtain
\begin{align} \label{bound00}
    \mathbb{E}_S\left[ \frac{1}{n} \sum_{k=1}^n \boldsymbol{G}^{(i)}(\hat{\boldsymbol{b}}_n ,D_k) \right]
    &\leq \mathbb{E}_S\left[ \frac{1}{n} \sum_{k=1}^n \boldsymbol{G}^{(i)}_{\beta_n(t)}(\hat{\boldsymbol{b}}_n ,D_k) \right]
    + \frac{32\mathcal{B}_b(t)^3 \sqrt{\log n} }{n}
\end{align}

Next, we derive upper bounds for $\frac{1}{n} \sum_{k=1}^n \boldsymbol{G}^{(i)}_{\beta_n(t)}(\hat{\boldsymbol{b}}_n ,D_k)$ for $i= 1,2, \ldots, d$. Recall that $$\hat{\boldsymbol{b}}^{(i)}_n \in \mathcal{F}_{ni} = \{\boldsymbol{b}^{(i)}: \mathbb{R}^{k+d+1} \to \mathbb{R} \mid  \boldsymbol{b} \in \mathcal{F}_{n} \}$$
and
$$\tilde{\mathcal{F}}_{ni}:= \{\boldsymbol{b}^{(i)}: \mathbb{R}^{k+d+1} \rightarrow \mathbb{R}\mid \boldsymbol{b}^{(i)}{\rm\ has\ the\ same\ architecture\ as\ that\ of\ }\boldsymbol{b}^{(i)}\in\tilde{\mathcal{F}}_{ni} \}.$$
 Then for any $z \geq 0$, we have
\begin{align*}
    P&\left\{  \frac{1}{n} \sum_{k=1}^n \boldsymbol{G}^{(i)}_{\beta_n(t)}(\hat{\boldsymbol{b}}_n ,D_k) \geq z \right\} \\
    =& P\left\{  \mathbb{E}_{S^\prime} [\boldsymbol{g}^{(i)}_{\beta_n(t)}(\hat{\boldsymbol{b}}_n , D_k^\prime)] -\frac{1}{n} \sum_{k=1}^n \boldsymbol{g}^{(i)}_{\beta_n(t)}(\hat{\boldsymbol{b}}_n ,D_k)
     \geq \frac{z}{2} + \frac{1}{2} \mathbb{E}_{S^\prime} [\boldsymbol{g}^{(i)}_{\beta_n(t)}(\hat{\boldsymbol{b}}_n ,D_k^\prime)] \right\}\\
     \leq& P\left\{ \exists \boldsymbol{b}^{(i)} \in \tilde{\mathcal{F}}_{ni}:  \mathbb{E}_{S^\prime} [\boldsymbol{g}^{(i)}_{\beta_n(t)}(\hat{\boldsymbol{b}}_n ,D_k^\prime)] -\frac{1}{n} \sum_{k=1}^n \boldsymbol{g}^{(i)}_{\beta_n(t)}(\hat{\boldsymbol{b}}_n ,D_k)
     \geq \frac{1}{2} \left( \frac{z}{2} + \frac{z}{2} +  \mathbb{E}_{S^\prime} [\boldsymbol{g}^{(i)}_{\beta_n(t)}(\hat{\boldsymbol{b}}_n , D_k^\prime)] \right) \right\}.
\end{align*}

Note that for any $\boldsymbol{b}$ with $\boldsymbol{b}^{(i)} \in \tilde{\mathcal{F}}_{ni}$, we have $\Vert \boldsymbol{b} \Vert \leq \mathcal{B}_b(t) \leq \beta_n(t)$. Additionally, $| \mathcal{T}_{\beta_n(t)}(\mathcal{A}_t^{(i)}(D_k))| \leq \beta_n(t),\|\boldsymbol{g}_{\beta_n(t)}\|_{\infty} \leq \beta_n(t)$. Then, by Theorem 11.4 of \citet{gyorfi2002distribution},  
we have
\begin{align*}
    P & \left\{ \exists \boldsymbol{b}^{(i)} \in \tilde{\mathcal{F}}_{ni}:  \mathbb{E}_{S^\prime} [\boldsymbol{g}^{(i)}(\hat{\boldsymbol{b}}_n ,D_k^\prime)] -\frac{1}{n} \sum_{k=1}^n \boldsymbol{g}^{(i)}(\hat{\boldsymbol{b}}_n ,D_k)
    \geq \frac{1}{2} \left( \frac{z}{2} + \frac{z}{2} + \mathbb{E}_{S^\prime} [\boldsymbol{g}^{(i)}(\hat{\boldsymbol{b}}_n ,D_k^\prime)] \right) \right\} \\
     &\qquad\qquad\qquad\qquad\leq 14 \mathcal{N}_{n}\left(\frac{z}{80\beta_n(t)}, \Vert \cdot \Vert_\infty, \tilde{\mathcal{F}}_{ni}\right) \exp{\left(- \frac{nz}{5136 \beta_n(t)^4}\right)},
\end{align*}
where the covering number $\mathcal{N}_n $ can be found in Appendix (\ref{app_defi}). Then for $a_{n,t} \geq 0$, we have
\begin{align} \label{bound1}
 &   \mathbb{E}_S  \left\{
        \frac{1}{n} \sum_{k=1}^n \boldsymbol{G}^{(i)}_{\beta_n(t)}(\hat{\boldsymbol{b}}_n ,D_k)
    \right\}\nonumber \\
     &\leq a_{n,t} +  {\displaystyle \int_{a_{n,t}}^{\infty} P\left\{ \frac{1}{n} \sum_{k=1}^n \boldsymbol{G}^{(i)}_{\beta_n(t)}(\hat{\boldsymbol{b}}_n ,D_k)  > z \right\}} \, dz \nonumber\\
    &\leq a_{n,t} +  {\displaystyle \int_{a_{n,t}}^{\infty} 14\mathcal{N}_{n}\left(\frac{z}{80\beta_n(t)}, \Vert \cdot \Vert_\infty, \tilde{\mathcal{F}}_{ni}\right) \exp{\left(- \frac{nz}{5136 \beta_n(t)^4}\right)} } \, dz \nonumber\\
    &\leq a_{n,t} +  {\displaystyle \int_{a_{n,t}}^{\infty} 14\mathcal{N}_{n}\left(\frac{a_{n,t}}{80\beta_n(t)}, \Vert \cdot \Vert_\infty, \tilde{\mathcal{F}}_{ni}\right) \exp{\left(- \frac{nz}{5136 \beta_n(t)^4}\right)} } \, dz \nonumber \\
   & = a_{n,t} + 14\mathcal{N}_{n}\left(\frac{a_{n,t}}{80\beta_n(t)}, \Vert \cdot \Vert_\infty, \tilde{\mathcal{F}}_{ni}\right) \exp{\left(- \frac{a_{n,t} n}{5136 \beta_n(t)^4}\right)}\frac{5136\beta_n(t)^4}{n}.
\end{align}
Choosing $a_{n,t} = \log(14\mathcal{N}_{n}(\frac{1}{n}, \Vert \cdot \Vert_\infty, \tilde{\mathcal{F}}_{ni})) \cdot  5136 \beta_n(t)^4/n$, we can verify that $a_{n,t}/(80\beta_n(t)) \geq 1/n$ and $\mathcal{N}_{n}(\frac{a_{n,t}}{80\beta_n(t)}, \Vert \cdot \Vert_\infty, \tilde{\mathcal{F}}_{ni})\le \mathcal{N}_{n}(\frac{1}{n}, \Vert \cdot \Vert_\infty, \tilde{\mathcal{F}}_{ni})$. Then we have
\begin{align} \label{bound2}
    \mathbb{E}_S  \left\{
        \frac{1}{n} \sum_{k=1}^n \boldsymbol{G}^{(i)}_{\beta_n(t)}(\hat{\boldsymbol{b}}_n ,D_k)
    \right\} &\leq \frac{5136\beta_n(t)^4(\log(14\mathcal{N}_{n}(\frac{1}{n}, \Vert \cdot \Vert_\infty, \tilde{\mathcal{F}}_{ni}))+1)}{n} \nonumber \\
    &=  \frac{5136\mathcal{B}_b(t)^4 {\log n}^2  (\log(14\mathcal{N}_{n}(\frac{1}{n}, \Vert \cdot \Vert_\infty, \tilde{\mathcal{F}}_{ni}))+1)}{n}
\end{align}


According to Theorem 12.2 in \citep{anthony1999neural}, the covering number of $\tilde{\mathcal{F}}_{ni}$ can be further linked to its Pseudo dimension $\operatorname{Pdim}(\tilde{\mathcal{F}}_{ni})$ by
\begin{align}\label{cover-pseudo}
\mathcal{N}_{n}\left(\frac{1}{n},\Vert\cdot\Vert_{\infty}, \tilde{\mathcal{F}}_{ni}\right) \leq\left(\frac{4 e \mathcal{B}_b(t) n^2}{\operatorname{Pdim}\left( \tilde{\mathcal{F}}_{ni} \right)}\right)^{\operatorname{Pdim}\left(\tilde{\mathcal{F}}_{ni} \right)}
\end{align}
provided that $2 n \geq \operatorname{Pdim}(\tilde{\mathcal{F}}_{ni})$, where the definition of Pseudo-dimension can be found in Appendix (\ref{app_defi}). Furthermore, based on Theorem 3 and 6 in \citep{bartlett2019nearly}, there exist universal constants $c$ and $C$ such that
\begin{align} \label{pseudo-para}
    c \cdot \mathcal{S D} \log (\mathcal{S} / \mathcal{D}) \leq \operatorname{Pdim}\left(\tilde{\mathcal{F}}_{ni}\right) \leq C \cdot \mathcal{S D} \log (\mathcal{S})
\end{align}
Finally, combining (\ref{bound0}), (\ref{bound2}), (\ref{cover-pseudo}) and (\ref{pseudo-para}), for each $i=1,\ldots,d$ we obtain an upper bound expressed in the network parameters
\begin{align} \label{bound3}
    \mathbb{E}_S  \left\{
        \frac{1}{n} \sum_{k=1}^n \boldsymbol{G}^{(i)}_{\beta_n(t)}(\hat{\boldsymbol{b}}_n ,D_k)
    \right\} \leq c_0 \mathcal{B}_b(t)^5 \log n^3 \frac{\mathcal{S D} \log (\mathcal{S})}{n}
\end{align}
for some constant $c_0>0$ not depending on $n, d, \mathcal{S}$ or $\mathcal{D}$. Based on (\ref{bound0}), (\ref{bound00}) and (\ref{bound3}), we finally get
\begin{align*}
    \mathbb{E}_S\{R^b_t(\hat{\boldsymbol{b}}_n ) - 2R^b_{t,n}(\hat{\boldsymbol{b}}_n ) + R^b_t(\boldsymbol{b}^* )\}
    \leq c_0 d \mathcal{B}_b(t)^5 \log n^3 \frac{\mathcal{S D} \log (\mathcal{S})}{n}.
\end{align*}

{\noindent \bf Part (II)}: To tackle the stochastic error of $\hat{\boldsymbol{\kappa}}_n $ ( or  $\hat{\boldsymbol{s}}_n $), we define
\begin{align*}
    \boldsymbol{h}^{(i)}(\boldsymbol{\kappa},D_k) &:= \vert {\boldsymbol{\eta}}_k^{(i)} + \hat{\boldsymbol{\kappa}}_n ^{(i)}(\mathbf{x}_k,\mathbf{y}_{t,k},t)  \vert^2- \vert  {\boldsymbol{\eta}}_k^{(i)} + \boldsymbol{\kappa}^{*(i)}(\mathbf{x}_k,\mathbf{y}_{t,k},t)  \vert^2, \\
    \boldsymbol{H}^{(i)}(\boldsymbol{\kappa} ,D_k) &:= \mathbb{E}_{S^\prime} [\boldsymbol{h}^{(i)}(\boldsymbol{\kappa} , D_k^\prime)]- 2 \boldsymbol{h}^{(i)}(\boldsymbol{\kappa} ,D_k),
\end{align*} which is similar to the definitions of $\boldsymbol{h}^{(i)}$ and $\boldsymbol{H}^{(i)}$ in Part (I). Let $\beta_n^\prime(t) = \mathcal{B}_\kappa(t) \sqrt{\log n}$, we define
\begin{align*}
&\boldsymbol{\kappa}^{*(i)}_{\beta_n^\prime(t)}(\mathbf{x},\mathbf{y},t) := \mathbb{E}[\mathcal{T}_{\beta_n^\prime(t)}(\boldsymbol{\eta}^{(i)}) \mid X = \mathbf{x}, Y_t=\mathbf{y}],\\
&\boldsymbol{h}^{(i)}_{\beta_n^\prime(t)} (\boldsymbol{\kappa},D_k) := \vert \mathcal{T}_{\beta_n^\prime(t)}(\boldsymbol{\eta}^{(i)}) - {\boldsymbol{\kappa}}^{(i)}(\mathbf{x}_k,\mathbf{y}_{t,k},t)   \vert^2- \vert \mathcal{T}_{\beta_n^\prime(t)}(\boldsymbol{\eta}^{(i)}) - \boldsymbol{\kappa}^{*(i)}_{\beta_n^\prime(t)}(\mathbf{x}_k,\mathbf{y}_{t,k},t)  \vert^2,\\
&\boldsymbol{H}_{\beta_n^\prime(t)}^{(i)} (\boldsymbol{\kappa},D_k) :=\mathbb{E}_{S^\prime} [\boldsymbol{h}^{(i)}_{\beta_n^\prime(t)}(\hat{\boldsymbol{\kappa}}_n , D_k^\prime)]- 2 \boldsymbol{h}^{(i)}_{\beta_n^\prime(t)}(\hat{\boldsymbol{\kappa}}_n ,D_k)
\end{align*}
for $i=1,\ldots,d$ and $k=1,\ldots,n$.
Similarly, for any $\boldsymbol{\kappa} \in \mathcal{F}_n^\prime$, we have
\begin{align*}
    |\boldsymbol{h}^{(i)}(\boldsymbol{\kappa}, D_k)- \boldsymbol{h}^{(i)}_{\beta_n^\prime(t)}(\boldsymbol{\kappa}, D_k)|
    \leq  4 \mathcal{B}_\kappa(t) | \boldsymbol{\eta}^{(i)}_k | \chi(|\boldsymbol{\eta}^{(i)}_k|>\beta_n^\prime(t))+4 \beta_n^\prime(t)|\boldsymbol{\eta}^{(i)}_k| \chi(|\boldsymbol{\eta}^{(i)}_k|>\beta_n).
\end{align*} Then, we have
\begin{align*}
 &   \mathbb{E}_S |\boldsymbol{h}^{(i)}(\boldsymbol{\kappa}, D_k)- \boldsymbol{h}^{(i)}_{\beta_n^\prime(t)}(\boldsymbol{\kappa}, D_k)| \\
 &\leq  4 \mathcal{B}_\kappa(t) \mathbb{E}_S[| \boldsymbol{\eta}^{(i)}_k | \chi(|\boldsymbol{\eta}^{(i)}_k|>\beta_n^\prime(t))] +4 \beta_n^\prime(t) \mathbb{E}_S[| \boldsymbol{\eta}^{(i)}_k | \chi(|\boldsymbol{\eta}^{(i)}_k|>\beta_n^\prime(t))].
\end{align*}
Recall that $\boldsymbol{\eta}^{(i)}_k \sim \mathcal{N}(0, 1)$.  Using Lemma \ref{lemma_subgauss} and \ref{lemma_gauss}, we have
\begin{align*}
    \mathbb{E}_S[| \boldsymbol{\eta}^{(i)}_k | \chi(|\boldsymbol{\eta}^{(i)}_k|>\beta_n^\prime(t))] &= \int_{\beta_n^\prime(t)}^\infty x P(|\boldsymbol{\eta}^{(i)}_k|>x)  dx \\
    &\leq 4 \int_{\beta_n^\prime(t)}^\infty x  \exp\left( -\frac{x^2}{2}\right)  dx = 4 \exp(-\beta_n^\prime(t)^2) \leq \frac{4}{n} .
\end{align*}
Then, it is easy to obtain
\begin{align} 
    \mathbb{E}_S\left[ \frac{1}{n} \sum_{k=1}^n \boldsymbol{H}^{(i)}(\hat{\boldsymbol{\kappa}}_n ,D_k) \right] \leq \mathbb{E}_S\left[ \frac{1}{n} \sum_{k=1}^n \boldsymbol{H}^{(i)}_{\beta_n^\prime(t)}(\hat{\boldsymbol{\kappa}}_n ,D_k) \right]+ \frac{32\mathcal{B}_\kappa(t) \sqrt{\log n}}{n}
\end{align} 
By Theorem 11.4 of \citep{gyorfi2002distribution}, for any $u\ge0$
	\begin{align*}
		P&\left\{ \frac{1}{n} \sum_{k=1}^n \boldsymbol{H}^{(i)}_{\beta_n^\prime(t)}(\hat{\boldsymbol{\kappa}}_n ,D_k)  \geq u \right\} \\
		=& P\left\{ \mathbb{E}_{S^\prime} [\boldsymbol{h}^{(i)}_{\beta_n^\prime(t)}(\hat{\boldsymbol{\kappa}}_n ,D_k^\prime)] -\frac{1}{n} \sum_{k=1}^n \boldsymbol{h}^{(i)}_{\beta_n^\prime(t)}(\hat{\boldsymbol{\kappa}}_n ,D_k)
		\geq \frac{u}{2} + \frac{1}{2} \mathbb{E}_{S^\prime} [\boldsymbol{h}^{(i)}_{\beta_n^\prime(t)}(\hat{\boldsymbol{\kappa}}_n ,D_k^\prime)] \right\}\\
		\leq& P \left\{ \exists \boldsymbol{\kappa}^{(i)} \in \tilde{\mathcal{F}}_{ni}^\prime:  \mathbb{E}_{S^\prime} [\boldsymbol{h}^{(i)}_{\beta_n^\prime(t)}(\hat{\boldsymbol{\kappa}}_n ,D_k^\prime)] -\frac{1}{n} \sum_{k=1}^n \boldsymbol{h}^{(i)}_{\beta_n^\prime(t)}(\hat{\boldsymbol{s}}_n ,D_k)
		\geq \frac{1}{2} \left( \frac{u}{2} + \frac{u}{2} + \mathbb{E}_{S^\prime} [\boldsymbol{h}^{(i)}_{\beta_n^\prime(t)}(\hat{\boldsymbol{\kappa}}_n ,D_k^\prime)] \right) \right\} \\
        \leq& 14 \mathcal{N}_{n}\left(\frac{z}{80\beta_n^\prime(t)}, \Vert \cdot \Vert_\infty, \tilde{\mathcal{F}}_{ni}^\prime\right) \exp{\left(- \frac{nz}{5136 \beta_n^{\prime 4}(t)}\right)},
	\end{align*}
	Similar to the proof in Part (I),  we then have
	\begin{align*}
		\mathbb{E}_S\{R^\kappa_t(\hat{\boldsymbol{\kappa}}_n ) - 2R^\kappa_{t,n}(\hat{\boldsymbol{\kappa}}_n ) + R^\kappa_t(\boldsymbol{\kappa}^* )\}
		\leq c_0^\prime d \mathcal{B}_\kappa(t)^5 \log n^3 \frac{\mathcal{S}^\prime \mathcal{D}^\prime \log (\mathcal{S}^\prime)}{n}.
	\end{align*} where $c_0^\prime>0$ is a constant not depending on $d,n,\mathcal{S}^\prime,\mathcal{D}^\prime$ and $t$.

This completes the proof.
\end{proof}

\subsection{Proof of Lemma \ref{theorem2}}

Before proving this lemma, we first introduce a lemma to handle some simple calculations.

\begin{lemma} \label{lemma_subgauss}
    If there exist $\sigma_i \in \mathbb{R}^+$  for $i = 1, ..., n$ such that independent random variables $z_i \in \mathbb{R}$ satisfying
    \begin{align*}
        \mathbb{E}[\exp(\lambda z_i)] \leq 2\exp\left(\frac{1}{2} \sigma_i^2 \lambda^2\right), \text{ for all } \lambda \in \mathbb{R}, i = 1, ..., n,
    \end{align*}
    we have
    \begin{align*}
        \mathbb{E}\left[\exp \left(\lambda \sum_{i=1}^n z_i\right)\right] \leq 2\exp \left(\frac{\lambda^2}{2}\left(\sum_{i=1}^n \sigma_i\right)^2\right)
    \end{align*} and
$$
P\left(\left\vert \sum_{i=1}^n z_i \right\vert \geq \lambda  \left(\sum_{i=1}^n \sigma_i\right) \right) \leq  4 C \exp (- \lambda^ 2/2),
$$
where $C>0$ depends only on $\sigma_i$, $i=1,   ..., n$.
\end{lemma}

\begin{proof} By H\"{o}lder's inequality, we have
\begin{align*}
& \mathbb{E}\left[\exp \left(\lambda \sum_{i=1}^n z_i\right)\right]=\mathbb{E}\left[\prod_{i=1}^n \exp \left(\lambda z_i\right)\right] \leq \prod_{i=1}^n\left(\mathbb{E}\left[\exp \left(\lambda X_i p_i\right)\right]\right)^{1 / p_i} \\
& \leq \prod_{i=1}^n\left(2 \exp \left(\lambda^2 \sigma_i^2 p_i^2 / 2\right)\right)^{1 / p_i} = 2 \prod_{i=1}^n \exp \left(\lambda^2 \sigma_i^2 p_i / 2\right) \\
& =2 \exp \left(\frac{\lambda^2}{2} \sum_{i=1}^n \sigma_i^2 p_i\right)
\end{align*}

Let $p_i=\frac{\sum_{j=1}^n \sigma_j}{\sigma_i}$, which satisfy the required conditions: $p_i \in(1, \infty], \sum_{i=1}^n \frac{1}{p_i}=1$. We compute

$$
\sum_{i=1}^n \sigma_i^2 p_i=\sum_{i=1}^n \sigma_i^2 \frac{\sum_{j=1}^n \sigma_j}{\sigma_i}=\left(\sum_{i=1}^n \sigma_i\right)^2,
$$

with which we then conclude

\begin{align} \label{equa_subgauss}
    \mathbb{E}\left[\exp \left(\lambda \sum_{i=1}^n z_i\right)\right] \leq 2 \exp \left(\frac{\lambda^2}{2}\left(\sum_{i=1}^n \sigma_i\right)^2\right).
\end{align}




According to the equivalent characterizations of subgaussianity \citep{vershynin2018high, wainwright2019high}, we have
$$
P\left(\left\vert \sum_{i=1}^n z_i \right\vert \geq \lambda  \left(\sum_{i=1}^n \sigma_i\right) \right) \leq  4 C \exp (- \lambda^ 2/2)
$$
for a constant $C>0$ depending only on $\sigma_i$, $i=1,   ..., n$.
\end{proof}

We now prove  Lemma \ref{theorem2}.
\noindent
\newline

\noindent
Lemma \ref{theorem2}. \textit{
Suppose that Assumptions \ref{assumption0} and \ref{assumption0b} are satisfied and there exists a constant $B_X$ such that $\Vert X \Vert_\infty \leq B_X$. For any $A_n \geq \max\{B_I, B_X, 1\}$, $t \in (0,1)$, assume that the drift function $\boldsymbol{b}^{*(i)}(\cdot,t)$ and denoising function $\boldsymbol{\kappa}^{*(i)}(\cdot,t)$ belong to the locally H\"{o}lder smooth class $\mathcal{H}^\beta_{\text{loc}}\left(\mathbb{R}^{k+d}, \mathcal{B}_b(t)\right)$ and $\mathcal{H}^{\beta^\prime}_{\text{loc}}\left(\mathbb{R}^{k+d}, \mathcal{B}_\kappa(t)\right)$ with $\beta>0$, $\beta^\prime>0$ for $i = 1,\ldots, d$.
}
     \renewcommand{\theenumi}{\roman{enumi}}
     \begin{enumerate}
         \item For any $U, V  \in \mathbb{N}^{+}$,  let $\mathcal{F}_{n}=\mathcal{F}_{\mathcal{D}, \mathcal{W}, \mathcal{U}, \mathcal{S}, \mathcal{B}_b(t)}$ be a class of neural networks with width $\mathcal{W}= 38(\lfloor\beta\rfloor+1)^2 3^{(k+d)} (k+d)^{\lfloor\beta\rfloor+2} (3+\left\lceil \log _2 U\right\rceil)U$, and depth $\mathcal{D}=21(\lfloor\beta\rfloor+1)^2 (3+\left\lceil \log _2 V\right\rceil)V+2 (k+d)$,  then the approximation error
         \begin{align*}
     &       \inf_{\boldsymbol{f} \in \mathcal{F}_n }  \mathbb{E} \Vert \boldsymbol{f}(X,Y_t,t) - \boldsymbol{b}^*(X,Y_t,t) \Vert^2\\
      & \leq \inf_{\boldsymbol{f} \in \mathcal{F}_n, \atop \Vert \mathbf{y} \Vert_\infty \leq A_n } \Vert \boldsymbol{f}(\cdot,t) - \boldsymbol{b}^*(\cdot,t)\Vert_\infty^2 + 4 \mathcal{B}_b(t) \exp \left(- \frac{A_n^2}{2 (\mathcal{B}_I(t) + \gamma(t))^2}\right)
        \end{align*}
where
\begin{align*}
            \mathop{\inf}_{
                \boldsymbol{f} \in \mathcal{F}_n,
        \atop
        \Vert \mathbf{y} \Vert_\infty \leq A_n} \Vert \boldsymbol{f}(\cdot,t) - \boldsymbol{b}^*(\cdot,t) \Vert_\infty \leq  19 \mathcal{B}_b(t)
        C(\beta) (2A_n)^\beta(U V)^{\frac{-2 \beta}{k+d}},
\end{align*} with $C(\beta) =  (\lfloor\beta\rfloor+1)^2 (k+d+1)^{\lfloor\beta\rfloor+(\beta \vee 1) / 2},
\text{ and } \beta \vee 1 = \max\{\beta, 1\}.$

        \item For any $U^\prime, V^\prime  \in \mathbb{N}^{+}$, let $\mathcal{F}_{n}^\prime=\mathcal{F}^\prime_{\mathcal{D}^\prime, \mathcal{W}^\prime, \mathcal{U}^\prime, \mathcal{S}^\prime, \mathcal{B}_\kappa(t)}$ be a class of neural networks with width $\mathcal{W}^\prime= 38(\lfloor\beta^\prime\rfloor+1)^2 3^{(k+d)} (k+d)^{\lfloor\beta^\prime\rfloor+2} (3+\left\lceil \log _2 U^\prime\right\rceil)U^\prime$, and depth $\mathcal{D}^\prime=21(\lfloor\beta^\prime\rfloor+1)^2 (3+\left\lceil \log _2 V^\prime\right\rceil)V^\prime+2 (k+d)$,  then the approximation error
        \begin{align*}
  &          \mathop{\inf}_{
                \boldsymbol{f} \in \mathcal{F}_n^\prime}  \mathbb{E} \Vert \boldsymbol{f}(X,Y_t,t) - \boldsymbol{\kappa}^*(X,Y_t,t) \Vert^2\\
                 &\leq   \inf_{\boldsymbol{f} \in \mathcal{F}_n^\prime, \atop \Vert \mathbf{y} \Vert_\infty \leq A_n  } \Vert \boldsymbol{f}(\cdot,t) - \boldsymbol{\kappa}^*(\cdot,t) \Vert_\infty^2 + 4 \mathcal{B}_\kappa(t) \exp \left(- \frac{A_n^2}{2 (\mathcal{B}_I(t) + \gamma(t))^2}\right).
        \end{align*}
        where
         \begin{align*}
            \mathop{\inf}_{
                \boldsymbol{f} \in \mathcal{F}_n^\prime,
        \atop
        \Vert \mathbf{y} \Vert_\infty \leq A_n} \Vert \boldsymbol{f}(\cdot,t) - \boldsymbol{\kappa}^*(\cdot,t) \Vert_\infty \leq  19 \mathcal{B}_\kappa(t)
    C(\beta^{\prime})
        (2A_n)^{\beta^\prime}(U^\prime V^\prime)^{\frac{-2 \beta^\prime}{k+d}}.
        \end{align*}
     \end{enumerate}

\begin{proof}[Proof of Lemma \ref{theorem2}]
    First, for a given $A_n$, we have
    \begin{align}\label{equ_lemma53_0}
   &      \mathbb{E} [\Vert \boldsymbol{b}(X,Y_t,t) - \boldsymbol{b}^*(X,Y_t,t) \Vert^2 ] \\
         &\leq \mathbb{E} [\Vert \boldsymbol{b}(X,Y_t,t) - \boldsymbol{b}^*(X,Y_t,t) \Vert^2 \mid \Vert Y_t \Vert_\infty > A_n] P(\Vert Y_t \Vert_\infty  > A_n) \nonumber \\
        & + \mathbb{E} [\Vert \boldsymbol{b}(X,Y_t,t) - \boldsymbol{b}^*(X,Y_t,t) \Vert^2 \mid \Vert Y_t \Vert_\infty  \leq A_n] P(\Vert Y_t \Vert_\infty  \leq A_n) \nonumber \\
       & \leq 2 \mathcal{B}_b(t) P(\Vert Y_t \Vert_\infty  > A_n) + \mathbb{E} [\Vert \boldsymbol{b}(X,Y_t,t) - \boldsymbol{b}^*(X,Y_t,t) \Vert^2 \mid \Vert Y_t \Vert_\infty  \leq A_n].
    \end{align}
    Recall that $\mathcal{I}(Y_0,Y_1,t)^{(i)}$ is $\mathcal{B}_I(t)$-subgaussian and we find that $\mathbb{E}[exp(\lambda^\prime \gamma(t) \boldsymbol{\eta}^{(i)})] = \exp(\gamma(t)^2 \lambda^{\prime 2} /2) $ for $i=1, ..., d$, any $\lambda^\prime \in \mathbb{R}$. Let $\lambda^\prime = An / (\mathcal{B}_I(t) + \gamma(t))$. Using Lemma \ref{lemma_subgauss}, we obtain
    \begin{align*}
        P(\vert Y_t^{(i)} \vert  > A_n) = P(\vert \mathcal{I}(Y_0,Y_1,t)^{(i)} + \gamma(t) \boldsymbol{\eta}^{(i)} \vert  > A_n) \leq 2 \exp \left(- \frac{A_n^2}{2 (\mathcal{B}_I(t) + \gamma(t))^2}\right),
    \end{align*} for $i=1, ..., d$. Then, we have
    \begin{align*}
        P(\Vert Y_t \Vert_\infty  > A_n)  \leq d P(\vert Y_t^{(i)} \vert  > A_n) \leq 2d \exp \left(- \frac{A_n^2}{2 (\mathcal{B}_I(t) + \gamma(t))^2}\right).
    \end{align*} Combining it with (\ref{equ_lemma53_0}), we obtain
    \begin{align*}
        \mathbb{E} [\Vert \boldsymbol{b}(X,Y_t,t) - \boldsymbol{b}^*(X,Y_t,t) \Vert^2 ] \leq& 4d \mathcal{B}_b(t) \exp \left(- \frac{A_n^2}{2 (\mathcal{B}_I(t) + \gamma(t))^2}\right) \\
        &+ \mathbb{E} [\Vert \boldsymbol{b}(X,Y_t,t) - \boldsymbol{b}^*(X,Y_t,t) \Vert^2 \mid \Vert Y_t \Vert_\infty  \leq A_n].
    \end{align*} Taking the infimum on both sides of the inequality, we obtain:
    \begin{align*}
   &     \inf_{\boldsymbol{f} \in\mathcal{F}_n}\mathbb{E} [\Vert \boldsymbol{f}(X,Y_t,t) - \boldsymbol{b}^*(X,Y_t,t) \Vert^2 ]\\
         &\leq 4d \mathcal{B}_b(t) \exp \left(- \frac{A_n^2}{2 (\mathcal{B}_I(t) + \gamma(t))^2}\right) \\
        &+ \inf_{\boldsymbol{f} \in\mathcal{F}_n} \mathbb{E} [\Vert \boldsymbol{f}(X,Y_t,t) - \boldsymbol{b}^*(X,Y_t,t) \Vert^2 \mid \Vert Y_t \Vert_\infty  \leq A_n] \\
     &   \leq 4d \mathcal{B}_b(t) \exp \left(- \frac{A_n^2}{2 (\mathcal{B}_I(t) + \gamma(t))^2}\right) + \inf_{\boldsymbol{f} \in\mathcal{F}_n \atop \Vert \mathbf{y} \Vert_\infty \leq A_n} \Vert \boldsymbol{f}(\cdot,t) - \boldsymbol{b}^*(\cdot,t) \Vert^2.
    \end{align*}
    Similarly, we have
    \begin{align*}
   &     \inf_{\boldsymbol{f} \in\mathcal{F}^\prime_n}\mathbb{E} [\Vert \boldsymbol{f}(X,Y_t,t) - \boldsymbol{\kappa}^*(X,Y_t,t) \Vert^2 ] \\
   &\leq 4d \mathcal{B}_\kappa(t) \exp \left(- \frac{A_n^2}{2 (\mathcal{B}_I(t) + \gamma(t))^2}\right) + \inf_{\boldsymbol{f} \in\mathcal{F}^\prime_n \atop \Vert \mathbf{y} \Vert_\infty \leq A_n} \Vert \boldsymbol{f}(\cdot,t) - \boldsymbol{\kappa}^*(\cdot,t) \Vert^2.
    \end{align*}
    Then, we derive the upper bound for terms
    $$\inf_{\boldsymbol{f} \in\mathcal{F}_n \atop \Vert \mathbf{y} \Vert_\infty \leq A_n} \Vert \boldsymbol{f}(\cdot,t) - \boldsymbol{b}^*(\cdot,t) \Vert^2
     \text{ and } \inf_{\boldsymbol{f} \in\mathcal{F}^\prime_n \atop \Vert \mathbf{y} \Vert_\infty \leq A_n} \Vert \boldsymbol{f}(\cdot,t) - \boldsymbol{\kappa}^*(\cdot,t) \Vert^2.$$

    Recall that the target function of our estimation is a vector-valued function. We consider a Cartesian product of function classes with the same model complexity for different components of the output of vector-valued functions. Specifically, we define $\bar{\mathcal{F}}_{n} := \bar{\mathcal{F}}_{n1} \times ... \times  \bar{\mathcal{F}}_{nd}=\{\boldsymbol{f}:\mathbb{R}^{k+d+1}\to\mathbb{R}^{d}\mid \boldsymbol{f}=(f_1,\ldots,f_d), f_i\in\bar{\mathcal{F}}_{ni}, i=1,\ldots,d\}$ as neural networks with $d$-dimensional output constructed by paralleling one-dimensional output networks in $\bar{\mathcal{F}}_{ni}, i=1,\ldots,d$. Here $\bar{\mathcal{F}}_{ni} := \mathcal{F}_{\mathcal{D}, \bar{\mathcal{W}}, \bar{\mathcal{U}},\bar{\mathcal{S}}}$ is a class of neural networks with depth $\mathcal{D}$ width $\bar{\mathcal{W}}$, neurons $\bar{\mathcal{U}}$ and size $\bar{\mathcal{S}}$. By the definition of $\bar{\mathcal{F}}_{n}$, the neural networks in $\bar{\mathcal{F}}_{n}$ has depth $\mathcal{D}$, width $\mathcal{W} = d \bar{\mathcal{W}}$, neurons $d\mathcal{U}$ and size $d \mathcal{S}$. If we let $\mathcal{F}_n$ be a class of neural networks with output dimension $d$, depth $\mathcal{D}$, width $\mathcal{W} = d \bar{\mathcal{W}}$, neurons $d\mathcal{U}$ and size $d \mathcal{S}$, then
    \begin{align*}
        \inf_{\boldsymbol{f} \in \mathcal{F}_n} \Vert \boldsymbol{f}(\cdot,t) - \boldsymbol{b}^*(\cdot,t) \Vert_\infty
        \leq \inf_{ \boldsymbol{f} \in \bar{\mathcal{F}}_{n} } \Vert \boldsymbol{f}(\cdot,t) - \boldsymbol{b}^*(\cdot,t) \Vert_\infty, \text{ for } t \in (0,1).
    \end{align*}

Note that the components of neural networks $\bar{\mathcal{F}}_{n}$ are calculated by neural networks in different function classes $\bar{\mathcal{F}}_{ni}, i=1,2, ...,d$, which do not interact with each other. Then
    \begin{align*}
         \inf_{ \boldsymbol{f} \in \bar{\mathcal{F}}_{n}} \Vert \boldsymbol{f}(\cdot,t) - \boldsymbol{b}^*(\cdot,t)  \Vert_\infty = \max_{i=1,\ldots,d}\inf_{ \boldsymbol{f} \in \bar{\mathcal{F}}_{ni^{*} }} \vert f(\cdot,t) - \boldsymbol{b}^{*(i)}(\cdot,t) \vert, \text{ for } t \in (0,1),
    \end{align*}
    and the approximation analysis for vector-valued functions can be reformulated into an approximation analysis for one-dimensional real-valued functions.

    Recall the drift functions $\boldsymbol{b}^{*(i)}(\cdot,t)$ belongs to the local H\"{o}lder class $\mathcal{H}^\beta\left(\mathbb{R}^{k+d}, \mathcal{B}_b(t)\right)$ with $\beta>0$ for $ i = 1,\ldots, d$.
    By Lemma \ref{corollary_3}, for any $U,V\in\mathbb{N}^+$, let the ReLU networks in $\bar{\mathcal{F}}_{ni},i=1,\ldots,d$ has width $\bar{\mathcal{W}}=38(\lfloor\beta\rfloor+1)^2 3^{k+d} (k+d)^{\lfloor\beta\rfloor+1} U_{}\left\lceil\log _2(8 U)\right\rceil$ and a depth of $\mathcal{D}=21(\lfloor\beta\rfloor+1)^2 V\left\lceil\log _2(8 V)\right\rceil+2 (k+d)$.  Then, for $i=,1,\ldots,d$ we can obtain
    \begin{align*}
         \inf_{ f \in \bar{\mathcal{F}}_{ni^{*} }} \vert f - \boldsymbol{b}^{*(i^*)} \vert =& \inf_{ g \in \bar{\mathcal{F}}_{ni^{*} }} \vert g - \boldsymbol{b}^{*(i^*)} \vert \\
         \leq& 19 B_0(\lfloor\beta\rfloor+1)^2 (k+d)^{\lfloor\beta\rfloor+(\beta \vee 1) / 2}(2A_n)^\beta(U V)^{-2 \beta / (k+d)}.
    \end{align*}
    Similarly, for the approximation of $\boldsymbol{\kappa}^* $, we can obtain
    \begin{align*}
         \inf_{ \boldsymbol{f} \in \bar{\mathcal{F}}^\prime_{n }} \Vert \boldsymbol{f} - \boldsymbol{\kappa}^*  \Vert_\infty
         \leq 19 B_0^\prime(\lfloor\beta^\prime\rfloor+1)^2 (k+d)^{\lfloor\beta^\prime\rfloor+(\beta^\prime \vee 1) / 2}(2A_n)^{\beta^\prime}(U^\prime V^\prime)^{-2 \beta^\prime / (k+d)},
    \end{align*}
    where the notations of $B_0^\prime, \beta^\prime, U^\prime$ and $V^\prime$ are defined accordingly. This completes the proof.
\end{proof}

\subsection{Proof of Lemma \ref{theorem_bound}}

\noindent
\newline

\noindent
Lemma \ref{theorem_bound}.
\textit{Suppose the conditions in Lemmas \ref{stoc_b}, \ref{theorem2}  are satisfied.}
    \renewcommand{\theenumi}{\roman{enumi}}
    \begin{enumerate}
        \item For any $U, V  \in \mathbb{N}^{+}$, let $\mathcal{F}_{n}=\mathcal{F}_{\mathcal{D}, \mathcal{W}, \mathcal{U}, \mathcal{S}, \mathcal{B}_b(t)}$ be a class of neural networks with width $\mathcal{W}= 38(\lfloor\beta\rfloor+1)^2 3^{(k+d)} (k+d)^{\lfloor\beta\rfloor+2} (3+\left\lceil \log _2 U\right\rceil)U$, and depth $\mathcal{D}=21(\lfloor\beta\rfloor+1)^2 (3+\left\lceil \log _2 V\right\rceil)V+2 (k+d)$,  then  for  $n \geq \max_{i = 1, ..., d} \operatorname{Pdim}( \tilde{\mathcal{F}}_{ni})/2$ and any $t \in (0, 1)$, we have
        \begin{align*}
    & \mathbb{E}_S \Vert \hat{\boldsymbol{b}}_n(X,Y_t,t) - \boldsymbol{b}^*(X,Y_t,t) \Vert^2 \\ &\leq   722 d \mathcal{B}_b(t)^2(\lfloor\beta\rfloor+1)^4 (k+d)^{2\lfloor\beta\rfloor+(\beta \vee 1)}(2 A_n)^{2 \beta}(U V)^{\frac{-4 \beta}{k+d}} \\
    &\ \ \ + c_0 d \mathcal{B}_b(t)^5 \log n^3 \frac{\mathcal{S D} \log (\mathcal{S})}{n} + 8 d \mathcal{B}_b(t) \exp \left(- \frac{A_n^2}{2 (\mathcal{B}_{I}(t) + \gamma(t))^2}\right),
        \end{align*}
        where $c_0>0$ is a universal constant.
        \vspace*{0.2cm}

        \item For any $U^\prime, V^\prime  \in \mathbb{N}^{+}$, let $\mathcal{F}_{n}^\prime=\mathcal{F}^\prime_{\mathcal{D}^\prime, \mathcal{W}^\prime, \mathcal{U}^\prime, \mathcal{S}^\prime, \mathcal{B}_\kappa(t)}$ be a class of neural networks with width $\mathcal{W}^\prime= 38(\lfloor\beta^\prime\rfloor+1)^2 3^{(k+d)} (k+d)^{\lfloor\beta^\prime\rfloor+2} (3+\left\lceil \log _2 U^\prime\right\rceil)U^\prime$, and depth $\mathcal{D}^\prime=21(\lfloor\beta^\prime\rfloor+1)^2 (3+\left\lceil \log _2 V^\prime\right\rceil)V^\prime+2 (k+d)$,  then for  $n \geq \max_{i = 1, ..., d} \operatorname{Pdim}( \tilde{\mathcal{F}}_{ni})/2$, we have
        \begin{align*}
    & \mathbb{E}_S \Vert \hat{\boldsymbol{s}}_n(X,Y_t,t) - \boldsymbol{s}^*(X,Y_t,t) \Vert^2 \\
    &\leq  722 d \mathcal{B}_\kappa(t)^2 \gamma(t)^{-1} (\lfloor\beta^\prime \rfloor+1)^4(k+d)^{2\lfloor\beta^\prime \rfloor+(\beta^\prime \vee 1)}(2A_n)^{2\beta^\prime}(U^\prime V^\prime)^{\frac{-4 \beta^\prime}{k+d}}\\
        &\ \ \ + c_0^\prime \gamma(t)^{-1}d \mathcal{B}_\kappa(t)^5 \log n^3 \frac{\mathcal{S}^\prime \mathcal{D}^\prime \log (\mathcal{S}^\prime)}{n} + 8 d \gamma(t)^{-1} \mathcal{B}_\kappa(t) \exp \left(- \frac{A_n^2}{2 (\mathcal{B}_{I}(t) + \gamma(t))^2}\right)
        \end{align*}
    \end{enumerate}
  \textit{  and $c_0^\prime>0$ is a universal constant.}

\begin{proof}
    According to Lemma \ref{lemma1} and \ref{theorem2}, for any $t\in(0,1)$ we have
    \begin{align*}
 &   \mathbb{E}_S \Vert \hat{\boldsymbol{b}}_n (X,Y_t,t) - \boldsymbol{b}^* (X,Y_t,t) &\Vert^2\\
     &= \mathbb{E}_S \{R_t^b(\hat{\boldsymbol{b}}) - R_t^b(\boldsymbol{b}^*)\} \\
    &    \leq  \mathbb{E}_S\{R^b_t(\hat{\boldsymbol{b}}_n) - 2R^b_{t,n}(\hat{\boldsymbol{b}}_n) + R^b_t(\boldsymbol{b}^*)\} + 2 d \inf_{\boldsymbol{b} \in \mathcal{F}_n, \Vert \mathbf{y} \Vert_\infty \leq A_n } \Vert \boldsymbol{b} - \boldsymbol{b}^* \Vert_\infty^2  \\
        &\ \ \ + 8 d \mathcal{B}_b(t) \exp \left(- \frac{A_n^2}{2 (\mathcal{B}_{I}(t) + \gamma(t))^2}\right),
    \end{align*}
    Combining Lemmas \ref{stoc_b}, \ref{theorem2} completes the proof.
\end{proof}

\subsection{Proof of Theorem \ref{corollary_bound}}

\noindent
\newline

\noindent
Theorem  \ref{corollary_bound}.
\textit{Suppose the conditions in Lemmas \ref{stoc_b}, \ref{theorem2}  are satisfied.
\renewcommand{\theenumi}{(\roman{enumi})}
    \begin{enumerate}
        \item Let $U= n^{(k+d)/(8\beta + 4k+4d)}$ and $V=n^{(k+d)/(8\beta + 4k+4d)},$ for $t \in (0,1)$,
        we have
        \begin{align*}
  \mathbb{E}_S \Vert \hat{\boldsymbol{b}}_n(X,Y_t,t) - \boldsymbol{b}^*(X,Y_t,t) \Vert^2 = H_b(t) O \left(n^{\frac{-2\beta}{2\beta+k+d}} \log^{\max\{8,\beta\}} n\right),
        \end{align*}
        where $H_b(t) := \mathcal{B}_b(t)^5 (\mathcal{B}_I(t)+\gamma(t))^{2\beta}$.
        \item  Let $U^\prime=n^{(k+d)/(8\beta + 4k+4d)}$ and $V^\prime= n^{(k+d)/(8\beta + 4k+4d)},$ for $t \in (0,1)$,
        we have
        \begin{align*}
            \mathbb{E}_S \Vert \hat{\boldsymbol{s}}_n(X, Y_t,t) - \boldsymbol{s}^*(X,Y_t,t) \Vert^2=\gamma(t)^{-1} H_\kappa(t) O \left(n^{\frac{-2\beta^\prime}{2\beta^\prime+k+d}} \log^{\max\{8,\beta^\prime\}} n\right),
        \end{align*}
        where $H_\kappa(t) := \mathcal{B}_\kappa(t)^5 (\mathcal{B}_I(t)+\gamma(t))^{2\beta^\prime}$.
    \end{enumerate}}

\begin{proof}
    For any $U, V \in \mathbb{N}^{+}$, the class of neural networks in
    $\mathcal{F}_n=\mathcal{F}_{\mathcal{D}, \mathcal{W}, \mathcal{U}, \mathcal{S}, \mathcal{B}_b(t)}$ have width $\mathcal{W}= 38(\lfloor\beta\rfloor+1)^2 3^d d^{\lfloor\beta\rfloor+2} \left\lceil 3+ \log _2 U\right\rceil U$ and depth $\mathcal{D}=21(\lfloor\beta\rfloor+1)^2 \left\lceil 3+\log _2 V\right\rceil V+2 (k+d)$. For $n \geq \max_{i = 1, ..., d} \operatorname{Pdim}( \tilde{\mathcal{F}}_{ni}) / 2$ and $t\in (0,1)$, the excess risk of $\hat{\boldsymbol{b}}_n $ satisfies
    \begin{align*}
    & \mathbb{E}_S \Vert \hat{\boldsymbol{b}}_n (X,Y_t,t) - \boldsymbol{b}^* (X,Y_t,t) \Vert^2 \\
      &  \leq   722 d \mathcal{B}_b(t)^2(\lfloor\beta\rfloor+1)^4 (k+d)^{2\lfloor\beta\rfloor+(\beta \vee 1)}(2 A_n)^{2 \beta}(U V)^{\frac{-4 \beta}{k+d}} \\
     &\ \ \ + c_0 d \mathcal{B}_b(t)^5 \log n^3 \frac{\mathcal{S D} \log (\mathcal{S})}{n} + 8 d \mathcal{B}_b(t) \exp \left(- \frac{A_n^2}{2 (\mathcal{B}_{I}(t) + \gamma(t))^2}\right).
    \end{align*}
    Recall that for any multi-layer neural network in $\mathcal{F}_n$, its parameters naturally satisfy
    \begin{align*}
        \max \{\mathcal{W}, \mathcal{D}\} \leq \mathcal{S} \leq (\mathcal{D}-1)\mathcal{W}^2 + (\mathcal{D}+k+2d+1)\mathcal{W} + d \leq 2 \mathcal{D}\mathcal{W}^2.
    \end{align*}
    Then, by plugging $\mathcal{W} \leq 38(\lfloor\beta\rfloor+1)^2 3^{(k+d)} d^{\lfloor\beta\rfloor+2} (3+ \log _2 U )U$ and $\mathcal{D}\leq 21(\lfloor\beta\rfloor+1)^2 (3+ \log _2 V )V+2(k + d + 1)$ into the right-hand side, we obtain
        \begin{align*}
    	c \mathcal{S} \mathcal{D} \log(\mathcal{S})
    	\leq  c_1 U^2 V^2 \log^2 U \log^2 V(\log U + \log V).
    \end{align*}
    For given $t \in (0,1)$, we set $A_n = (\mathcal{B}_I(t)+\gamma(t))\sqrt{\frac{4\beta}{2\beta+k+d}\log n}$, then we obtain
    \begin{align*}
      &  \mathbb{E}_S \Vert \hat{\boldsymbol{b}}_n (X,Y_t,t) - \boldsymbol{b}^* (X,Y_t,t) \Vert^2 \leq  8d\mathcal{B}_b(t) n^{-\frac{2\beta}{2\beta+k+d}}+ \mathcal{B}_b(t)^5 (\mathcal{B}_I(t)+\gamma(t))^{2\beta} \\
      &\cdot \left\{ c_1 \log n^3  \frac{U^2 V^2 \log^2 U \log^2 V(\log U + \log V)}{n}  + c_2 \log^\beta n(U V)^{-4 \beta / (k+d)} \right\},
    \end{align*}
     for some constants $c_1>0$ and $c_2 := 722 \cdot \frac{16\beta}{2\beta+k+d}^\beta d (\lfloor\beta\rfloor+1)^4 (k+d)^{2\lfloor\beta\rfloor+(\beta \vee 1)}$.

     To achieve the optimal rate with respect to $n$, let us consider the right-hand side of the inequality as a function of $U$ and $V$ neglecting constants, denoted by $f(U,V) :=  \log n^3{U^2 V^2 \log^2 U \log^2 V(\log U + \log V)}/n + \log^\beta n(U V)^{-4 \beta / (k+d)}$.  Note that $f(U,V)$ is convex and the relative positions of $U$ and $V$ in the equation are symmetric. Consequently, the optimal solutions will include the case $U=V$, and we can focus on $f(z) := f(z,z) = z^4 \log^5 z \log^3 n /n + \log^\beta n z^{-8 \beta / (k+d)}$. By taking the derivative of $f(z)$, we can determine the minimal achieves at $z^*$ that satisfies
    \begin{align*}
        \frac{8\beta }{k+d} \frac{\log^\beta n n}{\log^3 n} = (z^*)^{\frac{8\beta}{k+d}+4} \log^4 z^* (4\log(z^*)+5),
    \end{align*}
    which does not have a closed form solution. We let $z^{\frac{8\beta}{k+d}+4} = n$ then
    \begin{align*}
        \min_{z} f(z) \leq f(n^\frac{k+d}{8\beta+4k+4d+4}) = \frac{k+d}{8\beta + 4(k+d)}n^{\frac{-2\beta}{2\beta+k+d}} \log^{8} n + \log^\beta n n^{\frac{-2\beta}{2\beta+k+d}}.
    \end{align*}
    We finally obtain
    \begin{align*}
        \mathbb{E}_S \Vert \hat{\boldsymbol{b}}_n (X,Y_t,t) - \boldsymbol{b}^* (X,Y_t,t) \Vert^2
        \leq& \mathcal{B}_b(t)^5 (\mathcal{B}_I(t)+\gamma(t))^{2\beta} O \left(n^{\frac{-2\beta}{2\beta+k+d}} \log^{\max\{8,\beta\}} n\right)
        \\&+ 8d\mathcal{B}_b(t) n^{-\frac{2\beta}{2\beta+k+d}} \\
        =& \mathcal{B}_b(t)^5 (\mathcal{B}_I(t)+\gamma(t))^{2\beta} O \left(n^{\frac{-2\beta}{2\beta+k+d}} \log^{\max\{8,\beta\}} n\right) .
    \end{align*}

    Regarding the convergence rate of $\hat{\boldsymbol{s}}_n$, the excess risk bound depends on the value of $t$,
    \begin{align*}
        \mathbb{E}_S \Vert \hat{\boldsymbol{s}}_n &(X,Y_t,t) - \boldsymbol{s}^*(X,Y_t,t) \Vert^2 \\
        \leq& \gamma(t)^{-1}\left\{
            722 d \mathcal{B}_\kappa(t)^2(\lfloor\beta^\prime \rfloor+1)^4(k+d)^{2\lfloor\beta^\prime \rfloor+(\beta^\prime \vee 1)}(2A_n)^{2\beta^\prime}(U^\prime V^\prime)^{\frac{-4 \beta^\prime}{k+d}} \right.\\
             & \left.+ c_0^\prime d \mathcal{B}_\kappa(t)^5 \log n^3 \frac{\mathcal{S}^\prime \mathcal{D}^\prime \log (\mathcal{S}^\prime)}{n} + 8 d \mathcal{B}_\kappa(t) \exp \left(- \frac{A_n^2}{2 (\mathcal{B}_{I}(t) + \gamma(t))^2}\right)
        \right\}.
    \end{align*}
    Similarly, we obtain
    \begin{align*}
  &      \mathbb{E}_S \Vert \hat{\boldsymbol{s}}_n (X,Y_t,t) - \boldsymbol{s}^* (X,Y_t,t) \Vert^2\\
    &    \leq \gamma(t)^{-1} \mathcal{B}_\kappa(t)^5 (\mathcal{B}_I(t)+\gamma(t))^{2\beta^\prime} O \left(n^{\frac{-2\beta^\prime}{2\beta^\prime+k+d}} \log^{\max\{8,\beta^\prime\}} n\right).
    \end{align*}

\end{proof}

\subsection{Proof of Corollary  \ref{corollary_linear}}

\noindent
\newline

\noindent
Corollary  \ref{corollary_linear}.
\textit{Suppose the conditions in Lemmas \ref{stoc_b}, \ref{theorem2}  are satisfied. Suppose CSI satisfies (\ref{interpb}),  $Y_0\sim \mathcal{N}(\mathbf{0},\mathbf{I}_d) $,
and $\gamma(t) \neq 0, t \in (0, 1)$.
Let $U^\prime=n^{(k+d)/(8\beta + 4k+4d)}$ and $V^\prime= n^{(k+d)/(8\beta + 4k+4d)}$. Define the conditional score estimator
\begin{align} \label{estimator_s_2}
   \tilde{\boldsymbol{s}}_n(\mathbf{x},\mathbf{y},t) := \frac{b(t)}{A(t)}\hat{\boldsymbol{b}}_n(\mathbf{x},\mathbf{y},t) -\frac{\dot{b}(t)}{A(t)}\mathbf{y}.
\end{align}
We have
\begin{align*}
    \mathbb{E}_S \Vert \tilde{\boldsymbol{s}}_n(X, Y_t,t) - \boldsymbol{s}^*(X,Y_t,t) \Vert^2=  \frac{H_b(t)b^2(t)}{A^2(t)} O \left(n^{\frac{-2\beta}{2\beta+k+d}} \log^{\max\{8,\beta\}}n \right),  \  t \in (0, 1).
\end{align*}
Suppose that $b(t)/A(t)$ is a bounded function in $(0, 1)$. Then
\begin{align*}
\sup_{t \in (0, 1)}     \mathbb{E}_S \Vert \tilde{\boldsymbol{s}}_n(X, Y_t,t) - \boldsymbol{s}^*(X,Y_t,t) \Vert^2= H_b(t)O \left(n^{\frac{-2\beta}{2\beta+k+d}} \log^{\max\{8,\beta\}}n \right).
\end{align*}}

\begin{proof}
   According to Corollary \ref{pro_score2},  the conditional score function is given by
    $$ \boldsymbol{s}^{*}(\mathbf{x},\mathbf{y},t):= \frac{b(t)}{A(t)} \boldsymbol{b}^*(\mathbf{x},\mathbf{y},t) - \frac{\dot{b}(t)}{A(t)} \mathbf{y}.$$
 Then, we have
    \begin{align*}
        \mathbb{E}_S \Vert \tilde{\boldsymbol{s}}_n(X, Y_t,t) - \boldsymbol{s}^*(X,Y_t,t) \Vert^2
        &= \frac{b(t)}{A(t)} \mathbb{E}_S \Vert \hat{\boldsymbol{b}}_n(\mathbf{x},\mathbf{y},t) - \boldsymbol{b}^*(\mathbf{x},\mathbf{y},t) \Vert^2.
    \end{align*}
 By Theorem \ref{corollary_bound}, we get
    \begin{align*}
        \mathbb{E}_S \Vert \hat{\boldsymbol{b}}_n(X,Y_t,t) - \boldsymbol{b}^*(X,Y_t,t) \Vert^2 = H_b(t)O \left(n^{\frac{-2\beta}{2\beta+k+d}} \log^{\max\{8,\beta\}}n\right),\ t \in (0, 1).
    \end{align*}
 It follows that
    \begin{align*}
        \mathbb{E}_S \Vert \tilde{\boldsymbol{s}}_n(X, Y_t,t) - \boldsymbol{s}^*(X,Y_t,t) \Vert^2=   \frac{H_b(t)b^2(t)}{A^2(t)} O \left(n^{\frac{-2\beta}{2\beta+k+d}} \log^{\max\{8,\beta\}}n\right),  \  t \in (0, 1),
    \end{align*} which complete the proof.
\end{proof}

\subsection{Proof of Lemma \ref{theorem_fp}}
\noindent
\newline

\noindent
Lemma \ref{theorem_fp}. \textit{If Assumptions \ref{assumption0}, \ref{assumption0b} and  \ref{assump_unique} hold, then the marginal preserving property holds in the sense that, for any $\mathbf{x} \in \mathcal{X}$,
$$\rho^*_{\mathbf{x}}(\mathbf{z}, t) = \rho^{\text{ode}}_{\mathbf{x}}(\mathbf{z}, t) = \rho^{\text{sde}}_{\mathbf{x}}(\mathbf{z}, t), \ \text{ for any } \   ( \mathbf{z}, t) \in
 \mathbb{R}^d \times [0,1].
$$
}
\begin{proof}
    For any given $\mathbf{x} \in \mathcal{X}$, recall that the probability density function $\rho^{\text{ode}}_\mathbf{x}$ satisfies the Fokker-Planck equation (\ref{fp_equa}), while both $\rho^*_\mathbf{x}$ and $\rho^{\text{ode}}_\mathbf{x}$ satisfy the Transport equation (\ref{tp_equa}). Under Assumption \ref{assump_unique}, it is evident that $\rho_{\mathbf{x}}^*(\cdot, \cdot)$, $\rho^{\text{ode}}_{\mathbf{x}}(\cdot, \cdot)$, and $\rho^{\text{ode}}_{\mathbf{x}}( \cdot, \cdot)$ belong to $L^\infty(L^1 \cap L^\infty, [0,1])$ according to Proposition 2 in \citep{bris2008existence}. Firstly, at $t=0$, it is obvious that $\rho^*_{\mathbf{x}}(\mathbf{z}, 0) = \rho^{\text{ode}}_{\mathbf{x}}(\mathbf{z}, 0) = \rho^{\text{sde}}_{\mathbf{x}}(\mathbf{z}, 0)$. Furthermore, by the uniqueness of the solution to the Fokker-Planck equation and the Transport equation in the function space $L^\infty(L^1 \cap L^\infty, (0,1))$, we conclude that $\rho_{\mathbf{x}}^* = \rho^{\text{ode}}_{\mathbf{x}}$ for any given $\mathbf{x} \in \mathcal{X}$.

   Next we show that $\rho_{\mathbf{x}}^* = \rho^{\text{sde}}_{\mathbf{x}}$. Assume that for a given $\mathbf{x} \in \mathcal{X}$, there exists $\mathbf{y} \in \mathbb{R}^d$ and $t \in (0,1)$ such that $\rho^{\text{sde}}_\mathbf{x}(\mathbf{y},t) \neq \rho^*_\mathbf{x}(\mathbf{y},t)$. By substituting the definition of $\boldsymbol{b}_{u,\mathbf{x}}^*(\mathbf{z}, t)$ into the Transport equation, we obtain:
    \begin{align*}
        0 =& \partial_t \rho_{\mathbf{x}}^* +  \nabla_{\mathbf{z}} \cdot (\boldsymbol{b}^*_\mathbf{x}\rho_{\mathbf{x}}^*)\\
        =& \partial_t \rho_{\mathbf{x}}^* +  \nabla_{\mathbf{z}} \cdot \{(\boldsymbol{b}^*_\mathbf{x} + u(t)\nabla_{\mathbf{z}} \log \rho_{\mathbf{x}}^* - u(t) \nabla_{\mathbf{z}} \log \rho_{\mathbf{x}}^*) \rho_{\mathbf{x}}^*\} \\
        =& \partial_t \rho_{\mathbf{x}}^* +  \nabla_{\mathbf{z}} \cdot \{\boldsymbol{b}_u\rho_{\mathbf{x}}^* - u(t) \left(\frac{\nabla_{\mathbf{z}} \rho_{\mathbf{x}}^*}{\rho_{\mathbf{x}}^*}\right) \rho_{\mathbf{x}}^*\} \\
        =& \partial_t \rho_{\mathbf{x}}^* +  \nabla_{\mathbf{z}} \cdot (\boldsymbol{b}_u\rho_{\mathbf{x}}^*) - u(t) \Delta_{\mathbf{z}} \rho_{\mathbf{x}}^*.
    \end{align*} This means that $\rho^*_{\mathbf{x}}$ is also the solution of the Fokker-Planck equation (\ref{fp_equa}), which contradicts the uniqueness of the solution. Therefore,  we can conclude that $\rho_{\mathbf{x}}^* = \rho^{\text{sde}}_{\mathbf{x}}$. Then, we obtain  $\rho^*_{\mathbf{x}}(\mathbf{z}, 1) = \rho^{\text{ode}}_{\mathbf{x}}(\mathbf{z}, 1) = \rho^{\text{sde}}_{\mathbf{x}}(\mathbf{z}, 1)$ by integrating $\partial_t \rho{\mathbf{x}}^*$, $\partial_t \rho_{\mathbf{x}}^{\text{ode}}$, $\partial_t \rho_{\mathbf{x}}^{\text{sde}}$ over the interval $[0,1]$. Finally, we obtain $\rho^*_{\mathbf{x}} = \rho^{\text{ode}}_{\mathbf{x}} = \rho^{\text{sde}}_{\mathbf{x}}$ for any $\mathbf{x} \in \mathcal{X}$. This completes the proof.
\end{proof}

\subsection{ Proof of Lemma \ref{lemma_w2}}
\noindent
\newline

\noindent
Lemma \ref{lemma_w2}. \textit{
Suppose that Assumptions \ref{assumption0} and \ref{assumption0b} are satisfied. For any drift function $\boldsymbol{b}(\mathbf{x}, \mathbf{z}, t)$ satisfying $\left\|\boldsymbol{b}(\mathbf{x}, \mathbf{z}, t)-\boldsymbol{b}(\mathbf{x}, \mathbf{y}, t)\right\|_{\infty} \leq l\|\mathbf{z}-\mathbf{y}\|_{\infty}$ for some $l >0.$
    Let $Z_{t,X} := Z_t \mid X$  be obtained through $ \mathrm{d} Z_{t,X} = \boldsymbol{b}(X,Z_{t,X}, t)\mathrm{d}t$ for $t \in (0,1)$ with $Z_{0,X} = Y_0.$ Denote the conditional density of $Z_t = \mathbf{z} \mid X= \mathbf{x}$ by ${\rho}(\mathbf{x}, \mathbf{z}, t)= \rho_{\mathbf{x}}(\cdot,t)$. Then, we have
       {\begin{align} \label{w2}
        \mathbb{E}_{\mathbf{x}}[W_2^2(\rho^*_{\mathbf{x}}(\cdot,t), \rho_{\mathbf{x}}(\cdot,t))]  \leq \exp \left(2 l+1\right) \int_{0}^{t} \mathbb{E} \Vert \boldsymbol{b}^*(\mathbf{x}, Z_{s,\mathbf{x}}^{ode}, s)-\boldsymbol{b}(\mathbf{x}, Z_{s,\mathbf{x}}^{ode}, s) \Vert^2 ds,
       \end{align}}
       where $W_2^2(\cdot,\cdot)$ denotes the 2-Wasserstein distance.
}

\begin{proof} By the definition of the Wasserstein-2 distance and Jensen's inequality \citep{jensen1906fonctions},  it holds that
    \begin{align}
        \mathbb{E}_{\mathbf{x}}[W_2^2(\rho^*_{\mathbf{x}}(\cdot,t), \rho_{\mathbf{x}}(\cdot,t))]
        &\leq  \mathbb{E}_{\mathbf{z} \sim \rho(\cdot,0), \mathbf{x}} \Vert Z_{t,\mathbf{x}}^{ode}(\mathbf{z})- Z_{t,\mathbf{x}}(\mathbf{z}) \Vert^2 \nonumber  \\
        &= \mathbb{E}_{\mathbf{z} \sim \rho(\cdot,0),\mathbf{x}} \left\Vert \int_0^t Z_{s,\mathbf{x}}^{ode}(\mathbf{z})- Z_{s,\mathbf{x}}(\mathbf{z}) \mathrm{d}s \right\Vert^2 \nonumber  \nonumber \\
        &\leq  \int_0^t \mathbb{E}_{\mathbf{z} \sim \rho(\cdot,0), \mathbf{x}} \Vert Z_{s,\mathbf{x}}^{ode}(\mathbf{z})- Z_{s,\mathbf{x}}(\mathbf{z}) \Vert^2 \mathrm{d}s =: \int_0^t \mathbb{E}_{\mathbf{x}}[g(s,\mathbf{x})] \mathrm{d}s,  \label{ja}
    \end{align}
   Similar to the proof idea of \citet{albergo2022building}, considering (\ref{flow_equa}),
   it follows that
    \begin{align}
    \partial_s H(s,\mathbf{x}) = & \partial_s \mathbb{E}_{\mathbf{z} \sim \rho_\mathbf{x}(\cdot,0)} \left\Vert Z_{s,\mathbf{x}}^{ode}(\mathbf{z})- Z_{s,\mathbf{x}}(\mathbf{z}) \right\Vert^2 = \int_{\mathbb{R}^d} \partial_s  \left\Vert Z_{s,\mathbf{x}}^{ode}(\mathbf{z})- Z_{s,\mathbf{x}}(\mathbf{z}) \right\Vert^2 \rho_\mathbf{x}(\mathbf{z},0) \mathbf{dz} \nonumber\\
    = & \int_{\mathbb{R}^d} 2\left\langle \boldsymbol{b}^*(\mathbf{x}, Z_{s,\mathbf{x}}^{ode}(\mathbf{z}), s)-\boldsymbol{b}(\mathbf{x}, Z_{s,\mathbf{x}}(\mathbf{z}), s), Z_{s,\mathbf{x}}^{ode}(\mathbf{z})- Z_{s,\mathbf{x}}(\mathbf{z}) \right\rangle \rho_\mathbf{x}(\mathbf{z},0) \mathbf{dz} \nonumber\\
    = & \int_{\mathbb{R}^d} 2\left\langle \boldsymbol{b}^*(\mathbf{x}, Z_{s,\mathbf{x}}^{ode}(\mathbf{z}), s)-\boldsymbol{b}(\mathbf{x}, Z_{s,\mathbf{x}}^{ode}(\mathbf{z}), s), Z_{s,\mathbf{x}}^{ode}(\mathbf{z})- Z_{s,\mathbf{x}}(\mathbf{z})\right\rangle \rho_\mathbf{x}(\mathbf{z},0) \mathbf{dz} \label{w2_1}\\
    & +\int_{\mathbb{R}^d} 2\left\langle \boldsymbol{b}(\mathbf{x}, Z_{s,\mathbf{x}}^{ode}(\mathbf{z}), s)-\boldsymbol{b}(\mathbf{x}, Z_{s,\mathbf{x}}(\mathbf{z}), s), Z_{s,\mathbf{x}}^{ode}(\mathbf{z})- Z_{s,\mathbf{x}}(\mathbf{z}) \right\rangle \rho_\mathbf{x}(\mathbf{z},0) \mathbf{dz}. \label{w2_2}
    \end{align}
	 By Cauchy–Schwarz inequality, term (\ref{w2_1}) satisfies
\begin{align*}
    &\int_{\mathbb{R}^d} 2\left\langle \boldsymbol{b}^*(\mathbf{x}, Z_{s,\mathbf{x}}^{ode}(\mathbf{z}), s)-\boldsymbol{b}(\mathbf{x}, Z_{s,\mathbf{x}}^{ode}(\mathbf{z}), s), Z_{s,\mathbf{x}}^{ode}(\mathbf{z})- Z_{s,\mathbf{x}}(\mathbf{z})\right\rangle \rho_\mathbf{x}(\mathbf{z},0) \mathbf{dz} \\
    &\leq \mathbb{E}[ \Vert \boldsymbol{b}^*(\mathbf{x}, Z_{s,\mathbf{x}}^{ode}, s)-\boldsymbol{b}(\mathbf{x}, Z_{s,\mathbf{x}}^{ode}, s) \Vert^2 \mid \mathbf{x}] +g(s,\mathbf{x}).
\end{align*}
By the Lipschitz continuity of function $\boldsymbol{b}$, term (\ref{w2_2}) can be bounded by
\begin{align*}
    \int_{\mathbb{R}^d} 2\left\langle \boldsymbol{b}(\mathbf{x}, Z_{s,\mathbf{x}}^{ode}(\mathbf{z}), s)-\boldsymbol{b}(\mathbf{x}, Z_{s,\mathbf{x}}(\mathbf{z}), s), Z_{s,\mathbf{x}}^{ode}(\mathbf{z})- Z_{s,\mathbf{x}}(\mathbf{z}) \right\rangle \rho_\mathbf{x}(\mathbf{z},0) \mathbf{dz} \leq 2l g(s,\mathbf{x}).
\end{align*}
Therefore, we have
\begin{align*}
     \partial_s g(s,\mathbf{x}) \leq\left(2 l+1\right) g(s,\mathbf{x})+  \mathbb{E}[ \Vert \boldsymbol{b}^*(\mathbf{x}, Z_{s,\mathbf{x}}^{ode}, s)-\boldsymbol{b}(\mathbf{x}, Z_{s,\mathbf{x}}^{ode}, s) \Vert^2 \mid \mathbf{x}] .
\end{align*}
By Gr$\ddot{\text{o}}$nwall's inequality, it further yields
\begin{align*}
    g(s,\mathbf{x}) \leq \exp (2 l+1) \int_{0}^{s}  \mathbb{E}[ \Vert \boldsymbol{b}^*(\mathbf{x}, Z_{s,\mathbf{x}}^{ode}, s)-\boldsymbol{b}(\mathbf{x}, Z_{s,\mathbf{x}}^{ode}, s) \Vert^2 \mid \mathbf{x}] dt.
\end{align*}
Combining the above inequalities, we obtain
\begin{align} \label{w2_3}
    \mathbb{E}_{\mathbf{x}}[W_2^2(\rho^*_{\mathbf{x}}(\cdot,t), \rho_{\mathbf{x}}(\cdot,t))]  \leq \exp \left(2 l+1\right) \int_{0}^{t} \mathbb{E} \Vert \boldsymbol{b}^*(\mathbf{x}, Z_{t,\mathbf{x}}^{ode}, t)-\boldsymbol{b}(\mathbf{x}, Z_{t,\mathbf{x}}^{ode}, t) \Vert^2 dt,
\end{align}
which completes the proof.

\end{proof}

\subsection{Proof of Theorem \ref{lemma_w2}}
\noindent
\newline

\noindent
    Theorem \ref{lemma_w2}. \textit{
    Suppose $H_b(t)$ is integrable in $[0,1)$. Assuming that the conditions in Theorem \ref{corollary_bound}  and Lemma \ref{lemma_w2} are satisfied, we have
    \begin{align}
 \mathbb{E}_{S,X}[W_2^2(\rho^{*}_{X}(\cdot,1), \hat{{\rho}}_{n,X}^{ode}(\cdot,1))] = O \left(n^{\frac{-2\beta}{2\beta+k+d}} \log^{\max\{8,\beta\}}n \right).
 \end{align}
}

\begin{proof}
        As Lemma \ref{theorem_fp} states, $Z_{t}^{ode}$ and $Y_t$ follow the same distribution. By inequality (\ref{w2_3}), we obtain
        \begin{align*}
            \mathbb{E}_{\mathbf{x}}[W_2^2(\rho^*_{\mathbf{x}}(\cdot,t), \hat{\rho}_{n,\mathbf{x}}^{ode}(\cdot,t)) \mid S]
            &\leq \exp \left(2 l+1\right) \int_{0}^{t}\mathbb{E}[ \Vert \boldsymbol{b}^*(\mathbf{x}, Z_{t,\mathbf{x}}^{ode}, t)-\hat{\boldsymbol{b}}_{n}(\mathbf{x}, Z_{t,\mathbf{x}}^{ode}, t) \Vert^2 \mid S] dt,\\
            &= \exp \left(2 l+1\right)  \int_{0}^{t} \mathbb{E} [\Vert \boldsymbol{b}^*(\mathbf{x}, \mathbf{y}_t, t)-\hat{\boldsymbol{b}}_{n}(\mathbf{x}, \mathbf{y}_t, t) \Vert^2 \mid S] dt.
        \end{align*}
     Taking expectations with respect to the sample $S$ on both sides of above inequality, we have
     \begin{align*}
        \mathbb{E}_{,\mathbf{x}}[W_2^2(\rho^{ode}_{\mathbf{x}}(\cdot,t), \hat{\rho}_{n,\mathbf{x}}^{ode}(\cdot,t))] \leq \exp \left(2 l+1\right)  \int_{0}^{t}\mathbb{E}_S \Vert \boldsymbol{b}^*(\mathbf{x}, \mathbf{y}_t, t)-\boldsymbol{b}(\mathbf{x}, \mathbf{y}_t, t) \Vert^2 dt.
     \end{align*}
        Combining this bound with Corollary \ref{corollary_bound} and Lemma \ref{lemma_w2} completes the proof.
\end{proof}

\subsection{ Proof of Lemma \ref{kl_lemma}}

\noindent
\newline

\noindent
Lemma \ref{kl_lemma}. \textit{
Suppose that Assumptions \ref{assumption0} and \ref{assumption0b} are satisfied. For any velocity ${\boldsymbol{b}}(\mathbf{x}, \mathbf{z}, t)$ and distribution field ${\boldsymbol{s}}(\mathbf{x},\mathbf{z},t)$, we define ${\boldsymbol{b}}_{u}(\mathbf{x}, \mathbf{z},t):={\boldsymbol{b}}(\mathbf{x}, \mathbf{z},t)+ u(t){\boldsymbol{s}}(\mathbf{x},\mathbf{z},t)$. Let $Z_{t,X} := Z_t \mid X, t \in (0,1]$, be defined according to $\mathrm{d} Z_{t,X} = \boldsymbol{b}_u(X,Z_{t,X}, t)\mathrm{d}t + \sqrt{2u(t)} \mathrm{d} W_t$ with $Z_{0,X} = Y_0$.
Denote the corresponding time-dependent conditional density
 of $Z_{t,X}$ by ${\rho}(\mathbf{x}, \mathbf{z}, t)={\rho}_\mathbf{x}(\mathbf{z}, t).$ Then, for any integrable functions $u $ and $u ^{-1}$ and any $t\in [0,1],$ we have
   \begin{align} \label{kl}
       {\mathbb{E}_{\mathbf{x}}[\mathrm{KL}(\rho_{\mathbf{x}}^{*}(\cdot,t) \| {\rho}_{\mathbf{x}}(\cdot,t))]} \leq \int_0^t &\frac{1}{2u(s)}
       \mathbb{E} \Vert {\boldsymbol{b}(\mathbf{x},Z_{s,\mathbf{x}}^{sde}, s)} - \boldsymbol{b}^*(\mathbf{x},Z_{s,\mathbf{x}}^{sde}, s) \Vert^2 ds \\
       &+ \int_0^t \frac{u(s)}{2} \mathbb{E} \Vert {\boldsymbol{s}(\mathbf{x},Z_{s,\mathbf{x}}^{sde}, s)} - \boldsymbol{s}^*(\mathbf{x},Z_{s,\mathbf{x}}^{sde}, s) \Vert^2 ds. \nonumber
   \end{align}
}
\begin{proof}
     We introduce $\boldsymbol{b}_{u,\mathbf{x}}(\mathbf{z},t) := \boldsymbol{b}_u (\mathbf{x},\mathbf{z},t)$ and new notation $v$ to represent time. For notational simplicity, we omit the argument $(\mathbf{z},v)$ of all functions. Using (\ref{fp_equa}) and (\ref{fp_estmator}), we compute analytically, for any given $\mathbf{x} \in \mathcal{X}$,
{\small
    \begin{align*}
   & \frac{d}{d v} \mathrm{KL}(\rho_{\mathbf{x}}^*(\cdot,v) \| {\rho}_{{\mathbf{x}}}(\cdot,v))  =\frac{d}{d v} \int_{\mathbb{R}^d} \log \left(\frac{{\rho}_{\mathbf{x}}^*}{{\rho}_{\mathbf{x}}}\right) \rho_{\mathbf{x}}^* \mathbf{dz} \\
    & =\int_{\mathbb{R}^d} {\rho}_{\mathbf{x}}\left(\frac{\partial_v \rho_{\mathbf{x}}^*}{{\rho}_{\mathbf{x}}}-\frac{\rho_{\mathbf{x}}^*}{({\rho}_{\mathbf{x}})^2} \partial_v {\rho}_{\mathbf{x}}\right) \mathbf{dz}+\int \log \left(\frac{\rho_{\mathbf{x}}^*}{{\rho}_{\mathbf{x}}}\right) \partial_v \rho_{\mathbf{x}}^* \mathbf{dz} \\
    & =-\int_{\mathbb{R}^d}\left(\frac{\rho_{\mathbf{x}}^*}{{\rho}_{\mathbf{x}}}\right) \partial_v {\rho}_{\mathbf{x}} \mathbf{dz}+\int_{\mathbb{R}^d} \log \left(\frac{\rho_{\mathbf{x}}^*}{{\rho}_{\mathbf{x}}}\right) \partial_v \rho_{\mathbf{x}}^* \mathbf{dz} \\
    & =\int_{\mathbb{R}^d}\left(\frac{\rho_{\mathbf{x}}^*}{{\rho}_{\mathbf{x}}}\right) \left\{\nabla \cdot({\boldsymbol{b}}_{u,\mathbf{x}}  {\rho}_{\mathbf{x}}) -u(v) \Delta {\rho}_{\mathbf{x}} \right\} \mathbf{dz} \\
    &\quad \quad  + \int \log \left(\frac{\rho_{\mathbf{x}}^*}{{\rho}_{\mathbf{x}}}\right) \left\{-\nabla \cdot(\boldsymbol{b}_{u, \mathbf{x}}^* \rho_{\mathbf{x}}^*)+u(v)\Delta\rho_{\mathbf{x}}^* \right\} \mathbf{dz} \\
    & =\int_{\mathbb{R}^d} \frac{\rho_{\mathbf{x}}^*}{{\rho}_{\mathbf{x}}} \nabla \cdot ({\boldsymbol{b}}_{u,\mathbf{x}}  {\rho}_{\mathbf{x}}) \mathbf{dz} -\int_{\mathbb{R}^d} \log \left(\frac{\rho_{\mathbf{x}}^*}{{\rho}_{\mathbf{x}}}\right) \nabla \cdot (\boldsymbol{b}_{u,\mathbf{x}}^*  \rho_{\mathbf{x}}^*) \mathbf{dz} \\
    &\quad \quad - u(v) \int_{\mathbb{R}^d} \frac{\rho_{\mathbf{x}}^*}{{\rho}_{\mathbf{x}}} \nabla \cdot({\rho}_{\mathbf{x}} \nabla \log {\rho}_{\mathbf{x}}) \mathbf{dz} + u(v) \int_{\mathbb{R}^d} \log \frac{\rho_{\mathbf{x}}^*}{{\rho}_{\mathbf{x}}} \nabla \cdot(\rho_{\mathbf{x}}^* \nabla \log \rho_{\mathbf{x}}^*) \mathbf{dz}\\
    & =-\int_{\mathbb{R}^d} \nabla\left(\frac{\rho_{\mathbf{x}}^*}{{\rho}_{\mathbf{x}}}\right) \cdot {\boldsymbol{b}}_{u,\mathbf{x}}  {\rho}_{\mathbf{x}} \mathbf{dz} +\int_{\mathbb{R}^d} \nabla \log \left(\frac{\rho_{\mathbf{x}}^*}{{\rho}_{\mathbf{x}}}\right)\cdot \boldsymbol{b}_{u,\mathbf{x}}^* \rho_{\mathbf{x}}^* \mathbf{dz} \\
    &\quad \quad + u(v) \int_{\mathbb{R}^d} \nabla \left(\frac{\rho_{\mathbf{x}}^*}{{\rho}_{\mathbf{x}}}\right) \cdot({\rho}_{\mathbf{x}} \nabla \log {\rho}_{\mathbf{x}}) \mathbf{dz} - u(v) \int_{\mathbb{R}^d} \nabla \log \left(\frac{\rho_{\mathbf{x}}^*}{{\rho}_{\mathbf{x}}}\right) \cdot(\rho_{\mathbf{x}}^* \nabla \log \rho_{\mathbf{x}}^*) \mathbf{dz}\\
    & =-\int_{\mathbb{R}^d}\left(\frac{\nabla \rho_{\mathbf{x}}^*}{{\rho}_{\mathbf{x}}}-\frac{\rho_{\mathbf{x}}^* \nabla {\rho}_{\mathbf{x}}}{{\rho}_{\mathbf{x}}^2}\right) \cdot {\boldsymbol{b}}_{u, \mathbf{x}}  {\rho}_{\mathbf{x}} \mathbf{dz}+\int_{\mathbb{R}^d}(\nabla \log \rho_{\mathbf{x}}^*-\nabla \log {\rho}_{\mathbf{x}}) \cdot \boldsymbol{b}_{u,\mathbf{x}}^* \rho_{\mathbf{x}}^* \mathbf{dz} \\
    &\quad \quad + u(v) \int_{\mathbb{R}^d} \left(\frac{\nabla \rho_{\mathbf{x}}^*}{{\rho}_{\mathbf{x}}}-\frac{\rho_{\mathbf{x}}^* \nabla {\rho}_{\mathbf{x}}}{{\rho}_{\mathbf{x}}^2}\right) \cdot({\rho}_{\mathbf{x}} \nabla \log {\rho}_{\mathbf{x}}) \mathbf{dz} - u(v) \int_{\mathbb{R}^d} \nabla \log \left(\frac{\rho_{\mathbf{x}}^*}{{\rho}_{\mathbf{x}}}\right) \cdot(\rho_{\mathbf{x}}^* \nabla \log \rho_{\mathbf{x}}^*) \mathbf{dz}\\
    & =-\int_{\mathbb{R}^d}(\nabla \log \rho_{\mathbf{x}}^* - \nabla \log {\rho}_{\mathbf{x}}) \cdot {\boldsymbol{b}}_{u, \mathbf{x}}\rho_{\mathbf{x}}^* \mathbf{dz}+\int_{\mathbb{R}^d}(\nabla \log \rho_{\mathbf{x}}^*-\nabla \log {\rho}_{\mathbf{x}}) \cdot \boldsymbol{b}_{u,\mathbf{x}, \mathbf{x}}^*  \rho_{\mathbf{x}}^* \mathbf{dz} \\
    & + u(v) \int_{\mathbb{R}^d}  (\nabla \log \rho_{\mathbf{x}}^* - \nabla \log {\rho}_{\mathbf{x}}) \cdot \nabla \log {\rho}_{\mathbf{x}} \rho_{\mathbf{x}}^* \mathbf{dz} - u(v) \int_{\mathbb{R}^d} (\nabla \log \rho_{\mathbf{x}}^* - \nabla \log {\rho}_{\mathbf{x}}) \cdot( \nabla \log \rho_{\mathbf{x}}^*) \rho_{\mathbf{x}}^* \mathbf{dz}\\
    & =\int_{\mathbb{R}^d}(\nabla \log {\rho}_{\mathbf{x}}-\nabla \log \rho_{\mathbf{x}}^*) \cdot({\boldsymbol{b}}_{u,\mathbf{x}} -\boldsymbol{b}_{u,\mathbf{x}}^* ) \rho_{\mathbf{x}}^* d \mathbf{z}- u(v)\int_{\mathbb{R}^d}\Vert\nabla \log {\rho}_{\mathbf{x}}-\nabla \log \rho_{\mathbf{x}}^*\Vert ^2\rho_{\mathbf{x}}^* \mathbf{dz},
    \end{align*}
}
where for any differentiable function $f: \mathbb{R}^d \rightarrow \mathbb{R}^1$, it holds $\Delta f = \nabla \cdot (f \nabla \log f)$. Integrating both sides with respect to time from $0$ to $t$, we obtain
    \begin{align} \label{b12_1}
        \mathrm{KL}(\rho_{\mathbf{x}}^*(\cdot,t) \| {\rho}_{\mathbf{x}}(\cdot,t)) = & \int_0^t \int_{\mathbb{R}^d}(\nabla \log {\rho}_{\mathbf{x}}-\nabla \log \rho_{\mathbf{x}}^*) \cdot({\boldsymbol{b}}_{u,\mathbf{x}} -\boldsymbol{b}_{u,\mathbf{x}}^* ) \rho_{\mathbf{x}}^* \mathbf{dz} dv \nonumber \\
        &- \int_0^t u(v)\int_{\mathbb{R}^d}\Vert\nabla \log {\rho}_{\mathbf{x}}-\nabla \log \rho_{\mathbf{x}}^*\Vert ^2\rho_{\mathbf{x}}^* \mathbf{dz} dv.
    \end{align}
    Using the definition of $\boldsymbol{b}^*_u$ and ${\boldsymbol{b}}_{u}$, we have
{\small
    \begin{align} \label{b12_2}
  &      \int_0^t \int_{\mathbb{R}^d}(\nabla \log {\rho}_{\mathbf{x}}-\nabla \log \rho_{\mathbf{x}}^*) ({\boldsymbol{b}}_{u, \mathbf{x}} -\boldsymbol{b}_{u,\mathbf{x}}^* ) \rho_{\mathbf{x}}^* \mathbf{dz} dv \nonumber \\
        &= \int_0^t \int_{\mathbb{R}^d}(\nabla \log {\rho}_{\mathbf{x}}-\nabla \log \rho_{\mathbf{x}}^*)  \left[({\boldsymbol{b}}-\boldsymbol{b}^*  + u\hat{\boldsymbol{s}}_n- u \boldsymbol{s}^*)(\mathbf{x}, \cdot) \right] \rho_{\mathbf{x}}^* \mathbf{dz} dv \nonumber \\
        &= \int_0^t \int_{\mathbb{R}^d}\sqrt{2u(v)}(\nabla \log {\rho}_{\mathbf{x}}-\nabla \log \rho_{\mathbf{x}}^*) \left\{\frac{1}{\sqrt{2u(v)}}({\boldsymbol{b}}-\boldsymbol{b}^* )(\mathbf{x}, \cdot) + \sqrt{\frac{u(v)}{2}}({\boldsymbol{s}}- \boldsymbol{s}^* )(\mathbf{x}, \cdot) \right\} \rho_{\mathbf{x}}^* \mathbf{dz} dv \nonumber \\
        &\leq \int_0^t \int_{\mathbb{R}^d} u(v)\Vert \nabla\log {\rho}_{\mathbf{x}}- \nabla\log \rho_{\mathbf{x}}^* \Vert^2 \rho_{\mathbf{x}}^* \mathbf{dz} dv \nonumber \\
        & \quad \quad+ \frac{1}{2}\int_0^t \int_{\mathbb{R}^d} \Vert \frac{1}{\sqrt{2u(v)}} ({\boldsymbol{b}}-\boldsymbol{b}^* )(\mathbf{x}, \cdot) + \sqrt{\frac{u(v)}{2}}({\boldsymbol{s}}- \boldsymbol{s}^* )(\mathbf{x}, \cdot) \Vert^2 \rho_{\mathbf{x}}^* \mathbf{dz} dv \nonumber \\
        &\leq \int_0^t \int_{\mathbb{R}^d} u(v)\Vert \nabla\log {\rho}_{\mathbf{x}}- \nabla\log \rho \Vert^2 \rho_{\mathbf{x}}^* \mathbf{dz} dv \nonumber \\
        &  + \int_0^t \int_{\mathbb{R}^d} \frac{1}{2u(v)}\Vert ({\boldsymbol{b}}  - \boldsymbol{b}^*) (\mathbf{x}, \cdot) \Vert^2 \rho_{\mathbf{x}}^* \mathbf{dz} dv
        + \int_0^t \int_{\mathbb{R}^d} \frac{u(v)}{2}\Vert ({\boldsymbol{s}}  - \boldsymbol{s}^*)(\mathbf{x}, \cdot) \Vert^2 \rho_{\mathbf{x}}^* \mathbf{dz} dv \nonumber \\
        &= \int_0^t \int_{\mathbb{R}^d} u(v)\Vert \nabla\log {\rho}_{\mathbf{x}}- \nabla\log \rho_{\mathbf{x}}^* \Vert^2 \rho_{\mathbf{x}}^* \mathbf{dz} dv \nonumber \\
        &  + \int_0^t \frac{1}{2u(v)} \mathbb{E}\left[\Vert {\boldsymbol{b}}  - \boldsymbol{b}^* \Vert^2 \mid X = \mathbf{x} \right]  dv
        + \int_0^t \frac{u(v)}{2} \mathbb{E}\left[\Vert {\boldsymbol{s}}  - \boldsymbol{s}^* \Vert^2 \mid X = \mathbf{x}\right]  dv.
    \end{align}
}

Combining  (\ref{b12_1}) and (\ref{b12_2}),  and taking expectation with respect to $\mathbf{x}$ on both sides of the above inequality, we obtain
\begin{align*}
    \mathbb{E}_{\mathbf{x}}[\mathrm{KL}(\rho_{\mathbf{x}}^*(\cdot,t) \| {\rho}_{\mathbf{x}}(\cdot,t))]  \leq \int_0^t \frac{1}{2u(v)} \mathbb{E}\Vert {\boldsymbol{b}}  - \boldsymbol{b}^* \Vert^2  dv + \int_0^t \frac{u(v)}{2} \mathbb{E}\Vert {\boldsymbol{s}}  - \boldsymbol{s}^* \Vert^2  dv.
\end{align*}
This completes the proof.
\end{proof}

\subsection{ Proof of Theorem \ref{kl_theorem}}

\noindent
\newline

\noindent
Theorem \ref{kl_theorem}. \textit{
Suppose that the conditions stated in Theorem \ref{corollary_bound} are satisfied.
Furthermore, suppose that $H_b(t)u(t)^{-1}$ and $u(t)H_\kappa(t)\gamma(t)^{-1}$  are integrable on $[0,1],$ we have
    \begin{align*}
        {\mathbb{E}_{S,X}[\mathrm{KL}(\rho_{X}^{*}(\cdot,1) \| \hat{\rho}_{n, X}^{sde}(\cdot,1))]} = O \left(n^{\frac{-2\tilde{\beta}}{2\tilde{\beta}+k+d}} \log^{\max\{8,\tilde{\beta}\}}n\right),
    \end{align*} where $\tilde{\beta}=\min\{\beta,\beta^\prime\}$.
}
\begin{proof}
   As Lemma \ref{theorem_fp} states, $Z_t^{sde}$ and $Y_t$ follow the same distribution. Combining Inequality (\ref{kl}), we obtain
   \begin{align*}
   \mathbb{E}_{\mathbf{x}}[\mathrm{KL}(\rho_{\mathbf{x}}^*(\cdot,t) \| \hat{\rho}_{n,\mathbf{x}}^{sde}(\cdot,1)) \mid S]
    &\leq \int_0^t \frac{1}{2u(s)}
    \mathbb{E} [\Vert {\hat{\boldsymbol{b}}_n(\mathbf{x},\mathbf{z}_s, s)} - \boldsymbol{b}^*(\mathbf{x},\mathbf{z}_s, s) \Vert^2 \mid S] ds\\
 & \ \ \    + \int_0^t \frac{u(s)}{2} \mathbb{E} [\Vert {\hat{\boldsymbol{s}}_n(\mathbf{x},\mathbf{z}_s, s)} - \boldsymbol{s}^*(\mathbf{x},\mathbf{z}_s, s) \Vert^2 \mid S ] ds \\
    &= \int_0^t \frac{1}{2u(s)}
    \mathbb{E} [\Vert {\hat{\boldsymbol{b}}_n(\mathbf{x},\mathbf{y}_s, s)} - \boldsymbol{b}^*(\mathbf{x},\mathbf{y}_s, s) \Vert^2 \mid S] ds \\
 & \ \ \    + \int_0^t \frac{u(s)}{2} \mathbb{E} [\Vert {\hat{\boldsymbol{s}}_n(\mathbf{x},\mathbf{y}_s, s)} - \boldsymbol{s}^*(\mathbf{x},\mathbf{y}_s, s) \Vert^2 \mid S] ds
\end{align*}
Then, take expectations with respect to randomness in the sample $S$ on both sides of the inequality, we have
\begin{align*}
&
{\mathbb{E}_{S,\mathbf{x}}[\mathrm{KL}(\rho_{\mathbf{x}}^{sde}(\cdot,t) \| {\rho}_{\mathbf{x}}(\cdot,t))]} \\
    &\leq \int_0^t \frac{1}{2u(s)}
    \mathbb{E}_S \Vert {\boldsymbol{b}(\mathbf{x},\mathbf{y}_s, s)} - \boldsymbol{b}^*(\mathbf{x},\mathbf{y}_s, s) \Vert^2 ds
    + \int_0^t \frac{u(s)}{2} \mathbb{E}_S \Vert {\boldsymbol{s}(\mathbf{x},\mathbf{y}_s, s)} - \boldsymbol{s}^*(\mathbf{x},\mathbf{y}_s, s) \Vert^2 ds.
\end{align*}
   Combining this bound with Corollary \ref{corollary_bound} and Lemma \ref{kl_lemma} completes the proof.
\end{proof}

\section{Supporting lemmas}

In this section, we present several lemmas used in our proofs. Specifically, Lemmas \ref{lemma_b1}-\ref{lemma_b4} are established results. Furthermore, Lemmas \ref{theorem_33} and \ref{corollary_3} represent extensions of Theorem 3.3 and Corollary 3.1 from \citet{jiao2023deep}, respectively. These lemmas extend the estimation results of neural networks from a fixed interval to a variable interval.

\begin{lemma}[Proposition 4.3 in \citet{lu2021deep}] \label{lemma_b1}
    For any $N, M, d \in \mathbb{N}^{+}$and $\delta \in$ $(0,3 K]$ with $K=\left\lfloor N^{1 / d}\right\rfloor^2\left\lfloor M^{2 / d}\right\rfloor$, there exists a one-dimensional function $\phi$ implemented by a ReLU FNN with width $4\left\lfloor N^{1 / d}\right\rfloor+3$ and depth $4 M+5$ such that
$$
\phi(x)=k, \quad \text { if } x \in\left[\frac{k}{K}, \frac{k+1}{K}-\delta \cdot 1_{k<K-1}\right] \text {, for } k=0,1, \ldots, K-1.
$$
\end{lemma}

\begin{lemma}[Proposition 4.4 in \citet{lu2021deep}] \label{lemma_b2}
    Given any $N, M, s \in \mathbb{N}^{+}$and $\xi_i \in$ $[0,1]$ for $i=0,1, \ldots, N^2 L^2-1$, there exists a function $\phi$ implemented by a ReLU FNN with width $16 s(N+1)\left\lceil\log _2(8 N)\right\rceil$ and depth $5(M+2)\left\lceil\log _2(4 M)\right\rceil$ such that
$$
\left|\phi(i)-\xi_i\right| \leq N^{-2 s} M^{-2 s} \text {, for } i=0,1, \ldots, N^2 M^2-1 \text {, }
$$
and $0 \leq \phi(x) \leq 1$ for any $x \in \mathbb{R}$.
\end{lemma}

\begin{lemma}[Lemma 4.2, \citet{lu2021deep}] \label{lemma_b3}
    For any $N, L \in \mathbb{N}^{+}$and $a, b \in \mathbb{R}$ with $a<b$, there exists a function $\phi$ implemented by a ReLU FNN with width $9 N+1$ and depth $L$ such that
$$
|\phi(x, y)-x y| \leq 6(b-a)^2 N^{-L} \quad \text { for any } x, y \in[a, b].
$$
\end{lemma}

\begin{lemma}[Proposition 4.1 in \citet{lu2021deep}] \label{lemma_b4}
     Assume $P(x)=x^\alpha=x_1^{\alpha_1} x_2^{\alpha_2} \cdots x_d^{\alpha_d}$ for $\alpha \in \mathbb{N}^d$ with $\|\alpha\|_1 \leq k \in \mathbb{N}^{+}$. For any $N, M \in \mathbb{N}^{+}$, there exists a function $\phi$ implemented by a ReLU FNN with width $9(N+1)+k-1$ and depth $7 k^2 M$ such that
$$
|\phi(x)-P(x)| \leq 9 k(N+1)^{-7 k M}, \quad \text { for any } x \in[0,1]^d \text {. }
$$
\end{lemma}

\begin{lemma} \label{theorem_33}
    Assume that $f \in \mathcal{H}^\beta\left([a,b]^d, B_0\right)$ with $b > a$, $\beta=s+r, s \in \mathbb{N}_0$ and $r \in(0,1]$. For any $M, N \in \mathbb{N}^{+}$, there exists a function $\phi_0$ implemented by a ReLU network with width $\mathcal{W}=38(\lfloor\beta\rfloor+1)^2 d^{\lfloor\beta\rfloor+1} U\left\lceil\log _2(8 U)\right\rceil$ and depth $\mathcal{D}=21(\lfloor\beta\rfloor+1)^2 V\left\lceil\log _2(8 V)\right\rceil$ such that
$$
\left|f(\mathbf{z})-\phi_0(\mathbf{z})\right| \leq 18 B_0(\lfloor\beta\rfloor+1)^2 d^{\lfloor\beta\rfloor+(\beta \vee 1) / 2}(U V)^{-2 \beta / d},
$$
for all $\mathbf{z} \in[a,b]^d \backslash \Omega\left([a,b]^d, J, \delta\right)$, and
$$
\Omega\left([a,b]^d, J, \delta\right)=\bigcup_{i=1}^d\left\{\mathbf{z}=\left[\mathbf{z}^{(1)}, \mathbf{z}^{(2)}, \ldots, \mathbf{z}^{(d)}\right]^{\top}: \mathbf{z}^{(i)} \in \bigcup_{j=1}^{\lceil b-a\rceil J-1}(j / J-\delta, j / J)\right\}
$$
with $J=\lceil(U V)^{2 / d}\rceil / \lceil b-a\rceil $ and $\delta$ an arbitrary number in $(0,1 /(3 J)]$. Here, $\lceil t_1\rceil$ denotes the smallest integer no less than $t_1$.
\end{lemma}

\begin{proof}
    By extending the proof techniques from Theorem 3.3 in \citet{jiao2023deep}, we generalize the result to the case that the target function is defined on an arbitrary hypercube $[a,b]^d$. The proof idea is divided into three parts: discretization of the domian, approximation of Taylor coefficients and approximation of $f$. For convenience, we define a Hölder function, $\tilde{f}(\mathbf{x}) := f(\mathbf{x}) / B_0 \in \mathcal{H}^\beta\left([a,b]^d, 1\right)$ and first discuss the case when $\beta>1$.

    For discretization purposes, we consider $M$ small partitions of length $1/J$, where $M := \lceil b - a \rceil J$ and $J \in \mathbb{N}^{+}$. Given $\delta \in(0,1 /(3 J)]$, we define
$$
Q_{\boldsymbol{\tau}}:=\left\{\mathbf{z}=\left(\mathbf{z}^{(1)}, \ldots, \mathbf{z}^{(d)}\right): \mathbf{z}^{(i)} \in\left[\frac{\boldsymbol{\tau}^{(i)}}{J}, \frac{\boldsymbol{\tau}^{(i+1)}}{J}-\delta \cdot 1_{\boldsymbol{\tau}^{(i)}<M-1}\right], i=1, \ldots, d\right\},
$$ for all $\boldsymbol{\tau}=\left(\boldsymbol{\tau}_1, \ldots, \boldsymbol{\tau}_d\right) \in\{0,1, \ldots, M-1\}^d$. Note that $[a,b]^d \backslash \Omega\left([a,b]^d, J, \delta\right)=\bigcup_{\boldsymbol{\tau}} Q_{\boldsymbol{\tau}}$.

By Lemma \ref{lemma_b1}, there exists a ReLU network $\psi_1$ with width $4\left\lfloor U^{1 / d}\right\rfloor+3$ and depth $4 V +5$ such that
$$
\psi_1(\mathbf{z})=\frac{j}{J}, \quad \text { if } z \in\left[\frac{j}{J}, \frac{j+1}{J}-\delta \cdot 1_{\{j<M-1\}}\right], j=0,1, \ldots, M-1 .
$$

We then construct a ReLU network with width $d\left(4\left\lfloor U^{1 / d}\right\rfloor+3\right)$ and depth $4 V+5$, $\psi(\mathbf{z}):=\left(\psi_1\left(\mathbf{z}^{(1)}\right), \ldots, \psi_1\left(\mathbf{z}^{(d)}\right)\right)$. Note that $\psi(\mathbf{z})=\boldsymbol{\tau} / J:=\left(\boldsymbol{\tau}^{(1)}/ J, \ldots, \boldsymbol{\tau}^{(d)} / J\right)^{\top}$ for $\mathbf{z} \in Q_{\boldsymbol{\tau}}$. Using the one-to-one correspondence between $ \boldsymbol{\tau} \in\{0,1, \ldots, M-1 \}^d$ and $i_{\boldsymbol{\tau}}:=\sum_{i=1}^d \boldsymbol{\tau}^{(i)} (M)^{i-1} \in$ $\left\{0,1 \ldots, M^d-1\right\}$, we then define a ReLU network with width $d\left(4\left\lfloor N^{1 / d}\right\rfloor+3\right)$ and depth $4 M+5$,
$$
\psi_0(\mathbf{z}):=\left(M, M^2, \ldots, M^{d}\right) \cdot \psi(\mathbf{z})=\sum_{l=1}^d \psi_1\left(\mathbf{z}^{(l)}\right) M^{l}
$$ for $\mathbf{z} \in \mathbb{R}^d$, then
$$
\psi_0(\mathbf{z})=\sum_{l=1}^d \boldsymbol{\tau}^{(l)} M^{l-1}=i_{\boldsymbol{\tau}},
$$ if $ \mathbf{z} \in Q_{\boldsymbol{\tau}}, \boldsymbol{\tau} \in\{0,1, \ldots, M-1\}^d$.

Next, we construct a network to approximate the Taylor coefficients. For any $\boldsymbol{\alpha}\in \mathbb{N}_0^d$ satisfying $\|\boldsymbol{\alpha}\|_1 \leq$ $s$ and each $i=i_{\boldsymbol{\tau}}\in\left\{0,1, \ldots, M^d-1\right\}$, we denote $\xi_{\boldsymbol{\alpha}, i}:=\left(\partial^{\boldsymbol{\alpha}} \tilde{f}(\boldsymbol{\tau}/ J)+1\right) / 2 \in[0,1]$. Since $M^d \leq U^2 V^2$, by Lemma \ref{lemma_b2}, there exists a ReLU network $\varphi_{\boldsymbol{\alpha}}$ with width $16(s+1)(U+$ 1) $\lceil\log _2(8 U)\rceil$ and depth $5(V+2)\lceil\log _2(4 V)\rceil$ such that
$$
\left|\varphi_{\boldsymbol{\alpha}}(i)-\xi_{\boldsymbol{\alpha}, i}\right| \leq(UV)^{-2(s+1)},
$$
for all $i \in \{0,1, \ldots, M^d-1 \}$. We then obtain a network with width $16 d(s+1)(N+1)\left\lceil\log _2(8 N)\right\rceil \leq$ $32 d(s+1) N\left\lceil\log _2(8 N)\right\rceil$ and depth $5(M+2)\left\lceil\log _2(4 M)\right\rceil+4 M+5 \leq 15 M\left\lceil\log _2(8 M)\right\rceil$,
$$
\phi_{\boldsymbol{\alpha}}(\mathbf{z}):=2 \varphi_{\boldsymbol{\alpha}}(\psi_0(\mathbf{z}))-1 \in[-1,1].
$$

And for any $\boldsymbol{\tau} \in \{0,1, \ldots, M-1\}^d$, we have
\begin{align} \label{lemma_a8}
    \left|\phi_{\boldsymbol{\alpha}}(\mathbf{z})-\partial^{\boldsymbol{\alpha}} f({\boldsymbol{\tau}} / J)\right|=2\left|\varphi_{\boldsymbol{\alpha}}(i_{\boldsymbol{\tau}})-\xi_{{\boldsymbol{\alpha}}, i_{\boldsymbol{\tau}}}\right| \leq 2(U V)^{-2(s+1)}, \text{ for } \mathbf{z} \in Q_{\boldsymbol{\tau}}
\end{align}

Next, we approximate $\tilde{f}$ on $\bigcup_{\boldsymbol{\tau} \in\{0,1, \ldots, M-1\}^d} Q_{\boldsymbol{\tau}}$. By Lemma A. 8 in \citet{petersen2018optimal}, we obtain the following inequality,
\begin{align} \label{a_11}
    \left|\tilde{f}(\mathbf{z})-\tilde{f}\left(\frac{\boldsymbol{\tau}}{J}\right)-\sum_{1 \leq\|\boldsymbol{\alpha}\|_1 \leq s} \frac{\partial^{\boldsymbol{\alpha}} \tilde{f}\left(\frac{\boldsymbol{\tau}}{J}\right)}{\boldsymbol{\alpha}!}\left(\mathbf{z}-\frac{\boldsymbol{\tau}}{J}\right)^{\boldsymbol{\alpha}}\right| \leq d^s\left\|\mathbf{z}-\frac{\boldsymbol{\tau}}{J}\right\|_2^\beta \leq d^{s+\beta / 2} J^{-\beta},
\end{align} where $\boldsymbol{\alpha}! := \boldsymbol{\alpha}^{(1)}\times ... \times \boldsymbol{\alpha}^{(d)}$. Through Lemma \ref{lemma_b3}, we obtain a ReLU network $\phi_{\times}$ with width $9 U+1$ and depth $2(s+1) V$ such that for any $t_1, t_2 \in[-1,1]$,
\begin{align}\label{lemma_a9}
    \left|t_1 t_2-\phi_{\times}\left(t_1, t_2\right)\right| \leq 24 U^{-2(s+1) V},
\end{align}

Also, by Lemma \ref{lemma_b4}, for any $\boldsymbol{\alpha} \in \mathbb{N}_0^d$ with $\| \boldsymbol{\alpha} \| \leq s$, there exists a ReLU network $P_{\boldsymbol{\alpha}}$ with width $9U+s+8$ and depth $7(s+1)^2 V$ such that $P_{\boldsymbol{\alpha}}(\mathbf{z}) \in[-1,1]$ and
\begin{align} \label{lemma_a10}
    \left|P_{\boldsymbol{\alpha}}(\mathbf{z})-\mathbf{z}^{\boldsymbol{\alpha}}\right| \leq 9(s+1)(U+1)^{-7(s+1) V}.
\end{align}

We can now approximate the Taylor expansion of $\tilde{f}(\mathbf{z})$ using a combination of sub-networks. Let $\varphi(t)=(\min \{\max \{\frac{t-a}{b-a}, 0\}, 1\})(b-a)+a=(\sigma(\frac{t-a}{b-a})-\sigma(\frac{t-a}{b-a}-1))(b-a)+a$ for $t \in \mathbb{R}$ where $\sigma(\cdot)$ is the ReLU activation function. And we define
$$
\begin{aligned}
& \tilde{\phi}_0(\mathbf{z}):=\phi_{\mathbf{0}_d}(\mathbf{z})+\sum_{1 \leq\|\boldsymbol{\alpha}\|_1 \leq s} \phi_{\times}\left(\frac{\phi_{\boldsymbol{\alpha}}(\mathbf{z})}{\boldsymbol{\alpha}!}, P_{\boldsymbol{\alpha}}(\varphi(\mathbf{z})-\phi(\mathbf{z}))\right), \\
& \phi_0(\mathbf{z}):=\sigma\left(\tilde{\phi}_0(\mathbf{z})+1\right)-\sigma\left(\tilde{\phi}_0(\mathbf{z})-1\right)-1 \in[-1,1],
\end{aligned}
$$
where $\mathbf{0}_d=(0, \ldots, 0) \in \mathbb{N}_0^d$.

Recall that the width and depth of the sub-networks are as follows:
\begin{itemize}
    \item $\varphi$: width and depth of $(2d, 1)$,
    \item $\psi$: width and depth of $(d(4\left\lfloor U^{1 / d}\right\rfloor+3), 4 V+ 5)$,
    \item $P_{\boldsymbol{\alpha}}$: width and depth of $(9 U+s+8,7(s+1)^2 V)$,
    \item $\phi_\alpha$: width of  $(16 d(s+1)(U+1)\left\lceil\log _2(8 U)\right\rceil, 5(V+2)\left\lceil\log _2(4 V)\right\rceil+4 V+5)$,
    \item $\phi_{\times}$:  width and depth of $(9 U+1,2(s+1) V)$.
\end{itemize}
Therefore, by our construction, the neural network implementation of $\phi_0$ has a width of $38(s+1)^2 d^{s+1} U\left\lceil\log _2(8 U)\right\rceil$ and a depth of $21(s+1)^2 V\left\lceil\log _2(8 V)\right\rceil$. Note that, for any $\mathbf{z} \in Q_{\boldsymbol{\tau}}$, we have $\psi(\mathbf{z})=\boldsymbol{\tau} / K$ and $\varphi(\mathbf{z}^{(i)})=\mathbf{z}^{(i)}$ for $i = 1, ..., d$. Applying equation (\ref{a_11}), we can bound the approximation error $|\tilde{f}(\mathbf{z})-\phi_0(\mathbf{z})|$,
\begin{align*}
 |\tilde{f}(\mathbf{z})-\phi_0(\mathbf{z})| \leq&|\tilde{f}(\mathbf{z})-\tilde{\phi}_0(\mathbf{z})| \\
\leq & |\tilde{f}(\boldsymbol{\tau} / J)-\phi_{\mathbf{0}_d}(\mathbf{z})|+d^{s+\beta / 2} J^{-\beta} \\
& +\sum_{1 \leq\|\boldsymbol{\alpha}\|_1 \leq s}\left|\frac{\partial^{\boldsymbol{\alpha}} \tilde{f}(\boldsymbol{\tau} / J)}{\boldsymbol{\alpha}!}(\mathbf{z}-\boldsymbol{\tau} / J)^{\boldsymbol{\alpha}}-\phi_{\times}\left(\frac{\phi_{\boldsymbol{\alpha}}(\mathbf{z})}{\boldsymbol{\alpha}!}, P_{\boldsymbol{\alpha}}(\mathbf{z}-\boldsymbol{\tau} / J)\right)\right| \\
= & d^{s+\beta / 2}\left\lfloor(UV)^{2 / d}\right\rfloor^{-\beta}+\sum_{\|\boldsymbol{\alpha}\|_1 \leq s} \epsilon_{\boldsymbol{\alpha}},
\end{align*}
where we denote $\epsilon_{\boldsymbol{\alpha}}=\left|\frac{\partial^{\boldsymbol{\alpha}} \tilde{f}(\boldsymbol{\tau} / J)}{\boldsymbol{\alpha}!}(\mathbf{z}-\boldsymbol{\tau} / J)^{\boldsymbol{\alpha}}-\phi_{\times}\left(\frac{\phi_{\boldsymbol{\alpha}}(\mathbf{z})}{\boldsymbol{\alpha}!}, P_{\boldsymbol{\alpha}}(\mathbf{z}-\boldsymbol{\tau} / J)\right)\right|$ for each $\boldsymbol{\alpha} \in \mathbb{N}_0^d$ with $\|\boldsymbol{\alpha}\|_1 \leq s$. Using the inequality
\begin{align*}
&|t_1 t_2-\phi_{\times}\left(t_3, t_4\right)| \leq\left|t_1 t_2-t_3 t_2\right|+\left|t_3 t_2-t_3 t_4\right|+\mid t_3 t_4- \phi_{\times}\left(t_3, t_4\right)|\\
& \leq| t_1-t_3|+| t_2-t_4|+| t_3 t_4-\phi_{\times}\left(t_3, t_4\right) \mid
\end{align*}
 for any $t_1, t_2, t_3, t_4 \in[-1,1]$,
 and applying  (\ref{lemma_a8}), (\ref{lemma_a9}), and (\ref{lemma_a10}), we have,  for $1 \leq |\boldsymbol{\alpha}|_1 \leq s,$
\begin{align*}
    \epsilon_{\boldsymbol{\alpha}} \leq & \frac{1}{\boldsymbol{\alpha}!}\left|\partial^{\boldsymbol{\alpha}} \tilde{f}(\boldsymbol{\tau} / J)-\phi_{\boldsymbol{\tau}}(\mathbf{z})\right|+\left|(\mathbf{z}-\boldsymbol{z} / J)^{\boldsymbol{\alpha}}-P_{\boldsymbol{\alpha}}(\mathbf{z}-\boldsymbol{\tau} / J)\right| \\
    & +\left\vert \frac{\phi_{\boldsymbol{\alpha}}(\mathbf{z})}{\boldsymbol{\alpha}!} P_{\boldsymbol{\alpha}}(\mathbf{z}-\boldsymbol{\tau} / J)-\phi_{\times}\left(\frac{\phi_{\boldsymbol{\alpha}}(\mathbf{z})}{\boldsymbol{\alpha}!}, P_{\boldsymbol{\alpha}}(\mathbf{z}-\boldsymbol{\tau} / J)\right)\right\vert \\
\leq & 2(UV)^{-2(s+1)}+9(s+1)(U+1)^{-7(s+1) V}+6 U^{-2(s+1) V} \\
\leq & (9 s+17)(U V)^{-2(s+1)} .
\end{align*}

It is easy to check that the bound is also true when $\|\boldsymbol{\alpha}\|_1=0$ and $s=0$. Therefore,
\begin{align*}
    \left|\tilde{f}(\mathbf{z})-\phi_0(\mathbf{z})\right| & \leq \sum_{1 \leq\|\boldsymbol{\alpha}\|_1 \leq s}(9 s+17)(U V)^{-2(s+1)}+d^{s+\beta / 2}(UV)^{-2 \beta / d} \\
& \leq(s+1) d^s(9 s+17)(UV)^{-2(s+1)}+d^{s+\beta / 2}(UV)^{-2 \beta / d} \\
& \leq 18(s+1)^2 d^{s+\beta / 2}(UV)^{-2 \beta / d},
\end{align*} for any $\mathbf{z} \in \bigcup_{\boldsymbol{\tau} \in\{0,1, \ldots, M-1\}^d} Q_\theta$. 
Then, we have
$$
\left|f(\mathbf{z})-B_0 \phi_0(\mathbf{z})\right| \leq 18 B_0(s+1)^2 d^{s+\beta / 2}(UV)^{-2 \beta / d},
$$
for any $\mathbf{z} \in \bigcup_{\theta \in\{0,1, \ldots, K-1\}^d} Q_\theta$.

For the case of $0<\beta \leq 1$, similar to the techniques  of Theorem 2.1 in \citet{shen2019deep}, we finally obtain that there exists a function $\phi_0$ which is implemented by a neural network with width $\max\{4d\lfloor U^{1/d}\rfloor + 3d, 12U + 8\}$ and
depth $12V + 14$, such that
$$
\left|f(\mathbf{z})-\phi_0(\mathbf{z})\right| \leq 18 \sqrt{d} B_0(UV)^{-2 \beta / d},
$$
for any $x \in \bigcup_{\theta \in\{0,1, \ldots, M-1\}^d} Q_{\boldsymbol{\tau}}$. Combining the results for $\beta \in(0,1]$ and $\beta>1$, we have for $f \in \mathcal{H}^\beta\left([a,b]^d, B_0\right)$, there exists a function $\phi_0$ implemented by a neural network with width $38(s+1)^2 d^{s+1} U\left\lceil\log _2(8 U)\right\rceil$ and depth $21(s+1)^2 V\left\lceil\log _2(8 V)\right\rceil$ such that
$$
\left|f(\mathbf{z})-\phi_0(\mathbf{z})\right| \leq 18 B_0(s+1)^2 d^{s+\beta \vee 1 / 2}(UV)^{-2 \beta / d},
$$
for any $\mathbf{z} \in \bigcup_{\boldsymbol{\tau} \in\{0,1, \ldots, M-1\}^d} Q_{\boldsymbol{\tau}}$ where $s=\lfloor\beta\rfloor$.
\end{proof}

\begin{lemma} \label{corollary_3}
    Assume that $f\in\mathcal{H}^\beta([a,b]^d, B_0)$ with $\beta=s+r$, $s\in\mathbb{N}_0$ and $r\in(0,1]$. For any $U,V\in\mathbb{N}^+$, there exists a function $\phi$ implemented by a ReLU network  with width $\mathcal{W}=38(\lfloor\beta\rfloor+1)^23^dd^{\lfloor\beta\rfloor+1}U\lceil\log_2(8U)\rceil$ and depth $\mathcal{D}=21(\lfloor\beta\rfloor+1)^2V\lceil \log_2(8V)\rceil+2d$ such that
	$$\vert f(\mathbf{z})-\phi(\mathbf{z}) \vert\leq 19B_0(s+1)^2d^{s+(\beta\vee1)/2}(b-a)^\beta(UV)^{-2\beta/d},\ \mathbf{z} \in [a,b]^d.$$
\end{lemma}

\begin{proof}
    Following the proof strategy outlined in the corollary 3.1 of \citet{jiao2023deep}, we can establish similar results for an arbitrary hypercube $[a, b]^d$. For constructing a neural network $\phi$ that uniformly approximates $f$, we first define the middle value function ${\rm mid}(t_1, t_2, t_3)$ that returns the middle value of inputs $t_1, t_2, t_3 \in \mathbb{R}$.   It can be shown that:
    $${\rm mid}\{t_1,t_2,t_3\}=\sigma(t_1+t_2+t_3)-\sigma(-t_1-t_2-t_3)-\max\{t_1,t_2,t_3\}-
    \min\{t_1,t_2,t_3\},$$
    where the maximum and minimum functions can be implemented using ReLU networks. Specifically, the maximum of three scalars can be expressed as:
    \begin{align*}
        &\max\{t_1,t_2\}=\frac{1}{2}(\sigma(t_1+t_2)-\sigma(-t_1-t_2)+\sigma(t_1-t_2)+\sigma(t_2-t_1)), \\
        &\max\{t_1,t_2,t_3\}=\max\{\max\{t_1,t_2\},\sigma(t_3)-\sigma(-t_3)\},
    \end{align*}
    and can be implemented by a ReLU network with width 6 and depth 2. The minimum function has a similar construction. Consequently, the middle value function ${\rm mid}(t_1, t_2, t_3)$ can be implemented by a ReLU network with width 14 and depth 2.

    We inductively define
    $$\phi_i(\mathbf{z}):={\rm mid}(\phi_{i-1}(\mathbf{z}-\delta \mathbf{e}_i),\phi_{i-1}(\mathbf{z}),\phi_{i-1}(\mathbf{z}+\delta \mathbf{e}_i))\in[-1,1],\ i=1,\ldots,d,$$
    where $\phi_0$ is defined in the proof of Lemma \ref{theorem_33}, $\{\mathbf{e}_i\}_{i=1}^d$ is the standard orthogonal basis in $\mathbb{R}^d$. Then $\phi_d$ can be implemented by a ReLU network with width $38(s+1)^23^dd^{s+1}U\lceil\log_2(8U)\rceil$ and depth $21(s+1)^2V\lceil\log_2(8V)\rceil+2d$ recalling that $\phi_0$ has width $38(s+1)^2d^{s+1}U\lceil\log_2(8U)\rceil$ and depth $21(s+1)^2V\lceil\log_2(8V)\rceil$. Let $M := \lceil b-a\rceil J$, and define the sets:
    \begin{align*}
        E_i:=\{(\mathbf{z}^{(1)},\ldots,\mathbf{z}^{(d)})\in[a,b]^d:\mathbf{z}^{(j)} \in Q(J,\delta),j>i\}, \text{ for } i=0,\ldots,d,
    \end{align*} where $Q(J,\delta):=\bigcup_{l=0}^{M-1}[\frac{l}{J},\frac{l+1}{J}-\delta\cdot 1_{k<M-1}]$. Then $E_0=\bigcup_{\boldsymbol{\tau}\in\{0,1,\ldots,M-1\}^d} Q_{\boldsymbol{\tau}}$ and $E_d=[a,b]^d$.

    We assert that for all $\mathbf{z} \in E_i, i = 0, \ldots, d$, the following inequality holds:
    $$\vert\phi_i(\mathbf{z})-f(\mathbf{z})\vert\le \epsilon+iB_0\delta^{\beta\wedge1},$$
     where we define $\epsilon := 18B_0(s+1)^2d^{s+\beta/2}(UV)^{-2\beta/d}$, and $t_1\wedge t_2:=\min\{t_1,t_2\}$ for $t_1,t_2\in\mathbb{R}$.

     It is true for $i = 0$ by construction. Assume the assertion is true for some $i$, we will prove that it is also holds for $i + 1$. Note that for any $\mathbf{z}\in E_{i+1}$, at least two of $\mathbf{z}-\delta \mathbf{e}_{i+1}$, $\mathbf{z}$ and $\mathbf{z}+ \delta \mathbf{e}_{i+1}$ are in $E_i$. Therefore,  by assumption and the inequality $\vert f(\mathbf{z})-f(\mathbf{z}\pm\delta \mathbf{e}_{i+1})\vert\leq B_0\delta^{\beta\wedge1}$, at least two of the following inequalities hold,
     \begin{align*}
         \vert \phi_i(\mathbf{z}-\delta \mathbf{e}_{i+1})-f(\mathbf{z})\vert\leq &\vert \phi_i(\mathbf{z}-\delta \mathbf{e}_{i+1})-f(\mathbf{z}-\delta \mathbf{e}_{i+1})\vert +B_0\delta^{\beta\wedge1}\leq \epsilon +(i+1)B_0\delta^{\beta\wedge1},\\
             \vert \phi_i(\mathbf{z})-f(\mathbf{z})\vert\leq & \epsilon +iB_0\delta^{\beta\wedge1},\\
             \vert \phi_i(\mathbf{z}+\delta \mathbf{e}_{i+1})-f(\mathbf{z})\vert\leq & \vert \phi_i(\mathbf{z}+\delta \mathbf{e}_{i+1})-f(\mathbf{z}+\delta \mathbf{e}_{i+1})\vert +B_0\delta^{\beta\wedge1}\leq \epsilon +(i+1)B_0\delta^{\beta\wedge1}.
           \end{align*}
     In other words, at least two of $\phi_i(\mathbf{z}-\delta \mathbf{e}_{i+1})$, $\phi_i(\mathbf{z})$ and $\phi_i(\mathbf{z}+\delta \mathbf{e}_{i+1})$ are in the interval $[f(\mathbf{z})-\epsilon-(i+1)B_0\delta^{\beta\wedge1},f(\mathbf{z})+\epsilon+(i+1)B_0\delta^{\beta\wedge1}].$ Hence, their middle value $\phi_{i+1}(\mathbf{z})={\rm mid}(\phi_i(\mathbf{z}-\delta \mathbf{e}_{i+1},\phi_i(\mathbf{z}),\phi_i(\mathbf{z}+\delta \mathbf{e}_{i+1})))$ must be in the same interval, which means
     $$\vert \phi_{i+1}(\mathbf{z})-f(\mathbf{z})\vert\leq\epsilon+(i+1)B_0\delta^{\beta\wedge1}.$$
     So the assertion is true for $i+1$. We take $\delta=3J^{-\beta\vee1}$, then
    \begin{equation*}
         \delta^{\beta\wedge1}=\Big( \frac{1}{3J^{\beta\vee1}}\Big)^{\beta\wedge1}=
     \Big\{
         \begin{array}{cc}
             \frac{1}{3}J^{-\beta} & \beta\ge1, \\
             (3J)^{-\beta} & \beta<1,
         \end{array}
 \end{equation*}
     and $J= \lceil(U V)^{2 / d}\rceil / \lceil b-a\rceil $. Since $E_d=[a,b]^d$, let $\phi:=\phi_d$, we have
     \begin{align*}
         \Vert \phi-f\Vert_{L^\infty([0,1]^d)}\leq&\epsilon+dB_0\delta^{\beta\wedge1}\\
         \leq& 18B_0(s+1)^2d^{s+(\beta\vee1)/2}(UV)^{-2\beta/d}+dB_0(b-a)^\beta(UV)^{-2\beta/d}\\
         \leq&19B_0(s+1)^2d^{s+(\beta\vee1)/2}(b-a)^\beta(UV)^{-2\beta/d},
     \end{align*}
 where $s=\lfloor\beta\rfloor$, which completes the proof.
\end{proof}

\section{Numerical experiments}\label{secA2}

\subsection{Derivation of the drift and the score functions in
Subsection \ref{Reg1}} \label{app_for_example}

We present detailed derivations for the drift and score functions of the regression model
 in Subsection \ref{Reg1}. The stochastic interpolation is
\begin{align*}
    Y_t  = \tilde{\mathcal{I}}(Y_0,Y_1, t) := a(t) Y_0 + b(t) Y_1 + \gamma(t) \eta.
\end{align*} Here, $Y_1 = f(X)+ \epsilon$, and $Y_0$, $\epsilon$, and $\eta$ are independent Gaussian random variables. The conditional drift and score functions can be derived as follows:
\begin{align}
  \label{example_score}
   s(\mathbf{x},y,t) = -\frac{1}{\gamma(t)}\mathbb{E}[\eta \mid Y_t = y, X= \mathbf{x}],
\end{align}
\begin{align}
    b(\mathbf{x},y,t)& = \mathbb{E}[\partial_t \mathcal{I}(Y_0,Y_1,t)+\dot{\gamma}(t)\eta \mid Y_t = y, X= \mathbf{x}] \nonumber \\
    &= \mathbb{E}[\dot{a}(t)Y_0+\dot{b}(t)Y_1+\dot{\gamma}(t)\eta \mid Y_t = y, X= \mathbf{x}] \nonumber \\
    &= \dot{a}(t)\mathbb{E}[Y_0 \mid Y_t = y, X= \mathbf{x}] + \dot{b}(t)\mathbb{E}[Y_1 \mid Y_t = y, X= \mathbf{x}] + \dot{\gamma}(t)\mathbb{E}[\eta \mid Y_t = y, X= \mathbf{x}] \nonumber\\
    &= \dot{a}(t)\mathbb{E}[Y_0 \mid Y_t = y, X= \mathbf{x}] + \dot{b}(t)f(\mathbf{x}) + \dot{b}(t)\mathbb{E}[\epsilon \mid Y_t = y, X= \mathbf{x}] \nonumber\\
    & + \dot{\gamma}(t)\mathbb{E}[\eta \mid Y_t = y, X= \mathbf{x}]. \label{example_drift}
\end{align}
First, we have
\begin{align*}
    \mathbb{E}[\eta \mid Y_t = y, X= \mathbf{x}]
    &= \mathbb{E}\left[\eta \mid {a}(t)Y_0+{b}(t)Y_1+{\gamma}(t)\eta = y, X= \mathbf{x}\right] \nonumber \\
    &= \mathbb{E}[\eta \mid {a}(t)Y_0+{b}(t)(f(\mathbf{x})+\epsilon)+{\gamma}(t)\eta = y] \nonumber \\
    &= \mathbb{E}[\eta \mid {a}(t)Y_0+{b}(t)+\epsilon+{\gamma}(t)\eta = y-{b}(t)f(\mathbf{x})] \nonumber\\
    &= \int \eta \frac{p_{a(t)Y_0+b(t) \epsilon}(y-b(t)f(\mathbf{x})-\gamma(t)\eta)p_{\eta}(\eta)}{\int p_{a(t)Y_0+b(t) \epsilon}(y-b(t)f(\mathbf{x})-\gamma(t)z)p_{\eta}(z) dz} d \eta,
\end{align*}
where
\begin{align*}
 &   \frac{p_{a(t)Y_0+b(t) \epsilon}(y-b(t)f(\mathbf{x})-\gamma(t)\eta)p_{\eta}(\eta)}{\int p_{a(t)Y_0+b(t) \epsilon}(y-b(t)f(\mathbf{x})-\gamma(t)z)p_{\eta}(z) dz} \\
    &\propto \exp \left[-\frac{1}{2(a^2(t)+b^2(t))} (y -b(t)f(\mathbf{x})-\gamma(t)\eta)^2 -\frac{1}{2} \eta^2\right] \\
    & \propto \exp \left[-\frac{a^2(t)+b^2(t) +\gamma^2(t)}{2(a^2(t)+b^2(t))} \left(\eta - \frac{\gamma(t)(y-b(t)f(\mathbf{x}))}{a^2(t)+b^2(t) +\gamma^2(t)}\right)^2 \right],
\end{align*}
which implies that the conditional distribution of $\eta \mid Y_t = y, X= \mathbf{x}$
is a normal distribution. Consequently, we obtain
\begin{align}
    \mathbb{E}[\eta \mid Y_t = y, X= \mathbf{x}] = \frac{\gamma(t)(y-b(t)f(\mathbf{x}))}{a^2(t)+b^2(t) +\gamma^2(t)}. \label{c28}
\end{align}
Combining formulas \ref{example_score} and \ref{c28}, we  obtain
\begin{align*}
    \boldsymbol{s}(\mathbf{x},y,t) = \frac{-(y-b(t)f(\mathbf{x}))}{a(t)^2 + b(t)^2 + \gamma(t)^2}.
\end{align*}
From symmetry, similarly, we can obtain
\begin{align}
    \mathbb{E}[Y_0 \mid Y_t = y, X= \mathbf{x}] &= \frac{a(t)(y-b(t)f(\mathbf{x}))}{a^2(t)+b^2(t) +\gamma^2(t)}, \label{c29}\\
    \mathbb{E}[\epsilon \mid Y_t = y, X= \mathbf{x}] &= \frac{b(t)(y-b(t)f(\mathbf{x}))}{a^2(t)+b^2(t) +\gamma^2(t)}. \label{c30}
\end{align}
Combining  (\ref{example_drift}), (\ref{c28}), (\ref{c29}),  and (\ref{c30}), we finally obtain
\begin{align*}
    \boldsymbol{b}(\mathbf{x},y,t) = \left( \frac{\dot{a}(t) 
     a(t) + \dot{b}(t)  b(t) + \dot{\gamma}(t)  \gamma(t)}{a(t)^2 + b(t)^2 + \gamma(t)^2} \right)(y-b(t)f(\mathbf{x})) + \dot{b}(t)f(\mathbf{x}).
\end{align*}
This completes the derivation for the example in Subsection \ref{Reg1}.

\subsection{Hyperparameter setting for neural network structure in Subsection \ref{sec_stl10}}

We use the U-net architecture \citep{ronneberger2015u} to estimate both the score function and drift function. The hyperparameters and architectural specifications are set to be consistent for the estimators of these two functions. This architecture incorporates residual blocks, upsampling blocks, downsampling blocks, and cross-attention layers \citep{chen2021crossvit} for integrating time, labels and image information. Residual blocks consist of a $3 \times 3$ convolutional neural network (CNN) structure with padding of 1, while upsampling blocks and downsampling blocks employ CNN structures or nearest-neighbor interpolation for upsampling or downsampling, respectively. The model we use encompasses $10$ feature map resolutions ranging from $96 \times 96$ to $6 \times 6$ and the widths of networks are in sequence $(128,256,512,1024,2048,2048,1024,512,256,128)$. The shared architecture of each feature map includes two residual blocks and one function block (upsample or downsample block) , facilitating double upsampling or downsampling. Specifically, the initial five maps perform upsampling from $96 \times 96$ to $6 \times 6$, while the remaining five maps handle the reverse upsampling process. Due to computational limitations, attention layers are included only in the middle six feature maps (resolutions from $24 \times 24$ to $6 \times 6$).

\end{document}